\documentclass[usletter]{article}

\usepackage[final]{neurips_2021}

\usepackage[utf8]{inputenc} 
\usepackage[T1]{fontenc}    
\usepackage{hyperref}       
\usepackage{url}            
\usepackage{booktabs}       
\usepackage{amsfonts}       
\usepackage{nicefrac}       
\usepackage{microtype}      

\usepackage{microtype}
\usepackage{graphicx}
\usepackage{subfigure}
\usepackage{booktabs} 
\usepackage{stmaryrd}
\usepackage{amsmath,amssymb,amsthm}
\usepackage[dvipsnames]{xcolor}
\usepackage{dsfont}
\usepackage{tikz}
\usepackage{pifont}
\usepackage{enumitem}
\usepackage{sidecap}
\usepackage{algorithm}
\usepackage{algorithmic}
\usepackage{amsmath}
\usepackage{wrapfig}
\usepackage{array, makecell}
\usepackage{hhline}
\usepackage{caption}

\usepackage{xspace}

\usepackage{minitoc}
\definecolor{azure}{rgb}{0.0, 0.5, 1.0}
\definecolor{cadmiumred}{rgb}{0.89, 0.0, 0.13}
\hypersetup{ %
colorlinks=true,
linkcolor=cadmiumred,
citecolor=azure,
filecolor=azure,
urlcolor=azure}

\DeclareRobustCommand{\eg}{e.g.,\@\xspace}
\DeclareRobustCommand{\ie}{i.e.,\@\xspace}

\newcommand{\ldpucbvi}{\textsc{LDP-OBI}\xspace}
\newcommand{\ucbvi}{\textsc{UCB-VI}\xspace}

\newcommand{\algo}{\ldpucbvi}
\newcommand{\algol}{\textsc{LDP-OBI-L}\xspace}
\newcommand{\algog}{\textsc{LDP-OBI-G}\xspace}
\newcommand{\algob}{\textsc{LDP-OBI-RR}\xspace}
\newcommand{\algobnd}{\textsc{LDP-OBI-Bnd}\xspace}

\newcommand{\ppsrl}{\textsc{LDP-PSRL}\xspace}
\newcommand{\psrl}{\textsc{PSRL}\xspace}

\newcommand{\transp}{\mathsf{T}}
\DeclareMathOperator*{\argmax}{\arg\,\max}

\newcommand{\wt}[1]{\widetilde{#1}}
\newcommand{\wb}[1]{\overline{#1}}

\newcommand{\wh}[1]{\widehat{#1}}

\newtheorem{theorem}{Theorem}
\newtheorem{proposition}[theorem]{Proposition}
\newtheorem{definition}[theorem]{Definition}
\newtheorem{assumption}[theorem]{Assumption}

\newtheorem{lemma}[theorem]{Lemma}

\newtheorem*{theorem*}{Theorem}
\newtheorem*{proposition*}{Proposition}
\newtheorem*{corollary*}{Corollary}

\usepackage[colorinlistoftodos, textwidth=18mm]{todonotes}

\def\showcomments{0}
\ifnum\showcomments=1
\newcommand{\todoe}[1]{}
\newcommand{\todoeout}[1]{}
\newcommand{\todomp}[1]{\todo[color=Green!10, inline]{\small MP: #1}}
\newcommand{\todompout}[1]{\todo[color=Green!10]{\scriptsize MP: #1}}
\newcommand{\todovp}[1]{\todo[color=Magenta!10, inline]{\small VP: #1}}
\newcommand{\todovpout}[1]{\todo[color=Magenta!10]{\scriptsize VP: #1}}
\newcommand{\todoc}[1]{}
\newcommand{\todocout}[1]{\todo[color=Blue!10]{\scriptsize CPB: #1}}

\else
\definecolor{violet}{rgb}{0, 0, 0}
\definecolor{red}{rgb}{0, 0, 0}
\definecolor{purple}{rgb}{0, 0, 0}
\newcommand{\todoe}[1]{}
\newcommand{\todoeout}[1]{}
\newcommand{\todomp}[1]{}
\newcommand{\todompout}[1]{}
\newcommand{\todovp}[1]{}
\newcommand{\todovpout}[1]{}
\newcommand{\todoc}[1]{}
\newcommand{\todocout}[1]{}
\fi

\title{Local Differential Privacy for Regret Minimization in Reinforcement Learning}

\author{%
  Evrard Garcelon\\
  Facebook AI Research \& CREST, ENSAE\\
  Paris, France\\
  \texttt{evrard@fb.com}\\
  \And
  Vianney Perchet \\
  CREST, ENSAE Paris \& Criteo AI Lab\\
  Palaiseau, France, \\
  \texttt{vianney@ensae.fr}\\
  \And
  Ciara Pike-Burke\\
  Imperial College London\\
  London, United Kingdom\\
  \texttt{c.pikeburke@gmail.com}\\
  \And
  Matteo Pirotta \\
  Facebook AI Research\\
  Paris, France\\
  \texttt{matteo.pirotta@gmail.com}\\
}

\begin{document}

\maketitle
\doparttoc 
\faketableofcontents
\begin{abstract}
  Reinforcement learning algorithms are widely used in domains where it is desirable to provide a personalized service. In these domains it is common that user data contains sensitive information that needs to be protected from third parties. Motivated by this, we study privacy in the context of finite-horizon Markov Decision Processes (MDPs) by requiring information to be obfuscated on the user side. We formulate this notion of privacy for RL by leveraging the local differential privacy (LDP) framework. We establish a lower bound for regret minimization in finite-horizon MDPs with LDP guarantees which shows that guaranteeing privacy has a multiplicative effect on the regret. This result shows that while LDP is an appealing notion of privacy, it makes the learning problem significantly more complex. Finally, we present an optimistic algorithm that simultaneously satisfies $\varepsilon$-LDP requirements, and achieves $\sqrt{K}/\varepsilon$ regret in any finite-horizon MDP after $K$ episodes,  matching the lower bound dependency on the number of episodes $K$.
\end{abstract}

\section{Introduction} \label{sec:introduction}
The practical successes of Reinforcement Learning (RL) algorithms have led to them becoming
ubiquitous in many settings such as digital marketing, healthcare and finance, where it is desirable to provide a personalized service~\citep[e.g.,][]{mao2020realworld, wang2021roboadvising}.
However, users are becoming increasingly wary of the amount of personal information that these services require. 
This is particularly pertinent in many of the aforementioned domains where the data obtained by the RL algorithm are highly sensitive.
For example, in healthcare, the state encodes personal information such as gender, age, vital signs, etc.
In advertising, it is normal for states to include browser history, geolocalized information, etc.
Unfortunately,~\citep{pan2019how} has shown that, unless sufficient precautions are taken, the RL agent leaks information about the environment {\color{red}(i.e., states containing sensitive information)}. That is to say,
observing {\color{red}the policy computed by the RL algorithm} 
is sufficient to infer information about the data (e.g., states and rewards) used
to compute the policy (scenario \ding{172}). This puts users' privacy at jeopardy.
Users therefore want to keep their sensitive information private, not only to an observer but also to the service provider itself (i.e., the RL agent). 
In response, many services are adapting to provide stronger protection of user privacy and personal data, for example by guaranteeing privacy directly on the user side (scenario \ding{173}).
{\color{red}
This often means that user data (i.e., trajectories of states, actions, rewards) are privatized before being observed by the RL agent.
In this paper, we study the effect that this has on the learning problem in RL.
}

Differential privacy (DP)~\citep{dwork2006calibrating} is a standard mechanism for preserving data privacy, both on the algorithm and the user side. 
{\color{red} The $(\varepsilon,\delta)$-DP definition guarantees that it is statistically hard to infer information about the data used to train a model by observing its predictions, thus addressing scenario \ding{172}.}
In online learning, $(\varepsilon,\delta)$-DP has been studied in the multi-armed bandit framework~\citep[\eg][]{mishra2015Nearly,tossou2016dpmab}.
 However, 
\citep{shariff2018differentially} showed that DP is incompatible with regret minimization in the contextual bandit problems. This led to considering weaker or different notions of privacy~\citep[\eg][]{shariff2018differentially, boursier2020utility}.
Recently,~\citep{vietri2020privaterl} transferred some of these techniques to RL, presenting the first private algorithm for regret minimization in finite-horizon problems. In~\citep{vietri2020privaterl}, they considered a relaxed definition of DP called \emph{joint differential privacy} (JDP) and showed that, under JDP constraints, the regret only increases by an additive term which is logarithmic in the number of episodes.
Similarly to DP, in the JDP setting the privacy burden lies with the learning algorithm which directly observes user states and trajectories containing sensitive data. In particular, this means that the data itself is not private and could potentially be used --for example by the owner of the application-- to train other algorithms with no privacy guarantees.
An alternative and stronger definition of privacy is \emph{Local Differential Privacy} (LDP)~\citep{duchi2013local}. This requires that the user’s data is protected at collection time before the learning agent has access to it. This covers scenario \ding{173} and implies that the learner is DP.
Intuitively, in RL, LDP ensures that each sample (states and rewards associated to an user) is already private when observed by the learning agent, while JDP requires computation on the entire set of samples to be DP. 
Recently, \citep{zheng2020locally} showed that, in contrast to DP, LDP is compatible with regret minimization in contextual bandits.\footnote{
    \color{violet}
    This shows that there are peculiarities in the DP definitions that are unique to sequential decision-making problems such as RL. The discrepancy between DP and LDP in RL is due to the fact that, when guaranteeing DP, actions taken by the learner cannot depend on the current state (this would break the privacy guarantee). On the other hand, in the LDP setting, the user executes a policy prescribed by the learner on its end (i.e., directly on non-private states) and send a privatized result (sequence of states and rewards observed by executing the policy) to the learner. Hence the user can execute actions based on its current state leading to a sublinear regret.
}
LDP is thus a stronger definition of privacy, simpler to understand and more user friendly.
These characteristics make LDP more suited for real-world applications. However, as we show in this paper, guaranteeing LDP in RL makes the learning problem more challenging.

\textbf{Contributions.} In this paper, we study LDP for regret minimization in finite horizon reinforcement learning problems with $S$ states, $A$ actions, and a horizon of $H$.\footnote{We do not explicitly focus on preventing malicious attacks or securing the communication between the RL algorithm and the users.
This is outside the scope of the paper.} Our contributions are 
as follows. \textbf{1)} We provide a regret lower bound for $(\varepsilon,\delta)$-LDP of $\Omega\big(H\sqrt{SAK}/\min\{e^\varepsilon - 1, 1\} \big)$, showing LDP is inherently harder than JDP, where the lower-bound is only $\Omega\big( H\sqrt{SAK} + SAH \log(KH)/\varepsilon \big)$~\citep{vietri2020privaterl}. \textbf{2)} We propose the first LDP algorithm for regret minimization in RL.
We use a general privacy-preserving mechanism to perturb information associated to each trajectory and derive~\algo, an optimistic model-based $(\varepsilon,\delta)$-LDP algorithm with regret guarantees. \textbf{3)} We present multiple privacy-preserving mechanisms 
that are compatible with \algo and show that their regret is  $\wt{O}(\sqrt{K}/\varepsilon)$ up to some mechanism dependent terms depending on $S,A,H$. \textbf{4)} We perform numerical simulations to evaluate the impact of LDP on the learning process. For comparison, we build a Thompson sampling algorithm~\citep[e.g.,][]{Osband2013more} for which we provide LDP guarantees but no regret bound.

\textbf{Related Work.}
The notion of differential privacy was introduced in~\citep{dwork2006calibrating} and is now a standard in machine learning~\citep[\eg][]{erlingsson2014rappor, dwork2014algorithmic, abowd2018us}.
Several notions of DP have been studied in the literature, including the standard DP and LDP notions.
While LDP is a stronger definition of privacy compared to DP, recent works have highlighted that it possible to achieve a trade-off between the two settings in terms of privacy and utility.
The shuffling model of privacy~\citep{Cheu_2019,feldman2020hiding,chen2020distributed, balle2019privacy, erlingsson2020encode} allows to build $(\varepsilon, \delta)$-DP algorithm with an additional $(\varepsilon',\delta')$-LDP guarantee (for $\varepsilon' \approx \varepsilon + \ln(n)$, any $\delta'>0$ where $n$ is the number of samples), hence it is possible to trade-off between DP, LDP, and utility in this setting. However, the scope of this paper is ensuring $(\varepsilon, \delta)$-LDP guarantees for a fixed $\varepsilon$. In this case, shuffling will not provide an improvement in utility (error) (see Thm~$5.2$ in Sec.~$5.1$ of \citep{feldman2020hiding} and App.~\ref{app:shuffle_model_rl}).

The bandit literature has investigated different privacy notions, including DP, JDP and LDP~\citep{mishra2015Nearly,tossou2016dpmab,gajane2018corrupt,shariff2018differentially,sajed2019optimal,chen2020locallycombinatorial,zheng2020locally,ren2020multi}.
In contextual bandits,~\citep{shariff2018differentially} derived an impossibility result for learning under DP by showing a regret lower-bound $\Omega(T)$ for any $(\epsilon,\delta)$-DP algorithm. Since the contextual bandit problem is a finite-horizon RL problem with horizon $H=1$, this implies that DP is incompatible with regret minimization in RL as well.
Regret minimization in RL with privacy guarantees has only been considered in~\citep{vietri2020privaterl}, where the authors extended the JDP approach from bandit to finite-horizon RL problems. They proposed a variation of UBEV~\citep{dann2017unifying} using a randomized response mechanism 
to guarantee $\varepsilon$-JDP with an additive cost to the regret bound.
While \emph{local differential privacy}~\citep{duchi2013local} has attracted increasing interest in the bandit literature~\citep[e.g.,][]{gajane2018corrupt,chen2020locallycombinatorial,zheng2020locally,ren2020multi}, it remains unexplored in the RL literature, and we provide the first contribution in that direction. Finally, outside regret minimization, DP has been studied in off-policy evaluation~\citep{balle2016differentially}, in control with DP guarantees on only the reward function~\citep{wang2019privacy}, and in distributional RL~\citep{ono2020locally}.

\section{Preliminaries}

                We consider a finite-horizon time-homogeneous Markov Decision Process (MDP)~\citep[][Chp. 4]{puterman1994markov} $M = (\mathcal{S}, \mathcal{A}, p, r, H)$ with state space $\mathcal{S}$, action space $\mathcal{A}$, and horizon $H \in \mathbb{N}^+$. Every state-action pair is characterized by a reward distribution with mean $r(s,a)$ supported in $[0, 1]$ and a transition distribution $p(\cdot|s,a)$ over next state.\footnote{We can simply modify the algorithm to handle step dependent transitions and rewards. The regret is then multiplied by a factor $H\sqrt{H}$.}
                We denote by $S = |\mathcal{S}|$ and $A = |\mathcal{A}|$ the number of states and actions.
                A non-stationary Markovian deterministic (MD) policy is defined as a collection $\pi = (\pi_1, \ldots, \pi_H)$ of MD policies $\pi_h : \mathcal{S} \to \mathcal{A}$.
                For any $h \in [H] := \{1, \ldots, H\}$ and state $s \in \mathcal{S}$, the value functions of a policy $\pi$
                are defined as $Q_h^\pi(s,a) = r(s,a)+\mathbb{E}_{\pi}\left[ \sum_{i=h+1}^H r(s_i,a_i) \right]$ and $V^{\pi}_{h}(s) = Q^\pi(s, \pi_h(s))$.
                There exists an optimal Markovian and deterministic policy $\pi^\star$~\citep[][Sec. 4.4]{puterman1994markov}
                such that $V^\star_h(s) = V_h^{\pi^\star}(s) = \max_{\pi} V_h^\pi(s)$.
                The Bellman equations at stage $h \in [H]$ are defined as $Q_h^\star(s,a) = r_h(s,a) + \max_{a'} \mathbb{E}_{s' \sim p_h(s,a')}\left[ V^\star_{h+1}(s')\right]$.
                The value iteration algorithm (a.k.a.\ backward induction) computes $Q^\star$ by applying the Bellman equations starting from stage $H$ down to $1$, with $V_{H+1}^\star(s) = 0$ for any $s$. The optimal policy is simply the greedy policy: $\pi^\star_h(s) = \argmax_{a} Q^\star_h(s,a)$.
                By boundness of the reward, all value functions $V^\pi_h(s)$ are bounded in $[0, H-h+1]$
                for any $h$ and $s$.

                \textbf{The general interaction protocol.}
{\color{red}                The learning agent (e.g., a personalization service) interacts with an unknown MDP with multiple users in a sequence of episodes $k \in [K]$ of fixed length $H$. At each episode $k$, an user $u_k$ arrives and their personal information (e.g., location, gender, health status, etc.) is encoded by the state $s_{1,k}$.
                The learner selects a policy $\pi_k$ that is sent to the user $u_k$ for local execution on ``clear'' states.
                The outcome of the execution, i.e., a trajectory, $X_k = (s_{kh}, a_{kh}, r_{kh}, s_{k,h+1})_{h\in[H]}$ is sent to the learner to update the policy. Note that we have not yet explicitly taken into consideration privacy in here.}
                We evaluate the performance of a learning algorithm $\mathfrak{A}$ which plays policies $\pi_1,\dots, \pi_K$ by its cumulative regret after $K$ episodes
                \begin{align}\label{eq:regret_finite_horizion}
                        \Delta(K) = \sum_{k = 1}^{K} ( V^{\star}_{1}(s_{1,k}) - V^{\pi_{k}}_{1}(s_{1,k})).
                \end{align}

        \subsection{Local Differential Privacy in RL}
        		In many application settings, when modelling a decision problem as a finite horizon MDP, it is natural to
                view each episode $k \in [K]$ as a trajectory associated to a specific user.
                In this paper, we assume that the sensitive information is contained in the states and
                rewards of the trajectory. Those quantities need to be kept private.
                This is reasonable in many settings such as healthcare, advertising, and finance, where states encode personal information, such as location, health, income etc.
                {\color{red}For example, an investment service may aim to provide each user with investment suggestions tailored to their income, deposit amount, age, risk level, properties owned, etc. This information is encoded in the user state and evolves over time as a consequence of investment decisions. The service provides guidances in the form of a policy (e.g., where, when and how much to invest) and the user follows the strategy for a certain amount of time. After that and based on the newly acquired information the provider may decide to change the policy.
                However, the user may want to keep their personal and sensitive information private to the company, while still receiving 
                a personalised and meaningful service.}
                This poses a fundamental challenge since in many cases, this information about actions
                taken in each state is essential for learning and creating a personalized experience
                for the user.
                The goal of a private RL algorithm is thus to ensure that the sensitive information remains private, while preserving the learnability of the problem.

                Privacy in RL has been tackled in~\citep{vietri2020privaterl} through the lens of
                \emph{joint differential privacy} (JDP). Intuitively, JDP requires that when a user changes, the actions observed by the other $K-1$ users will not change
                 much~\citep{vietri2020privaterl}.
                 The privacy burden thus lies with the RL algorithm.
                 The algorithm has access to all the information about the users (\ie trajectories)
                 containing sensitive data. It then has to provide guarantees about the privacy
                 of the data and carefully select the policies to execute in order to guarantee JDP.
                This approach to privacy requires the user to trust the RL algorithm to privately handle
                the data and not to expose or share sensitive information, and does not cover the examples mentioned above.

                In contrast to prior work, in this paper, we consider \emph{local differential privacy}
                (LDP) in RL. This removes the requirement that the RL algorithm observes the true
                sensitive data, achieving stronger privacy guarantees.
                LDP requires that an algorithm
                has access to user information (trajectories in RL) only through samples that have been privatized before being passed to the learning agent.
                This is different to JDP or DP where the trajectories are directly fed to the RL agent.
                In LDP, information is secured locally by the user using a private randomizer
                $\mathcal{M}$, before being sent to the RL agent.
                The appeal of this local model is that \emph{privatization can be done locally on the user-side}. Since nobody other than the user has ever
                access to any piece of non private data, this local setting is far more private. There are several variations of
                LDP available in the literature. In this paper, we focus on the non-interactive setting.
                We argue that this is more appropriate for RL. Indeed, due to the RL interaction framework, the data generated by
                user $k$ is a function of the data of all users $l < k$, therefore the data are not i.i.d. and
                the standard definition of sequential interactivity for LDP
                (Eq.~$1$ in \citep{duchi2013local}) is not applicable.
                It is therefore
                more natural to study the non-interactive setting (Eq.~$2$ in \citep{duchi2013local}) in RL. We formally define this below.

                Following the definition in \citep{vietri2020privaterl}, a user $u$
                is characterized by a starting state distribution $\rho_{0,u}$ (\ie for user $u$,
                 $s_{1}\sim \rho_{0,u}$) and a tree of depth $H$, describing all the possible sequence of states and
                 rewards corresponding to all possible sequences of actions.
                Alg.~\ref{alg:local_privacy} describes the LDP private interaction protocol
                between $K$ unique users $\{u_{1}, \ldots, u_{K}\} \subset \mathcal{U}^{K}$,
                with $\mathcal{U}$ the set of all users, and an RL algorithm $\mathfrak{A}$.
                For any $k \in [K]$, let $s_{1,k} \sim \rho_{0,u_k}$ be the initial state for
                user $u_k$ and denote by $X_{u_k} =\{(s_{k,h}, a_{k,h}, r_{k,h})\mid h\in [H]\}
                \in \mathcal{X}_{u_k}$ the trajectory corresponding to user $u_k$ {\color{violet} executing a policy $\pi_{k}$}. We write
                $\mathcal{M}(X_{u_k})$ to denote the privatized data generated by the randomizer
                for user $u_k$. The goal of mechanism $\mathcal{M}$ is to privatize sensitive informations
                 while encoding sufficient information for learning.
                With these notions in mind, 
                LDP in RL can be defined as follows:

                \begin{definition}\label{def:RL-LDP}
                        For any $\varepsilon\geq0$ and $\delta\geq 0$, a privacy preserving
                        mechanism $\mathcal{M}$  is said to be $(\varepsilon, \delta)$-Locally
                        Differential Private (LDP) if and only if for all users $u, u'\in \mathcal{U}$,
                        trajectories $(X_u,X_{u'}) \in \mathcal{X}_u \times \mathcal{X}_{u'}$ and all $O\subset
                        \{ \mathcal{M}(\mathcal{X}_{u}) \mid u\in \mathcal{U}\}$:
                        \begin{align}\label{eq:LDP_RL}
                                \mathbb{P}\left( \mathcal{M}({X}_{u}) \in O\right) \leq e^\varepsilon\, \mathbb{P}\left( \mathcal{M}({X}_{u'}) \in O\right) + \delta
                        \end{align}
                        where $\mathcal{X}_{u}$ is the space of trajectories associated to user $u$.
                \end{definition}

                Def.~\ref{def:RL-LDP} ensures that if the RL algorithm observes the output of the privacy mechanism $\mathcal{M}$ for two different input trajectories, then it is statistically difficult to guess which output is from which input trajectory. As a consequence, the users' privacy is preserved.


\section{Lower Bound}\label{sec:lower_bound}
We provide a lower bound on the regret that any LDP RL algorithm must incur. For this, as is standard when proving lower bounds on the regret in RL \citep[e.g.,][]{auer2002nonstochastic,lattimore2020bandit}, we construct a hard instance of the problem. The proof (see App.~\ref{app:proof_lower_bound}) relies on the fact that LDP acts as Lipschitz function, with respect to the KL-divergence, in the space of probability distribution.
\begin{theorem}[Lower-Bound]\label{thm:regret_lower_bound}
For any algorithm $\mathfrak{A}$ associated to a $\varepsilon$-LDP mechanism, any number of states $S\geq 3$, actions $A\geq2$ and $H \geq 2\log_{A}(S-2) + 2$, there exists an MDP $M$ with $S$ states and $A$ actions such that:
$\mathbb{E}_{M}(\Delta(K)) \geq \Omega\left( \frac{H\sqrt{SAK}}{\min\left\{\exp(\varepsilon) - 1, 1\right\}}\right)$.
\end{theorem}
{\color{violet} The lower bound of Thm.~\ref{thm:regret_lower_bound} shows that the price to pay for LDP in the RL setting is a factor $1/(\exp(\varepsilon) - 1)$ compared to the non-private lower bound of $H\sqrt{SAK}$.} The regret lower bound scales multiplicatively with the privacy parameter $\varepsilon$. The recent work of \citep{vietri2020privaterl} shows that for JDP, the regret in finite-horizon MDPs is
lower-bounded by $\Omega\left( H\sqrt{SAK} + \frac{1}{\varepsilon}\right)$.
Thm.~\ref{thm:regret_lower_bound} shows that the local differential privacy
setting is inherently harder than the joint differential privacy one
for small $\epsilon$, as our lower-bound scales with $\sqrt{K}/\varepsilon$
when $\varepsilon \approxeq 0$. Both bounds scale with $\sqrt{K}$
when $\varepsilon \rightarrow +\infty$.
%


\begin{figure}[t]
  \begin{minipage}[t]{.49\textwidth}
    \begin{algorithm}[H]
            \small
            \caption{Locally Private Episodic RL}
            \label{alg:local_privacy}
         \begin{algorithmic}
            \STATE {\bfseries Input:} Agent: $\mathfrak{A}$, Local Randomizer: $\mathcal{M}$,
                    Users: $u_{1},\ldots, u_{K}$
            \FOR{$k=1$ {\bfseries to} $K$}
                    \STATE Agent $\mathfrak{A}$ computes $\pi_k$ using $\{\mathcal{M}(X_{u_l})\}_{l\in[K-1]}$
                    \STATE User $u_{k}$ receives $\pi_{k}$ from agent $\mathfrak{A}$ and observes $s_{1,k} \sim \rho_{0,u_k}$
                    \STATE User $u_k$ executes policy $\pi_k$ on ``non-private'' states and observes a trajectory $X_{u_k} = \{(s_{h,k}, a_{h,k}, r_{h,k})\}_{h\in[H]}$
                    \STATE User $u_k$ sends back private data $\mathcal{M}(X_{u_k})$ to $\mathfrak{A}$
            \ENDFOR
         \end{algorithmic}
    \end{algorithm}
  \end{minipage}\hfill
  \begin{minipage}[t]{.49\textwidth}
\begin{algorithm}[H]
  \caption{\algo($\mathcal{M}$)}
  \label{alg:LDP-UCB-VI}
  \begin{algorithmic}
    \small
    \STATE {\bfseries Input:} $\delta \in (0,1)$, $\alpha>1$, randomizer $\mathcal{M}$
    with parameters $(\epsilon_0,\delta_0)$
    \FOR{$k=1$ {\bfseries to} $K$}
          \STATE Compute $\wt{p}_{k}$ and $\wt{r}_{k}$ as in Eq.~\eqref{eq:average_statistics} using $\{\mathcal{M}(X_{u_l})\}_{l\in[K-1]}$, $\beta_k^r$ and $\beta_k^p$ as in Prop.~\ref{prop:concentration_ldp} using $\{c_{k,i}(\varepsilon_{0}, \delta_{0}, \frac{3\delta}{2 k^2 \pi^2})\}_{i}$, and $b_{h,k}$ 
          \STATE Compute $\pi_k$ as in Eq.~\eqref{eq:truncated.vi} and send it to user $u_{k}$
          \STATE User $u_{k}$ executes policy $\pi_{k}$, collects trajectory $X_{k}$ and sends back privatized value $\mathcal{M}(X_{k})$
    \ENDFOR
  \end{algorithmic}
\end{algorithm}
\end{minipage}
\end{figure}

\section{Exploration with Local Differential Privacy}

  A standard approach to the design of the private randomizer $\mathcal{M}$ is to inject noise into the data to be preserved~\citep{dwork2014algorithmic}. A key challenge in RL is that we cannot simply inject noise to each component of the trajectory since this will break the \emph{temporal consistency} of the trajectory and possibly prevent learning. In fact, a trajectory is not an arbitrary sequence of states, actions, and rewards but obeys the Markov reward process induced by a policy.
  Fortunately, Def.~\ref{def:RL-LDP} shows that the output of the randomizer need not necessarily be a trajectory but could be any private information built from it. In the next section, we show how to leverage this key feature to output succinct information that preserves the information encoded in a trajectory while satisfying the privacy constraints.
  We show that the output of such a randomizer can be used by an RL algorithm to build estimates of the unknown rewards and transitions. While these estimates are biased, we show that they carry enough information to derive optimistic policies for exploration. We leverage these tools to design \algo, an optimistic model-based algorithm for exploration with LDP guarantees.



\subsection{Privacy-Preserving Mechanism} \label{sec:privmech}

Consider the locally-private episodic RL protocol described in Alg.~\ref{alg:local_privacy}.
At the end of each episode $k \in [K]$, user $u_{k}$ uses a private randomizer $\mathcal{M}$ to generate a private statistic $\mathcal{M}(X_{u_{k}})$ to pass to the RL algorithm $\mathfrak{A}$.
This statistic should encode sufficient information for the RL algorithm to improve the policy while maintaining the user's privacy.
In \emph{model-based} settings, a sufficient statistic is a local estimate of the rewards and transitions. Since this cannot be reliably obtained from a single trajectory, we resort to counters of visits and rewards that can be aggregated by the RL algorithm.

For a given trajectory $X= \{ (s_{h}, a_{h}, r_{h})\}_{h\in [H]}$,
let $R_X(s,a) = \sum_{h=1}^{H} r_{h}\mathds{1}_{\{s_{h} = s, a_{h} = a\}}$, $N^r_X(s,a) = \sum_{h=1}^{H} \mathds{1}_{\{s_{h} = s, a_{h} = a\}}$ and $N^p_X(s,a,s') = \sum_{h=1}^{H-1} \mathds{1}_{\{s_{h} = s, a_{h} = a, s_{h+1} = s'\}}$ be  the true non-private statistics, which the agent will never observe.
We design the mechanism $\mathcal{M}$ so that for a given trajectory $X$, $\mathcal{M}$ returns private versions $\mathcal{M}(X) = (\wt{R}_{X}, \wt{N}^{r}_{X}, \wt{N}^{p}_{X})$
 of these statistics.
%
Here, $\wt{R}_X(s,a)$ is a noisy version of the cumulative reward $R_X(s,a)$, 
and $\wt{N}^r_X$ 
and $\wt{N}^p_X$ 
are perturbed counters of visits to state-action and state-action-next state tuples, respectively.
At the beginning of episode $k$, the algorithm has access to the aggregated private statistics:
\begin{equation}\label{eq:noisy_aggregate}
\begin{aligned}
\wt{R}_{k}(s,a)= \sum_{l < k} \wt{R}_{X_{u_{l}}}(s,a),~ \wt{N}_{k}^{r}(s,a) = \sum_{l < k} \wt{N}^{r}_{X_{u_{l}}}(s,a),~ \wt{N}_{k}^{p}(s,a,s') = \sum_{l < k} \wt{N}_{X_{u_{l}}}^{p}(s,a,s')
\end{aligned}
\end{equation}
We denote the non-private counterparts of these aggregated statistics as
$R_{k}(s,a) = \sum_{l < k} R_{X_{u_{l}}}(s,a)$, $N_{k}^{r}(s,a) = \sum_{l < k} N_{X_{u_{l}}}^{r}(s,a)$ and $N_{k}^{p}(s,a,s') = \sum_{l < k} N_{X_{u_{l}}}^{p}(s,a,s')$, these are also \emph{unknown} to the RL agent.
%
Using these private statistics, we can define conditions that a private randomizer
must satisfy in order for {\color{violet} our RL agent, \algo,} to be  able to learn the reward and dynamics of the MDP.
\todocout{I'm a bit confused by this. Is this section about any private randomizer or just the one we consider? But we consider a class of randomizers right?}
\begin{assumption} 
  \label{assumption:concentration_privacy}
  The private randomizer $\mathcal{M}$ satisfies $(\varepsilon_{0}, \delta_{0})$-LDP, Def.~\ref{def:RL-LDP}, with $\varepsilon_{0},\delta_{0} \geq0$.
  Moreover, for any $\delta>0$ and $k\geq 0$, there exist four finite strictly positive function, $c_{k,1}(\varepsilon_{0}, \delta_{0}, \delta), c_{k,2}(\varepsilon_{0}, \delta_{0},\delta), c_{k,3}(\varepsilon_{0}, \delta_{0},\delta) 
  ,c_{k,4}(\varepsilon_{0}, \delta_{0},\delta)\in \mathbb{R}^{\star}_{+}$ such that with probabilty at least $1-\delta$ for all $(s,a,s')\in \mathcal{S}\times\mathcal{A}\times\mathcal{S}$:
\begin{align*}
  &\left| \wt{R}_{k}(s,a) - R_{k}(s,a) \right| \leq c_{k,1}(\varepsilon_{0}, \delta_{0},\delta), && \left|\wt{N}_{k}^{r}(s,a) - N_{k}^{r}(s,a) \right| \leq c_{k,2}(\varepsilon_{0}, \delta_{0},\delta)\\
  &\left|  \sum_{s'} N_{k}^{p}(s,a,s') - \wt{N}_{k}^{p}(s,a,s')\right| \leq c_{k,3}(\varepsilon_{0}, \delta_{0},\delta), && \left| N_{k}^{p}(s,a,s') - \wt{N}_{k}^{p}(s,a,s')\right| \leq c_{k,4}(\varepsilon_{0}, \delta_{0},\delta)
\end{align*}
The functions  $c_{k,1}(\varepsilon_{0}, \delta_{0},\delta)$, $c_{k,2}(\varepsilon_{0}, \delta_{0},\delta)$, $c_{k,3}(\varepsilon_{0}, \delta_{0},\delta)$ and $c_{k,4}(\varepsilon_{0}, \delta_{0},\delta)$ must be increasing functions of $k$ and decreasing functions of $\delta$.
 We also write $c_{k,1}(\varepsilon_{0},\delta)$, $c_{k,2}(\varepsilon_{0},\delta)$, $c_{k,3}(\varepsilon_{0},\delta)$ and $c_{k,4}(\varepsilon_{0}, \delta)$ when $\delta_{0} = 0$.
\end{assumption}
In Sec.~\ref{sec:inst}, we will present schemas satisfying Asm.~\ref{assumption:concentration_privacy} and discuss their impacts on privacy and regret.

\subsection{Our LDP Algorithm For Exploration}
In this section, we introduce \algo (\emph{Local Differentially Private Optimistic Backward Induction}), a flexible optimistic model-based algorithm for exploration that can be paired with any privacy mechanism satisfying Asm.~\ref{assumption:concentration_privacy}.
When developing optimistic algorithms it is necessary to define confidence intervals using an estimated model that are broad enough to capture the true model with high probability, but narrow enough to ensure low regret.
This is made more complicated in the LDP setting, since the estimated model is defined using randomized counters.
In particular, this means we cannot use standard concentration inequalities such as those used in~\citep{azar2017minimax,zanette2019tight}.
Moreover, when working with randomized counters, classical estimators like the empirical mean can even be ill-defined as the number of visits to a state-action pair, for example, can be negative.

Nevertheless, we show that by exploiting the properties of the mechanism $\mathcal{M}$ in Asm.~\ref{assumption:concentration_privacy}, it is still possible to define an empirical model which can be shown to be close to the true model with high probability.
%
To construct this empirical estimator, we rely on the fact that for each state-action pair $(s,a)$, $\wt{N}_{k}^{r}(s,a) + c_{k,2}(\varepsilon_{0}, \delta_{0}, \delta) \geq N_{k}^{r}(s,a) \geq 0$ with high probability where the precision $ c_{k,2}(\varepsilon_{0}, \delta_{0}, \delta) $ ensures the positivity of the noisy number of visits to a state action-pair. A similar argument holds for the transitions.
Formally, the estimated \emph{private} rewards and transitions before episode $k$ are defined as follows:
\begin{equation}\label{eq:average_statistics}
  \begin{aligned}
  &\wt{r}_k(s,a) = \frac{\wt{R}_k(s,a)}{\wt{N}_{k}^{r}(s,a) + \alpha c_{k,2}(\varepsilon_{0}, \delta_{0}, \delta)}, \qquad \wt{p}_{k}(s'\mid s,a) = \frac{\wt{N}_{k}^{p}(s,a,s')}{\wt{N}_{k}^{p}(s,a) + \alpha c_{k,3}(\varepsilon_{0}, \delta_{0}, \delta)}
  \end{aligned}
\end{equation}
Note that unlike in classic optimistic algorithms, $\wt{p}_{k}$ is not a probability measure but a signed sub-probability measure.
However, this does not preclude good performance.
By leveraging properties of Asm.~\ref{assumption:concentration_privacy} we are able to build confidence intervals using these private quantities (see App.~\ref{app:proof_algo}).

\begin{proposition}\label{prop:concentration_ldp}
  For any $\varepsilon_{0}>0$, $\delta_{0}\geq0$, $\delta >0$, $\alpha > 1$ and episode $k$, using mechanism $\mathcal{M}$ satisfying Asm.~\ref{assumption:concentration_privacy}, then with probability at least $1-2\delta$, for any $(s,a)\in\mathcal{S}\times\mathcal{A}$
  {\small\begin{align*}
    \left|r(s,a) - \wt{r}_k(s,a)\right| \leq \beta_k^r(s,a) = \sqrt{\frac{2\ln\left(\frac{4\pi^{2}SAHk^{3}}{3\delta}\right)}{\wt{N}_{k}^{r}(s,a) + \alpha c_{k,2}(\varepsilon_{0},\delta_{0},\delta)}} +\frac{(\alpha + 1)c_{k,2}(\varepsilon_{0}, \delta_{0}, \delta) + c_{k,1}(\varepsilon_{0}, \delta_{0},\delta)}{\wt{N}_{k}^{r}(s,a) + \alpha c_{k,2}(\varepsilon_{0}, \delta_{0}, \delta)}\\
    \|p(\cdot| s,a) - \wt{p}_{k}(\cdot| s,a)\|_{1}\leq \beta_k^p(s,a) = 
    \sqrt{\frac{14S\ln\left(\frac{4\pi^{2}SAHk^{3}}{3\delta}\right)}{\wt{N}_{k}^{p}(s,a) + \alpha c_{k,3}(\varepsilon_{0},\delta_{0}, \delta)}}+ \frac{Sc_{k,4}(\varepsilon_{0},\delta_{0}, \delta)}{\wt{N}_{k}^{p}(s,a) + \alpha c_{k,3}(\varepsilon_{0},\delta_{0}, \delta)} +
    \\\frac{(\alpha + 1)c_{k,3}(\varepsilon_{0},\delta_{0}, \delta)}{\wt{N}_{k}^{p}(s,a) + \alpha c_{k,3}(\varepsilon_{0},\delta_{0}, \delta)}
  \end{align*}}
\end{proposition}
The shape of the bonuses in Prop.~\ref{prop:concentration_ldp} highlights two terms. The first term is reminiscent of Hoeffding bonuses as it scales with $\mathcal{O}\Big(1/\sqrt{\wt{N}_{k}^{p}}\Big)$. The other term is of order $\mathcal{O}\Big(1/\wt{N}_{k}^{p}\Big)$ and accounts for the variance (and potentially bias) of the noise added by the privacy-preserving mechanism.

As commonly done in the literature~\citep[\eg][]{azar2017minimax,Qian2019bonus,neu2020unifying}, we use these concentration results to define a bonus function
  $b_{h,k}(s,a) := (H-h+1)\cdot \beta_k^p(s,a) + \beta_k^r(s,a)$
which is used to define an optimistic value function and policy by running the following backward induction procedure:
\begin{align}\label{eq:truncated.vi}
  Q_{h,k}(s,a)s= \wt{r}_k(s,a) + b_{h,k}(s,a) + \wt{p}_k(\cdot|s,a)^\transp V_{h+1,k}, ~~~
  \pi_{h,k}(s)= \argmax_{a} Q_{h,k}(s,a)
\end{align}
where $V_{h,k}(s) = \min\{H-h+1, \max_a Q_{h,k}(s,a)\}$ and $V_{H+1,k}(s) = 0$.



\subsection{Regret Guarantees}\label{sec:regret}
We get the following general guarantees for any LDP mechanism satisfying Asm.~\ref{assumption:concentration_privacy} in \ldpucbvi.

\begin{theorem}\label{thm:regret_any_mechanism}
For any privacy mechanism $\mathcal{M}$ satisfying Asm.~\ref{assumption:concentration_privacy} with $\varepsilon>0$, $\delta_{0}\geq 0$, and for any $\delta>0$ the regret of \ldpucbvi is bounded with probability at least $1-\delta$ by:
\begin{equation}\label{eq:regret}
\begin{aligned}
\Delta(K) \leq \tilde{\mathcal{O}}\Bigg( \underbrace{HS\sqrt{AT}}_{\text{\ding{182}}} + SAH^{2}c_{K,3}\left(\varepsilon ,\delta_{0}, \frac{3\delta}{2\pi^{2}K^{2}}\right) +
H^{2}S^{2}Ac_{K,4}\left(\varepsilon ,\delta_{0},\frac{3\delta}{2\pi^{2}K^{2}}\right)&\\
+ SAH c_{K,2}\left(\varepsilon ,\delta_{0},\frac{3\delta}{2\pi^{2}K^{2}}\right) +
SAH c_{K,1}\left(\varepsilon ,\delta_{0},\frac{3\delta}{2\pi^{2}K^{2}}\right)   \Bigg)&
\end{aligned}
\end{equation}
The combination of $\mathcal{M}$ and \algo is also $(\varepsilon, \delta_{0})$-LDP.
\end{theorem}
Thm.~\ref{thm:regret_any_mechanism} shows that the regret of \ldpucbvi \emph{1)} is lower bounded by the regret in non-private settings; and \emph{2)} depends directly on the precision of the privacy mechanism used though $c_{K,1}, \dots, c_{K,4}$. Thus improving the precision, that is to say reducing the amount of noise that needs to be added to the data to guarantee LDP of the privacy mechanism, directly improves the regret bounds of \ldpucbvi.
{\color{red}
The first term in the regret bound (\ding{182}) is of the order expected in the non-private setting (see e.g., \cite{Jaksch10}).
Classical results in DP suggest that the $\{c_{K,i}\}_{i\leq 4}$ terms should be \emph{approximately} of order $\sqrt{K}/\varepsilon$ (this is indeed the case for many natural choices of randomizer).
In such a case, the dominant term in \eqref{eq:regret}, is no longer \ding{182} but rather a term of order $H^2S^2A\sqrt{K}/\varepsilon$ (from e.g. $c_{K,4}$).
The dependency on $S,A,H$ is larger than in the non-private setting.
This is because the cost of LDP is multiplicative, so it also impacts the lower order terms in the concentration results (see e.g. the second term in~\ref{prop:concentration_ldp}), 
which are typically ignored in the non-private setting.
In addition, this implies that variance reduction techniques for RL (e.g., based on Bernstein) classically used to decrease the dependence on $S,H$ will not lead to any improvement here.
This is to be contrasted with the JDP setting where \cite{vietri2020privaterl} shows that the cost of privacy is additive so using variance reduction techniques can reduce the dependency of the regret on $S,A,H$.
\todocout{I've changed this paragraph, is it still correct? The writing definitely needs to be clarified further if it is correct}
}


\begin{table}[t!]
  \centering
    \footnotesize
    \renewcommand{\arraystretch}{1.75}
    \begin{tabular}{|c|c|c|c|}
      \hline
      $\mathcal{M}$ & Noise  & $(\epsilon,\delta)$-LDP level & Regret $\Delta(T)$\\
      \hline \hline
      Laplace &  $\text{Lap}(6H/\varepsilon)$ & $(\varepsilon,0)$ & $\wt{O}(H^3 S^2 A\sqrt{K}/\varepsilon)$\\
      \hline
      Gaussian & $\mathcal{N}(0, (H/\varepsilon)^{2})$   & $(\varepsilon, \delta_0)$ & $\wt{O}(H^3 S^2 A\sqrt{K\ln(1/\delta_{0})}/\varepsilon)$\\
      \hline
      \makecell{Randomized\\ Response} & $\text{Ber}((e^{\varepsilon/H} - 1)^{-1})$
      & $(\varepsilon, 0)$ & $\wt{O}(H^{7/2} S^2 A\sqrt{K}/\varepsilon)$\\
      \hline
      \makecell{Bounded\\ Noise} &See \citep{dagan2020boundednoise}  and App.~\ref{app:bounded_noise}& $(\varepsilon, \delta_{0})$ & $\wt{O}(H^{2} S^{3} A^{3/2}\sqrt{K\ln(1/\delta_{0})}/\varepsilon)$\\
      \hline
    \end{tabular}
  \vspace{0.04in}
  \caption{
    Summary of the guarantees of $\algo$ with different randomizers for $\varepsilon >0$ and $\delta_0 > 0$. For the mechanism in this table, we have approximately
    $c_{k,i} = \wt{\mathcal{O}}(\sqrt{kH}/\varepsilon)$ for $i\in\{1,2,4\}$ (ignoring $\log$ terms) and $c_{k,3} = \wt{\mathcal{O}}(\sqrt{SkH}/\varepsilon)$}

  \label{tab:algovariants.guarantees}
\end{table}

\section{Choice of Randomizer} \label{sec:inst}
There are several randomizers that satisfy Asm.~\ref{assumption:concentration_privacy}, for example Laplace~\citep{dwork2014algorithmic}, randomized response~\citep{erlingsson2014rappor, kairouz2016discrete}, Gaussian~\citep{wang2019locally} and bounded noise~\citep{dagan2020bounded} mechanisms.
Since one method can be preferred to another depending on the application, we believe it is important to understand the regret and privacy guarantees achieved by \ldpucbvi with these randomizers.
Tab.~\ref{tab:algovariants.guarantees} provides a global overview of the properties of \ldpucbvi with different randomized mechanism.
The detailed derivations are deferred to App.~\ref{app:other_mechanisms}.

\textbf{Privacy.} All the mechanisms satisfy Asm.~\ref{assumption:concentration_privacy} but only the Laplace and Randomized Response mechanisms guarantees $(\varepsilon, 0)$-LDP.
Note that in all cases, in order to guarantee a $\varepsilon$ level of privacy (or $(\varepsilon, \delta)$ for the Gaussian and bounded noise mechanisms), it is necessary to scale the parameter $\varepsilon$ proportional to $1/H$. This is because the statistics computed by the privacy-preserving mechanism are the sum of $H$ observations which are bounded in $[0,1]$, the sensitivity\footnote{For a function $f: \mathcal{X} \rightarrow \mathbb{R}$ the sensitivity is defined as $S(f) = \max_{x,y\in \mathcal{X}} \left| f(x) - f(y)\right|$} of those statistics is bounded by $H$. Directly applying the composition theorem for DP~\citep[][Thm 3.14]{dwork2014algorithmic} over the different counters, 
would lead to an upper-bound on the privacy of the mechanism of $S^{2}AH\varepsilon$ and corresponding regret bound of $\wt{O}\left((H^{4}S^{4}A^{2}\sqrt{K})/\varepsilon\right)$.
For the randomizers that we use, the impact on $\varepsilon$ is lower thanks to fact that they are designed to exploit the structure of the input data (a trajectory). \todocout{Is this what we are trying to say? I didn't follow the commented out sentence}

\textbf{Regret Bound.}
From looking at Table~\ref{tab:algovariants.guarantees}, we see that while all the mechanisms achieve a regret bound of order $\wt{O}(\sqrt{K})$ the dependence on the privacy level $\varepsilon$ varies as well as the privacy guarantees.
The regret of Laplace, Gaussian and bounded noise mechanisms scale with $\varepsilon^{-1}$,
whereas the randomized response has an exponential dependence in $\varepsilon$ similar to
the lower bound. However, this improvement comes at the price of worse dependency in $H$ when $\varepsilon$ is small, and a worse multiplicative constant in the regret. This is due to the randomized response mechanism perturbing the counters for each stage $h \in [H]$, leading to up to $HS^2A$ obfuscated elements. 
This worse dependence is also observed in our numerical simulations.

For many of the randomizers, our regret bounds scale as $\wt{O}(H^3S^2A\sqrt{K}/\varepsilon)$.
Aside from the $\sqrt{K}/\epsilon$ rate which is expected, our bounds exhibit worse dependence on the MDP characteristics when compared to the non-private setting.
We believe that this is unavoidable due to the fact that we have to make $S^2A$ terms private, while the extra dependence on $H$ comes from dividing $\varepsilon$ by $H$ to ensure privacy over the whole trajectory.
%
{\color{red} Moreover, the DP literature~\citep[e.g.,][]{duchi2017minimax,duchi2019lower,ye2017optimal} suggests that the extra dependency on $S, A, H$ may be inherent to model-based algorithms due to the explicit estimation of private rewards and transitions. Indeed,
\citep{ye2017optimal} shows that the minimax error rate in $\ell_{1}$ norm for estimating a distribution over $S$ states is $\Omega\left(\frac{S}{\sqrt{n}(\exp(\varepsilon) - 1)}\right)$ with $n$ samples in the high privacy regime ($\varepsilon<1$),
while there is no change in the low privacy regime.
This means that in the high privacy regime the concentration scales with a multiplicative $\sqrt{S}$ term which would translate directly into the regret bound. Furthermore, this results assumes that the number $n$ of samples is known to the learner. In our setting, $n$ maps to $N_k(s,a)$ which is unknown to the algorithm. 
Since we only observe a perturbed estimate of $n$, estimating $p(\cdot| s,a)$ here is strictly harder than the aforementioned setting.}\footnote{\color{violet}
We are not aware of any lower-bound in the literature that applies to this setting but we 
believe that the $S^{2}A\sqrt{KH}/\varepsilon$ dependence may be unavoidable for model-based algorithms. This is because
 $N_{k}(s,a)$ and $\wt{N}_{k}(s,a)$ differ by at most $\sqrt{kH\log(SA)}$ (which is a well-known lower bound for the counting elements problem see \citep{bassily_2015}). Intuitively this difference creates a bias when estimating each component
$p(\cdot|s,a)$, a bias that would scale with the size of the support $p(\cdot|s,a)$ and the relative difference between $N_{k}(s,a)$ and $\wt{N}_{k}(s,a)$. Hence, the bias would scale with $S\sqrt{kH}/N_{k}(s,a)$. Summing over all episodes and $SA$ counters gives the conjectured result.}
\todocout{This discussion needs polishing. I've had a go but it still needs more work. (I also kept refering to the lower bound as a conjecture, but is it just a conjecture?)}
%
This suggests that it is impossible for any model-based algorithm which directly estimates the transition probabilities to match the lower bound.
However, this does not rule out the possibility of a model-free algorithm being able to match the lower bound. Designing such a model-free algorithm which is able to work with LDP trajectories is non-trivial and we leave it to future work.


Another direction for future work is to investigate whether the recently developed shuffling model~\citep{erlingsson2020amplification} may be used to improve our regret bounds in the LDP setting.
Preliminary investigations of the shuffling model (see App.~\ref{app:shuffle_model_rl}) show that it is not possible while preserving a fixed $\varepsilon$-LDP constraint, which is the focus of this paper. Nonetheless, if we were to relax the privacy constraint to only guarantee $\varepsilon$-JDP then the shuffling model could be used to retrieve the regret bound in~\citep{vietri2020privaterl} while guaranteeing some level of local differential privacy, although the level of LDP would be much weaker than the one considered in this paper.
We believe the study of this model sitting in-between the joint and local DP settings for RL is a promising direction for future work and that the tools developed in this paper will be helpful for tackling this problem.

\section{Numerical Evaluation}\label{sec:experiments}
In this section, we evaluate the empirical performance of \algo on a toy MDP.
We compare \algo with the \textit{non-private} algorithm \ucbvi~\citep{azar2017minimax}.
To the best of our knowledge there is no other LDP algorithm for regret minimization in MDPs in the literature.
To increase the comparators, we introduce a novel LDP algorithm based on Thompson sampling~\citep[\eg][]{Osband2013more}.

\paragraph{LDP-PSRL.}
Thompson sampling algorithms~\citep[e.g., PSRL,][]{Osband2013more} have proved to be effective in several applications~\citep{RussoRKOW18}.
Due to their inherent randomization, one may imagine that they are also well suited to LDP regret minimization. Here, we introduce and evaluate \ppsrl, an LDP variant of PSRL and provide a first empirical evaluation.
Informally, by defining by $\mathcal{W}_k = \{ (S,A,p,r,H) : \|p - \wt{p}\|_1 \leq \beta_k^p, |r - \wt{r}| \leq \beta_k^r\}$ the \emph{private} set of plausible MDPs constructed using the definition in Prop.~\ref{prop:concentration_ldp}, we can see posterior sampling as drawing an MDP from this set at each episode $k$ and running the associated optimal policy:
\[
\emph{i)} ~M_k \sim \mathbb{P}(\mathcal{W}_k),~~~~ \emph{ii)}~ \pi_k = \max_{\pi} \{V^{\pi}_1(M_k)\}.
\]
More formally, we consider Gaussian and Dirichlet prior for rewards and transition which lead to Normal-Gamma and Dirichlet distributions as posteriors. We use the private counters defined in Asm.~\ref{assumption:concentration_privacy} to update the parameters of the posterior distribution and thus the distribution over plausible models.
We provide full details of this schema in App.~\ref{app:psrl_ldp} and show that it is LDP. However, we were not able to provide a regret bound for this algorithm.

\textbf{Simulations.} We consider the RandomMDP environment described in \citep{dann2017unifying} where for each state-action pair transition probabilities are sampled from a $\text{Dirichlet}(\alpha)$ distribution (with $\alpha_{s,a,s'} = 0.1$ for all $(s,a,s')$)  and rewards are deterministic in $\{0, 1\}$ with $r(s,a) = \mathds{1}_{\{U_{s,a} \leq 0.5\}}$ for $(U_{s,a})_{(s,a)\in \mathcal{S}\times\mathcal{A}} \sim \mathcal{U}([0,1])$
sampled once when generating the MDP. We set the number of states $S = 2$, number of actions $A=2$ and horizon $H=2$. We evaluate the regret of our algorithm for $\varepsilon\in\{0.2, 2, 20\}$ and $K = 1\times 10^{8}$ episodes. For each $\varepsilon$, we run $20$ simulations. Confidence intervals are the minimum and maximum runs.
Fig.~\ref{fig:exp.randommdp} shows that the learning speed of the optimistic algorithm \algo
is severely impacted by the LDP constraint.
This is consistent with our theoretical results.
The reason for this is the very large confidence intervals that are needed to deal with the noise from
the privacy preserving mechanism that is necessary to guarantee privacy.
While the regret looks almost linear for $\varepsilon=0.2$, the decreasing trend of the per-step regret shows that \algol is learning.
%
Although these experimental results only consider a small MDP, we expect that many of the observations will carry across to larger, more practical settings. However, further experiments are needed to conclusively assess the impact of LDP in large MDPs.
%
Fig.~\ref{fig:exp.randommdp} also shows that \ppsrl performs slightly better than \algo. This is to be expected, since even in the non-private case \psrl usually outperforms optimistic algorithm empirically. 
Finally, Fig.~\ref{fig:comp} compares the mechanisms with different privacy levels and illustrates the empirical impact of the privacy-preserving mechanism on the performance of \algo. We observe empirically that the bounded noise mechanism is the most effective approach, followed by the Laplace mechanism. However, the former suffers from a higher variance in its performance. 

    \begin{figure}[t]
      \begin{minipage}[t]{.42\linewidth}
        \hspace{-.3in}
      \includegraphics[width=1.15\textwidth]{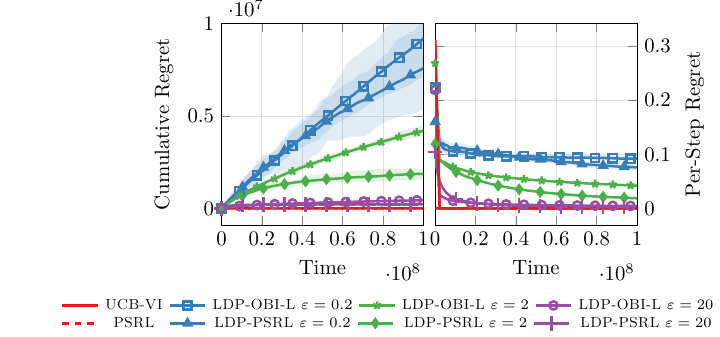}
      \caption{Evaluation of \algo{} with the Laplace mechanism and \ppsrl. 
      \emph{Left)} Cumulative regret. \emph{Right)} per-step regret ($k \mapsto R_{k}/k$).
      }
      \label{fig:exp.randommdp}
      \end{minipage}
      \hfill
      \begin{minipage}[t]{.55\linewidth}
      \hspace{-.3in}\includegraphics[width=1.1\textwidth]{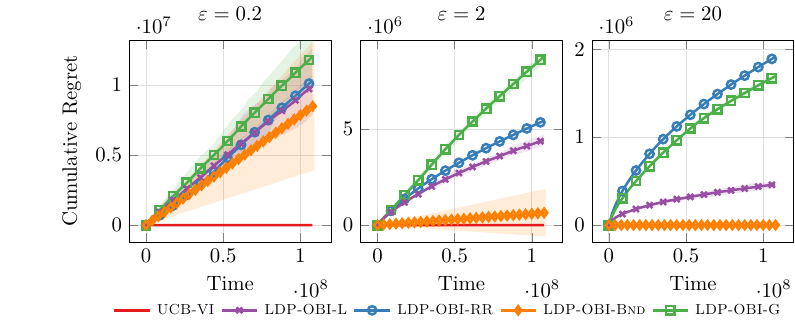}
      \caption{Regret for \algo coupled with different mechanisms.
      For all $\varepsilon$, $\delta = 0.1$ for the Gaussian and Bounded Noise mechanism.}
      \label{fig:comp}
      \end{minipage}
    \end{figure}

\vspace{-0.2cm}
\section{Conclusion} \label{sec:conclusions}
\vspace{-0.2cm}
We have introduced the definition of local differential privacy in RL and designed the first LDP algorithm, \algo, for regret minimization in finite-horizon MDPs.
We provided an intuition why model-based approaches may suffer a higher dependence in the MDP characteristics. Designing a model-free algorithm able to reduce or close the gap with the lower-bound is an interesting technical question for future works.
%
As mentioned in the paper, the shuffling privacy model does not provide any privacy/regret improvement in the strong LDP setting. An interesting direction is to investigate the trade-off between JDP and LDP that can be obtained in RL using shuffling. In particular, we believe that, sacrificing LDP guarantees, it is possible to achieve better regret leveraging variance reduction techniques (that are not helpful in strong LDP settings). Finally, there are other privacy definition that can be interesting for RL. For example, profile-based privacy~\citep{GeumlekC19,AcharyaBKRS20} allows to privatize only specific information or geo-privacy~\citep{AndresBCP13} focuses on privacy between elements that are ``similar''.

\begin{ack}
V. Perchet acknowledges support from the French National Research Agency
(ANR) under grant number \#ANR-$19$-CE$23$-$0026$ as well as the support grant, as well as from
the grant “Investissements d’Avenir” (LabEx Ecodec/ANR-$11$-LABX-$0047$).
\end{ack}
\bibliography{citations}
\bibliographystyle{unsrtnat}

\newpage{}
\appendix
\part{Appendix}
\parttoc

\section{Extended Related Work}

The notion of differential privacy was introduced in~\citep{dwork2006calibrating} and is now a standard in machine learning~\citep[\eg][]{erlingsson2014rappor, dwork2014algorithmic, abowd2018us}.
In stochastic multi-armed bandits, $\epsilon$-DP algorithms have been extensively studied~\citep[see \eg][]{mishra2015Nearly,tossou2016dpmab}. Recently,~\citep{sajed2019optimal} proposed an $\epsilon$-DP algorithm for stochastic multi-armed bandits that achieves the private lower-bound presented in~\citep{shariff2018differentially}. 
In contextual bandits,~\citep{shariff2018differentially} derived an impossibility result for learning under DP by showing a regret lower-bound $\Omega(T)$ for any $(\epsilon,\delta)$-DP algorithm. Instead, they considered the relaxed JDP setting and proposed an optimistic algorithm with sublinear regret and $\epsilon$-JDP guarantees.
Since the contextual bandit problem is an episodic RL problem with horizon $H=1$, this suggests that DP is incompatible with regret minimization in RL as well.

Recently, \emph{local differential privacy}~\citep{duchi2013local} has attracted increasing interest in the bandit literature.
\citep{gajane2018corrupt} were the first to study LDP in stochastic MABs. 
\citep{chen2020locallycombinatorial} extended LDP to combinatorial bandits, and \citep{zheng2020locally, ren2020multi} focused on LDP for MAB 
and contextual bandit. 
Private algorithms for regret minimization have also been investigated in multi-agent bandits (a.k.a. federated learning) in centralized and decentralized settings~\citep[\eg][]{tossou2015privatemultiagent,dubey2020federated, dubey2020private},
and empirical approaches have been considered in~\citep{hannun2019privacy, malekzadeh2020privacy}.
%

In RL,~\citep{balle2016differentially} proposed the first private algorithm for policy evaluation with linear function approximation that ensures privacy with respect to the change of trajectories collected off-policy. \citep{wang2019privacy} considered the RL problem in continuous space, where reward information is protected. They designed a private version of Q-learning with function approximation where privacy with respect to different reward functions is achieved by injecting noise in the value function. 
\citep{ono2020locally} recently studied LDP for actor-critic methods in the context of distributed RL.
None of these works considered regret minimization under privacy constraints.
Regret minimization with privacy guarantees has only been considered in RL recently.
\citep{vietri2020privaterl} designed a private optimistic algorithm for regret minimization with JDP. They proposed a variation of UBEV~\citep{dann2017unifying} 
 using a randomized response mechanism with parameter $\epsilon/H$ to guarantee privacy. 
 Their algorithm \textsc{PUCB} achieves a regret bound $\wt{O}(\sqrt{H^4SAK} + SAH^{3}(S + H)/\varepsilon)$ while enjoying $\varepsilon$-JDP. Compared to the worst case regret of 
UBEV, 
the penalty for JDP privacy is only additive, as 
shown by their lower-bound of $\wt{\Omega}\big(H\sqrt{SAK} + SAH/\varepsilon\big)$.
 \section{Regret Lower Bound (Proof of Thm.~\ref{thm:regret_lower_bound})}\label{app:proof_lower_bound}

Let's consider the following MDP for a given number of states $S$ and actions $A$. The initial state $0$ has $A$ actions which deterministically lead the next state. The MDP is a tree with $A$ children for each node and exactly $S-2$ states.

We denote by $x_{1}, \dots, x_{L}$ the leaves of this tree.
Each leaf can transition to one of the two terminal states denoted by $+$ and $-$, where the agent will receive reward of 1 or 0 respectively, and the agent will stay there until the end of the episode.
There exists a unique action $a^{\star}$ and leaf $x_{i^{\star}}$ such that: $\mathbb{P}(+ \mid x_{i^\star}, a^{\star}) = 1/2 + \Delta$ for a chosen $\Delta$. Each other leaf transitions with equal probability to two states $+$ and $-$ where each has a reward of $1$ and $0$.
All other states have a reward of 0 and every other transition is deterministic.

\begin{figure}[h]
\centering
\includegraphics[width=0.8\linewidth]{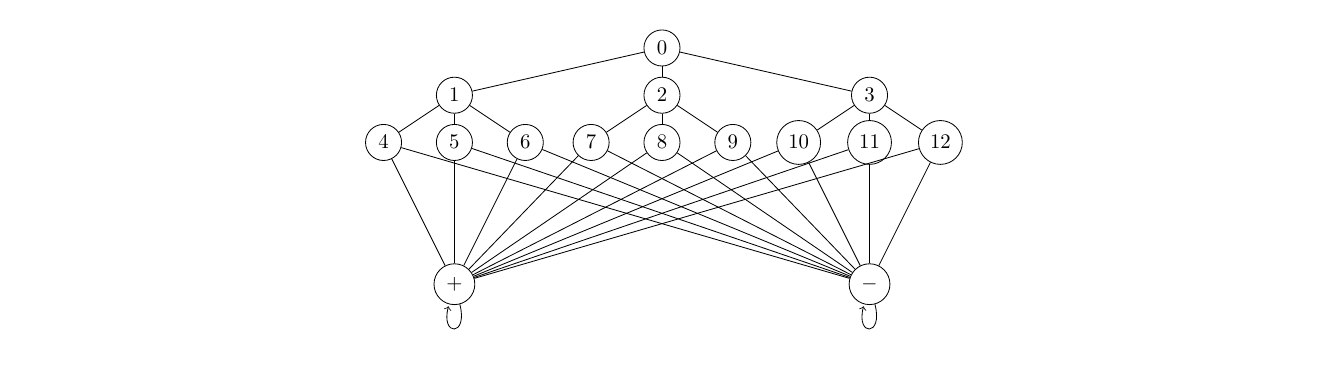}
\caption{{Example of an MDP described in this section with $S = 15$ and $A = 3$}}
\end{figure}

Once the agent arrives at $+$ or $-$, it stay there until the end of the episode.  In addition, we assume that $H \geq 2\ln(S-2)/\ln(A) + 2$.
Let $d>0$ be the depth of the tree, \ie the depth of the tree with $S-2$ nodes is $d-1$ and nodes $+$,$-$ are at depth $d$. Then leaves $x_{1}, \hdots, x_{L}$ are at depth either $d-1$ or $d-2$. Without loss of generality we assume that all $x_{1}, \hdots, x_{L}$ are at depth $d-1$, \ie the number of leaves is $L = A^{d-1} \geq (S-2)/2$, stated otherwise, the tree without the nodes $+$ and $-$ is a perfect $A$-ary tree. In the general case we have that $L \geq (S-2)/2$.

For a policy $\pi$, the value function can be written:
\begin{align}
V^{\pi}(0) = (H-d)\mathbb{P}(s_{d} = +) = (H-d)(1/2 + \Delta\mathbb{P}\left(s_{d-1} = x_{i^{\star}}, a_{d-1} = a^{\star} \right))
\end{align}
Thus the regret can be written as:
\begin{align}
R(K, I) = (H-d)\Delta \Big( K - \underbrace{\sum_{k=1}^{K} \mathbb{P}\left( s_{k,d-1} = x_{i^{\star}}, a_{k,d-1} = a^{\star}\right)}_{:= \mathbb{E}(T(K,I))} \Big)
\end{align}
where $I = (x_{i^{\star}}, a^{\star})$ \textcolor{purple}{is the optimal state action pair} and we define $T(K,I)$ as:
\begin{align}
T(K, I) = \sum_{k=1}^{K} \mathds{1}_{\{s_{k, d-1} = x_{i^{\star}}, a_{k, d-1} = a^{\star}\}}.
\end{align}
$T(K,I)$ is a function of the history observed by the algorithm. Since we consider the LDP setting,   this history can be written as:
\begin{align}
\mathcal{M}(\mathcal{H}_{K}) = \{\mathcal{M}(X_{l}) \mid l\leq K\}
\end{align}
where $X_{l} = \{ (s_{l,h}, a_{l,h}, r_{l,h}) \mid h\leq H\}$ is the trajectory observed by the user for episode $l$ and $\mathcal{M}$ is a privacy mechanism which maintains $\varepsilon$-LDP. Thus $T(K,I)$ is a function of $\mathcal{M}(\mathcal{H}_{K})$. By Lem. A.1 in~\citep{auer2002nonstochastic}:
\begin{align}\label{eq:KL_privacy}
\mathbb{E}(T(K,I)) \leq \mathbb{E}_{0}(T(K,I)) + K\sqrt{\text{KL}\Big(\mathbb{P}_{0}(\mathcal{M}(\mathcal{H}_{K})) \mid \mid \mathbb{P}(\mathcal{M}(\mathcal{H}_{K})) \Big)}
\end{align}
where $\mathbb{E}_{0}$ is the expectation when $\Delta = 0$.
However, because $T(K,I)$ can be seen as a function on the history only, we can use Exercise $14.4$ in~\citep{lattimore2020bandit} which states that for any random variable $Y:\Omega \to [a,b]$ with $(\Omega, \mathcal{F})$ a measurable space, $a<b$ and two distributions $P$ and $Q$ on $\mathcal{F}$, then:
\begin{align}
\left| \int_{w\in \Omega} Y(w)dP(w) - \int_{w\in \Omega} Y(w)dQ(w)\right| \leq (b-a)\sqrt{\frac{\text{KL}(P||Q)}{2}}
\end{align}
In our case the random variable $Y$ is the combination of $T(K,I)$ and the privacy mechanism $\mathcal{M}$ so we have:
\begin{align}\label{eq:KL_no_privacy}
\mathbb{E}(T(K,I)) \leq \mathbb{E}_{0}(T(K,I)) + K\sqrt{\text{KL} \Big( \mathbb{P}_{0}(\mathcal{H}_{K}) \mid \mid \mathbb{P}(\mathcal{H}_{K}) \Big)}
\end{align}
Putting together Eq.~\eqref{eq:KL_privacy} and~\eqref{eq:KL_no_privacy}, we get:
\begin{equation}\label{eq:ET_decomposition}
    \begin{aligned}
    \mathbb{E}(T(K,I)) \leq \mathbb{E}_{0}(T(K,I)) + K \min \Bigg\{
        \underbrace{
            \sqrt{\text{KL}\Big(\mathbb{P}_{0}(\mathcal{M}(\mathcal{H}_{K})) \mid \mid \mathbb{P}(\mathcal{M}(\mathcal{H}_{K})) \Big)}
        }_{\textcircled{1}},& \\
        \underbrace{
            \sqrt{\text{KL} \Big( \mathbb{P}_{0}(\mathcal{H}_{K}) \mid \mid \mathbb{P}(\mathcal{H}_{K}) \Big)}
        }_{\textcircled{2}}&
    \Bigg\}
    \end{aligned}
\end{equation}

\paragraph{Bounding \textcircled{1}.}
Now we bound the KL-divergence between the two measures for the history. Using the chain rule we have:
{\small
\begin{align}
\text{KL}\left(\mathbb{P}_{0}(\mathcal{M}(\mathcal{H}_{K})) \mid \mid \mathbb{P}(\mathcal{M}(\mathcal{H}_{K})) \right) = \sum_{k=1}^{K} \mathbb{E}_{\mathcal{H}_{k-1}\sim \mathbb{P}_{0}}\left( \text{KL}\left(\mathbb{P}_{0}(\cdot | \mathcal{M}(\mathcal{H}_{k-1}))\mid\mid\mathbb{P}(\cdot | \mathcal{M}(\mathcal{H}_{k-1}))\right)\right)
\end{align}}
But because $\mathcal{M}$ is an $\varepsilon$-LDP mechanism, Thm.~$1$ in~\citep{duchi2013local} ensures that:
\begin{align}
\text{KL}\left(\mathbb{P}_{0}(\cdot | \mathcal{M}(\mathcal{H}_{k-1}))\mid\mid\mathbb{P}(\cdot | \mathcal{M}(\mathcal{H}_{k-1}))\right) \leq 4(\exp(\varepsilon) - 1)^{2}\text{KL}\left(\mathbb{P}_{0}(\cdot | \mathcal{H}_{k-1})\mid\mid\mathbb{P}(\cdot | \mathcal{H}_{k-1})\right) \label{eq:klepsM}
\end{align}
Additionally, the KL-divergence can be written as:
{\small\begin{align}
\text{KL}\left(\mathbb{P}_{0}(\cdot | \mathcal{H}_{k-1})\mid\mid\mathbb{P}(\cdot | \mathcal{H}_{k-1})\right)  &= \sum_{h=1}^{H} \mathbb{E}_{X_{k}\sim \mathbb{P}_{0}}\left( \ln\left(\frac{\mathbb{P}_{0}(s_{k,h},a_{k,h},r_{k,h}\mid \mathcal{H}_{k-1}, (s_{k,j},a_{k,j},r_{k,j})_{j\leq h-1})}{\mathbb{P}(s_{k,h},a_{k,h},r_{k,h} \mid \mathcal{H}_{k-1},(s_{k,j},a_{k,j},r_{k,j})_{j\leq h-1})}\right)\right)
\end{align}}
where $X_{k} = \{ (s_{k,h}, a_{k,h}, r_{k,h})\mid h\leq H\}$ is a trajectory sampled from the MDP with the transitions distributed according to $\mathbb{P}_{0}$ and for each step $h$, $s_{k,h}$ is a state, $a_{k,h}$ an action and $r_{k,h}$ the reward associated with $(s_{k,h}, a_{k,h})$.

Therefore for a step $h \geq 1$,
{\small\begin{align}
\ln\left(\mathbb{P}_{0}(s_{k,h},a_{k,h},r_{k,h} \mid \mathcal{H}_{k-1}, (s_{k,j},a_{k,j},r_{k,j})_{j\leq h-1})\right)= \ln\left(\mathbb{P}_{0}(s_{k,h} \mid \mathcal{H}_{k-1}, (s_{k,j},a_{k,j},r_{k,j})_{j\leq h-1})\right)&\nonumber\\
+ \ln\left(\mathbb{P}_{0}(a_{k,h} \mid \mathcal{H}_{k-1}, (s_{k,j},a_{k,j},r_{k,j})_{j\leq h-1}, s_{k,h})\right)& \nonumber\\
+ \ln\left(\mathbb{P}_{0}(r_{k,h} \mid \mathcal{H}_{k-1}, (s_{k,j},a_{k,j},r_{k,j})_{j\leq h-1}, s_{k,h}, a_{k,h})\right)& \nonumber
\end{align}}
\todocout{Is $\ln$ here the same as $\ln$? shall we change one so notations are consistent?}
\todoeout{yes done}
By the Markov property of the environment:
\begin{align}
\ln\left(\mathbb{P}_{0}(s_{k,h} \mid \mathcal{H}_{k-1}, (s_{k,j},a_{k,j},r_{k,j})_{j\leq h-1})\right) = \ln\left(\mathbb{P}_{0}(s_{k,h} \mid s_{k,h-1},a_{k,h-1})\right)
\end{align}
Also, since the reward only depends on the current state-action pair:
\begin{align}
\ln\left(\mathbb{P}_{0}(r_{k,h} \mid \mathcal{H}_{k-1}, (s_{k,j},a_{k,j},r_{k,j})_{j\leq h-1}, s_{k,h}, a_{k,h})\right)= \ln\left(\mathbb{P}_{0}(r_{k,h} \mid s_{k,h}, a_{k,h})\right).
\end{align}
The same results holds for $\mathbb{P}$, thus:
\begin{equation}
    \begin{aligned}
        \text{KL}&\left(\mathbb{P}_{0}(\cdot | \mathcal{H}_{k-1})\mid\mid\mathbb{P}(\cdot | \mathcal{H}_{k-1})\right)
        = \sum_{h=1}^{H} \mathbb{E}_{X_{k}\sim \mathbb{P}_{0}}\Bigg( \ln\left(\frac{\mathbb{P}_{0}(s_{k,h}\mid s_{k,h-1}, a_{k,h-1})}{\mathbb{P}_{}(s_{k,h}\mid s_{k,h-1}, a_{k,h-1})}\right) \\
        &+ \ln\left(\frac{\mathbb{P}_{0}(a_{k,h}\mid \mathcal{H}_{k-1}, (s_{k,j},a_{k,j},r_{k,j})_{j\leq h-1}, s_{k,h})}{\mathbb{P}(a_{k,h}\mid \mathcal{H}_{k-1}, (s_{k,j},a_{k,j},r_{k,j})_{j\leq h-1}, s_{k,h})}\right) + \ln\left(\frac{\mathbb{P}_{0}(r_{k,h}\mid s_{k,h}, a_{k,h})}{\mathbb{P}(r_{k,h}\mid s_{k,h}, a_{k,h})}\right)\Bigg)
    \end{aligned}
\end{equation}
But for $\mathbb{P}$ and $\mathbb{P}_{0}$ the rewards are distributed accordingly to the same distribution hence $\ln\left(\frac{\mathbb{P}_{0}(r_{k,h}\mid s_{k,h}, a_{k,h})}{\mathbb{P}(r_{k,h}\mid s_{k,h}, a_{k,h})}\right)= 0$ for each $h\leq H$. Also, the action taken at each step depends only the history of data and the current state, thus
 $\ln\left(\frac{\mathbb{P}_{0}(a_{k,h}\mid \mathcal{H}_{k-1}, (s_{k,j},a_{k,j},r_{k,j})_{j\leq h-1})}{\mathbb{P}(a_{k,h}\mid \mathcal{H}_{k-1}, (s_{k,j},a_{k,j},r_{k,j})_{j\leq h-1})}\right) = 0$.
 Lastly, transition dynamics between $\mathbb{P}$ and $\mathbb{P}_{0}$ only differ when at step $d-1$ thus for all $h\neq d-1$ \todocout{$d$ or $d-1$?}\todoeout{$d-1$, chagned it}, $\ln\left(\frac{\mathbb{P}_{0}(s_{k,h}\mid s_{k,h-1}, a_{k,h-1})}{\mathbb{P}_{0}(s_{k,h}\mid s_{k,h-1}, a_{k,h-1})}\right) =0$.
  Overall, we get:
{\small\begin{align*}
\text{KL}&\left(\mathbb{P}_{0}(\cdot | \mathcal{H}_{k-1})\mid\mid\mathbb{P}(\cdot | \mathcal{H}_{k-1})\right)= \sum_{l=1}^{L}\sum_{a=1}^{A}\sum_{j\in \{-, +\}}\mathbb{E}_{X_{k}\sim \mathbb{P}_{0}}
\Bigg( \ln\left(\frac{\mathbb{P}_{0}(j\mid x_{l}, a)}{\mathbb{P}(j\mid x_{l}, a)}\right)\mathds{1}_{\Big\{\begin{subarray}{l}s_{k,d-1} = x_{l}, \\a_{k,d-1} = a,\\ s_{k,d} = j\end{subarray}\Big\}} \Bigg)
\end{align*}}

Finally, for $j\in \{-,+\}$, $x_{l} \neq x_{i^{\star}}$ and $a \neq a^{\star}$, $\mathbb{P}(j\mid x_{l}, a) = \mathbb{P}_{0}(j\mid x_{l}, a)$. Hence,
\begin{align}\label{eq:final_bound_kl_per_epsiode}
\text{KL}&\left(\mathbb{P}_{0}(\cdot | \mathcal{H}_{k-1})\mid\mid\mathbb{P}(\cdot | \mathcal{H}_{k-1})\right)= \frac{1}{2}\ln\left(\frac{1}{1-4\Delta^{2}}\right)\mathbb{E}_{X_{k}\sim\mathbb{P}_{0}}\left( \mathds{1}_{\{s_{k,d-1} = x_{i^{\star}}, a_{k,d-1} = a^{\star}\}}\right)
\end{align}
where we have used $\mathbb{P}(+\mid x_{i^{\star}}, a^{\star}) = \frac{1}{2} + \Delta$, $\mathbb{P}_{0}(+\mid x_{i^{\star}}, a^{\star}) = \frac{1}{2}$, $\mathbb{P}(-\mid x_{i^{\star}}, a^{\star}) = \frac{1}{2} - \Delta$ and $\mathbb{P}_{0}(-\mid x_{i^{\star}}, a^{\star}) = \frac{1}{2}$.

Therefore combining~\eqref{eq:klepsM}~and~\eqref{eq:final_bound_kl_per_epsiode} and summing over the episodes, we get:
{\small
\begin{equation}\label{eq:bound_private_KL}
\begin{aligned}
\text{KL}\Big(\mathbb{P}_{0}(\mathcal{M}(\mathcal{H}_{K})) \mid \mid \mathbb{P}(\mathcal{M}(\mathcal{H}_{K})) \Big) &\leq 2(e^{\varepsilon} - 1)^{2}\ln\left(\frac{1}{1 - 4\Delta^{2}}\right)\sum_{k=1}^{K} \mathbb{P}_{0}\left( s_{k,d-1} = x_{i^{\star}}, a_{k,d-1} = a^{\star} \right)\\
&=2(e^{\varepsilon} - 1)^{2}\ln\left(\frac{1}{1 - 4\Delta^{2}}\right)\mathbb{E}_{0}(T(K,I))
\end{aligned}
\end{equation}}

\paragraph{Bounding \textcircled{2}.}
Using again the chain rule of the KL-divergence, we have that:
\begin{align}
\text{KL}\left(\mathbb{P}_{0}(\mathcal{H}_{K}) \mid \mid \mathbb{P}(\mathcal{H}_{K}) \right) = \sum_{k=1}^{K} \mathbb{E}_{\mathcal{H}_{k-1}\sim \mathbb{P}_{0}}\left( \text{KL}\left(\mathbb{P}_{0}(\cdot | \mathcal{H}_{k-1})\mid\mid\mathbb{P}(\cdot |\mathcal{H}_{k-1})\right)\right)
\end{align}
Therefore, using Eq.~\eqref{eq:final_bound_kl_per_epsiode}, we have:
{\small \begin{equation}\label{eq:bound_non_private_KL}
  \begin{aligned}
  \text{KL}\left(\mathbb{P}_{0}(\mathcal{H}_{K}) \mid \mid \mathbb{P}(\mathcal{H}_{K}) \right)&= \sum_{k=1}^{K} \mathbb{E}_{\mathcal{H}_{k-1}\sim \mathbb{P}_{0}}\left(\frac{1}{2}\ln\left(\frac{1}{1-4\Delta^{2}}\right)\mathbb{E}_{X_{k}\sim\mathbb{P}_{0}}\left(\mathds{1}_{\Big\{\begin{subarray}{l}s_{k,d-1} = x_{i^{\star}},\\ a_{k,d-1} = a^{\star}\end{subarray}\Big\}}  \right)\right) \\
  &= \frac{1}{2}\ln\left(\frac{1}{1-4\Delta^{2}}\right)\mathbb{E}_{0}(T(K,I))
  \end{aligned}
\end{equation}}
\paragraph{Finishing the proof.}
Hence using Eq.~\eqref{eq:bound_private_KL} and Eq.~\eqref{eq:bound_non_private_KL} in Eq.~\eqref{eq:ET_decomposition}:
\begin{align}
\mathbb{E}(T(K,I)) \leq \mathbb{E}_{0}(T(K,I)) + K\min\left\{\sqrt{2}(e^{\varepsilon} - 1), \frac{1}{\sqrt{2}}\right\}\sqrt{\mathbb{E}_{0}(T(K,I))\ln\left(\frac{1}{1-4\Delta^{2}}\right)}
\end{align}
Now, let's assume that $I= (x_{i^{\star}}, a^{\star})$ is distributed uniformly over $\{x_{1}, \dots, x_{L}\}\times [A]$. That is to say, that the leaf $i^{\star}\sim \mathcal{U}( [L])$ and given the realization of $i^{\star}$, $a^{\star}$ is drawn uniformly in the action set of node $x_{i^{\star}}$ \ie $a^{\star} \sim\mathcal{U}( [A])$.
We denote the expectation over the random variable $(x_{i^{\star}}, a^{\star})$ by $\mathbb{E}_{I}$. It then holds that:
\begin{align}\label{eq:uniform_I}
\mathbb{E}_{I}\mathbb{E}_{0}(T(K,I)) = \mathbb{E}_{0}\sum_{k=1}^{K} \sum_{l=1}^{L} \sum_{a = 1}^{A} \frac{1}{LA}\mathds{1}_{\{s_{k,d-1} = s, a_{k,d-1} = a\}} = \frac{K}{LA}
\end{align}
Therefore thanks to Jensen's inequality the regret is lower-bounded by:
{\small\begin{align}
\mathbb{E}_{I} R(K, I) \geq (H-d)\Delta K \left( 1 - \frac{1}{LA} - \min\left\{\sqrt{2}(e^{\varepsilon} - 1), \frac{1}{\sqrt{2}}\right\}\sqrt{\frac{K}{LA}\ln\left(1 + \frac{4\Delta^{2}}{1 - 4\Delta^{2}}\right)}\right)
\end{align}}
Therefore for $LA\geq 2$, $K\geq \frac{LA}{\min\left\{8(e^{\varepsilon} - 1), 4\right\}^{2}}$ and choosing $\Delta = \sqrt{\frac{LA}{K}}\times \frac{1}{16\sqrt{2}\min\left\{(e^{\varepsilon} - 1), \frac{1}{2}\right\}}$ we get that:
\begin{align*}
\min\left\{\sqrt{2}(\exp(\varepsilon) - 1), \frac{1}{\sqrt{2}}\right\}\sqrt{\frac{K}{LA}\ln\left(1 + \frac{4\Delta^{2}}{1 - 4\Delta^{2}}\right)} \leq \frac{1}{4}
\end{align*}
Hence:
\begin{align}
\max_{I \in \{x_{1}, \dots, x_{L}\}\times [A]} R(K, I) \geq \mathbb{E}_{I} R(K, I) \geq \frac{(H-d)\sqrt{KLA}}{64\min\left\{(\exp(\varepsilon) - 1), \frac{1}{2}\right\}}
\end{align}
And because $I$ is a finite random variable there exist $I^{\star}$ such that $\max_{I \in \{x_{1}, \dots, x_{L}\}\times [A]} R(K, I) = R(K, I^{\star})$.
\begin{align}
R(K, I^{\star}) \geq \frac{(H-d)\sqrt{KLA}}{64\min\left\{(\exp(\varepsilon) - 1), \frac{1}{2}\right\}}
\end{align}
Thus we have that there exists an MDP such that its frequentist regret is $\Omega\left(\frac{H\sqrt{SAK}}{\min\{1, \exp(\varepsilon)-1\}}\right)$.

%

\section{Concentration under Local Differential Privacy (Proof of Prop.~\ref{prop:concentration_ldp}):}\label{app:proof_concentration}

In this subsection, we proceed with the proof of Prop.~\ref{prop:concentration_ldp} (recalled below).

\begin{proposition*}
  For any $\varepsilon_{0}>0$, $\delta_{0}\geq0$, $\delta >0$, $\alpha > 1$ and episode $k$, using mechanism $\mathcal{M}$ satisfying Asm.~\ref{assumption:concentration_privacy}, then with probability at least $1-2\delta$, for any $(s,a)\in\mathcal{S}\times\mathcal{A}$
  {\small\begin{align*}
    \left|r(s,a) - \wt{r}_k(s,a)\right| \leq \beta_k^r(s,a) = \sqrt{\frac{2\ln\left(\frac{4\pi^{2}SAHk^{3}}{3\delta}\right)}{\wt{N}_{k}^{r}(s,a) + \alpha c_{k,2}(\varepsilon_{0},\delta_{0},\delta)}} +\frac{(\alpha + 1)c_{k,2}(\varepsilon_{0}, \delta_{0}, \delta) + c_{k,1}(\varepsilon_{0}, \delta_{0},\delta)}{\wt{N}_{k}^{r}(s,a) + \alpha c_{k,2}(\varepsilon_{0}, \delta_{0}, \delta)}\\
    \|p(\cdot| s,a) - \wt{p}_{k}(\cdot| s,a)\|_{1}\leq \beta_k^p(s,a) = 
    \sqrt{\frac{14S\ln\left(\frac{4\pi^{2}SAHk^{3}}{3\delta}\right)}{\wt{N}_{k}^{p}(s,a) + \alpha c_{k,3}(\varepsilon_{0},\delta_{0}, \delta)}}+ \frac{Sc_{k,4}(\varepsilon_{0},\delta_{0}, \delta)}{\wt{N}_{k}^{p}(s,a) + \alpha c_{k,3}(\varepsilon_{0},\delta_{0}, \delta)} +
    \\\frac{(\alpha + 1)c_{k,3}(\varepsilon_{0},\delta_{0}, \delta)}{\wt{N}_{k}^{p}(s,a) + \alpha c_{k,3}(\varepsilon_{0},\delta_{0}, \delta)}
  \end{align*}}
\end{proposition*}

\begin{proof}
On the event that all inequalities of Def.~\ref{assumption:concentration_privacy} holds, we have:
\begin{align}\label{eq:deviation_reward}
\left| \frac{\wt{R}_k(s,a)}{\wt{N}_{k}^{r}(s,a) + \alpha c_{k,2}(\varepsilon_{0}, \delta_{0}, \delta)}  - \frac{R_{k}(s,a)}{\wt{N}_{k}^{r}(s,a) + \alpha c_{k,2}(\varepsilon_{0}, \delta_{0}, \delta)}\right| \leq \frac{c_{k,1}(\varepsilon_{0}, \delta_{0}, \delta)}{\wt{N}_{k}^{r}(s,a) + \alpha c_{k,2}(\varepsilon_{0}, \delta_{0}, \delta)}
\end{align}
since  $\wt{N}_{k}^{r}(s,a) + \alpha c_{k,2}(\varepsilon_{0}, \delta_{0}, \delta) > N_k^k(s,a) \geq 0$ with $\alpha > 1$.
But, we also have that with probability $1-\delta$:
{\small\begin{align}
\Bigg| &\frac{R_{k}(s,a)}{\wt{N}_{k}^{r}(s,a) + \alpha c_{k,2}(\varepsilon_{0}, \delta_{0}, \delta)} - r(s,a) \Bigg| \leq \Bigg|r(s,a)\left(\frac{N_{k}^{r}(s,a)}{\wt{N}_{k}^{r}(s,a) + \alpha c_{k,2}(\varepsilon_{0}, \delta_{0}, \delta)} - 1\right)\Bigg|\\
&\hspace{5cm}+\Bigg|\frac{N_{k}^{r}(s,a)}{\wt{N}_{k}^{r}(s,a) + \alpha c_{k,2}(\varepsilon_{0}, \delta_{0}, \delta)} \times
\underbrace{
\left(\frac{R_{k}(s,a)}{N_{k}^{r}(s,a)} - r(s,a)\right)
}_{:= \wb{r}_k(s,a) - r(s,a)}
\Bigg| \nonumber\\
&\leq \frac{N_{k}^{r}(s,a)}{\wt{N}_{k}^{r}(s,a) + \alpha c_{k,2}(\varepsilon_{0}, \delta_{0}, \delta)}\frac{L(\delta)}{\sqrt{N_{k}^{r}(s,a)}} + r(s,a)\left|  1 - \frac{N_{k}^{r}(s,a)}{\wt{N}_{k}^{r}(s,a) + \alpha c_{k,2}(\varepsilon_{0}, \delta_{0}, \delta)} \right|  \\
&\leq \frac{L(\delta)\sqrt{N_{k}^{r}(s,a)}}{\wt{N}_{k}^{r}(s,a) + \alpha c_{k,2}(\varepsilon_{0}, \delta_{0}, \delta)} + \frac{(\alpha + 1)c_{k,2}(\varepsilon_{0}, \delta_{0}, \delta)}{\wt{N}_{k}^{r}(s,a) + \alpha c_{k,2}(\varepsilon_{0}, \delta_{0}, \delta)}
\label{eq:concentration_reward_biased}
\end{align}}
where the second inequality follows from Chernoff-Hoeffding bound on the empirical non-private rewards with $L(\delta) = \sqrt{2\ln(4\pi^{2}SAHk^{3}/3\delta)}$, and we use Def.~\ref{assumption:concentration_privacy} for the last.
Furthermore:
{\small\begin{equation}\label{eq:bound_nb_visits_reward}
\frac{L(\delta)\sqrt{N_{k}^{r}(s,a)}}{\wt{N}_{k}^{r}(s,a) + \alpha c_{k,2}(\varepsilon_{0}, \delta_{0}, \delta)} \leq \frac{L(\delta)\sqrt{\wt{N}_{k}^{r}(s,a) + c_{k,2}(\varepsilon_{0}, \delta_{0}, \delta)}}{\wt{N}_{k}^{r}(s,a) + \alpha c_{k,2}(\varepsilon_{0}, \delta_{0}, \delta)} \leq \frac{L(\delta)}{\sqrt{\wt{N}_{k}^{r}(s,a) + \alpha c_{k,2}(\varepsilon_{0}, \delta_{0}, \delta)}}
\end{equation}}
Therefore combining Eq.~\eqref{eq:deviation_reward}, \eqref{eq:concentration_reward_biased} and \eqref{eq:bound_nb_visits_reward}, we have:
{\small\begin{align*}
\left| \frac{\wt{R}_k(s,a)}{\wt{N}_{k}^{r}(s,a) + \alpha c_{k,2}(\varepsilon_{0}, \delta_{0}, \delta)}  - r(s,a)\right|\leq \frac{c_{k,1}(\varepsilon_{0}, \delta_{0}, \delta) + (\alpha + 1)c_{k,2}(\varepsilon_{0}, \delta_{0}, \delta)}{\wt{N}_{k}^{r}(s,a) + \alpha c_{k,2}(\varepsilon_{0}, \delta_{0}, \delta)} &\\
+ \frac{L(\delta)}{\sqrt{\wt{N}_{k}^{r}(s,a) + \alpha c_{k,2}(\varepsilon_{0}, \delta_{0}, \delta)}}&
\end{align*}}
thus proving the first statement of the proposition.
Now, we bound the deviation between the private estimate $\wt{p}_{k}$ and the true transition dynamics $p$.
First, because $\alpha>1$, we have that $\sum_{s'} \wt{N}_{k}^{p}(s,a,s') + \alpha c_{k,3}(\varepsilon_{0}, \delta_{0},\delta) \geq \sum_{s'}N_{k}^{p}(s,a,s') + (\alpha - 1)c_{k,3}(\varepsilon_{0}, \delta_{0},\delta) > 0$.
%
We start by decomposing the error as
{\small\begin{equation}
  \begin{aligned}
  &\sum_{s'\in \mathcal{S}} \left| \wt{p}(s'|s,a) - p(s'| s,a) \right| =
  \sum_{s'\in \mathcal{S}}\left| \frac{\wt{N}_{k}^{p}(s,a,s')}{\sum_{s'}\wt{N}_{k}^{p}(s,a,s') + \alpha c_{k,3}(\varepsilon_{0}, \delta_{0}, \delta)}  - p(s'| s,a) \right|\\
  &\leq \underbrace{
    \sum_{s'\in \mathcal{S}}\left|\frac{N_{k}^{p}(s,a,s')}{\sum_{s'}\wt{N}_{k}^{p}(s,a,s') + \alpha c_{k,3}(\varepsilon_{0}, \delta_{0}, \delta)}  - p(s'\mid s,a) \right|
  }_{\textcircled{1}} +
  \underbrace{
    \sum_{s'\in \mathcal{S}} \left| \frac{\wt{N}_{k}^{p}(s,a,s') - N_{k}^{p}(s,a,s')}{\sum_{s'}\wt{N}_{k}^{p}(s,a,s') + \alpha c_{k,3}(\varepsilon_{0}, \delta_{0}, \delta)}\right|
  }_{\textcircled{2}}
\end{aligned}
\end{equation}}
Recall that $\sum_{s'}\wt{N}_{k}^{p}(s,a,s') = \wt{N}_k^p(s,a)$ and $\sum_{s'}{N}_{k}^{p}(s,a,s') = {N}_k^p(s,a)$ and define $\wb{p}_{k}(\cdot| s,a)=\frac{N_k^p(s,a,\cdot)}{N_k^p(s,a)}$.
Therefore:
{\small
\begin{equation}\label{eq:concentration_transisition}
  \begin{aligned}
\textcircled{1} &= \sum_{s'\in \mathcal{S}} \left| \frac{N_{k}^{p}(s,a,s')}{N_{k}^{p}(s,a)}\frac{N_{k}^{p}(s,a)}{\wt{N}_k^p(s,a)+\alpha c_{k,3}(\varepsilon_{0}, \delta_{0}, \delta)}  - p(s'\mid s,a) \right| \nonumber\\
&= \sum_{s'} \Bigg|
\underbrace{
\frac{\left(\frac{N_{k}^{p}(s,a,s')}{N_{k}^{p}(s,a)} - p(s'|s,a)\right)N_{k}^{p}(s,a)}{\wt{N}_k^p(s,a)+\alpha c_{k,3}(\varepsilon_{0}, \delta_{0}, \delta)}
}_{\color{red}>0}  + p(s'| s,a)\left(\frac{N_{k}^{p}(s,a)}{\wt{N}_k^p(s,a)+\alpha c_{k,3}(\varepsilon_{0}, \delta_{0}, \delta)}  - 1\right)\Bigg|\\
&\leq \sum_{s'} \bigg(p(s'| s,a) \frac{(\alpha + 1)c_{k,3}(\varepsilon_{0}, \delta_{0}, \delta)}{\wt{N}_k^p(s,a) + \alpha c_{k,3}(\varepsilon_{0}, \delta_{0}, \delta)}  \bigg) + \frac{N_{k}^{p}(s,a)\|\wb{p}_{k}(\cdot| s,a) - p(\cdot|s,a)\|_{1}}{\wt{N}_k^p(s,a) + \alpha c_{k,3}(\varepsilon_{0}, \delta_{0}, \delta)} \nonumber \\
&\stackrel{(a)}{\leq}  \frac{(\alpha + 1)c_{k,3}(\varepsilon_{0}, \delta_{0}, \delta)}{\wt{N}_k^p(s,a) + \alpha c_{k,3}(\varepsilon_{0}, \delta_{0}, \delta)} + \frac{N_{k}^{p}(s,a)}{\wt{N}_k^p(s,a) + \alpha c_{k,3}(\varepsilon_{0}, \delta_{0}, \delta)}\frac{L(\delta)}{\sqrt{N_{k}^{p}(s,a)}} \nonumber \\
&\leq \frac{(\alpha + 1)c_{k,3}(\varepsilon_{0}, \delta_{0}, \delta)}{\wt{N}_k^p(s,a) + \alpha c_{k,3}(\varepsilon_{0}, \delta_{0}, \delta)} + \frac{L(\delta)}{\sqrt{\wt{N}_k^p(s,a) + \alpha c_{k,3}(\varepsilon_{0}, \delta_{0}, \delta)}}
  \end{aligned}
\end{equation}}
where $L(\delta) = \sqrt{14S\ln(4\pi^{2}SAHk^{3}/3\delta)}$ and inequality $(a)$ follows from the Weissman inequality \citep{weissman2003inequality}, and we have again used the fact that the inequalities in Def.~\ref{assumption:concentration_privacy} hold. 

\todocout{we might need to be a bit careful here. Probably want to do the above analysis on the event that $\tilde N + \alpha c_{k,3} >0$? or is this already covered in the event everything in def 2 holds? if so state this}
\todoeout{I have added an other comment that because $\alpha>1$ on the event of Def.2 then $\sum_{s'} \wt{N}_{k}^{p}(s,a,s') + \alpha c_{k,3}(\varepsilon_{0}, \delta_{0},\delta) \geq \sum_{s'}N_{k}^{p}(s,a,s') + (\alpha - 1)c_{k,3}(\varepsilon_{0}, \delta_{0},\delta) > 0$}
In addition,we have:
\begin{align}\label{eq:biased_visits}
  \textcircled{2}
\leq \sum_{s'\in \mathcal{S}}  \frac{\left|c_{k,4}(\varepsilon_{0}, \delta_{0}, \delta)\right|}{\wt{N}_{k}^{p}(s,a) + \alpha c_{k,3}(\varepsilon_{0}, \delta_{0}, \delta)} = \frac{Sc_{k,4}(\varepsilon_{0}, \delta_{0}, \delta)}{\wt{N}_{k}^{p}(s,a) + \alpha c_{k,3}(\varepsilon_{0}, \delta_{0}, \delta)}
\end{align}
Hence putting together Eq.~\eqref{eq:biased_visits} and Eq.~\eqref{eq:concentration_transisition}, we have:
{\small\begin{equation}
\begin{aligned}
  \sum_{s'\in \mathcal{S}} \left| \frac{\wt{N}_{k}^{p}(s,a,s')}{\wt{N}_{k}^{p}(s,a) + \alpha c_{k,3}(\varepsilon_{0}, \delta_{0}, \delta)}  - p(s'\mid s,a) \right| \leq \frac{Sc_{k,4}(\varepsilon_{0}, \delta_{0}, \delta) + (\alpha + 1)c_{k,3}(\varepsilon_{0}, \delta_{0}, \delta)}{\wt{N}_{k}^{p}(s,a) + \alpha c_{k,3}(\varepsilon_{0}, \delta_{0}, \delta)}  \\
  +\frac{L(\delta)}{\sqrt{\wt{N}_{k}^{p}(s,a) + \alpha c_{k,3}(\varepsilon_{0}, \delta_{0}, \delta)}}&
  \end{aligned}
\end{equation}}\end{proof}

\section{Regret Upper Bound (Proof of Thm.~\ref{thm:regret_any_mechanism})}\label{app:proof_regret_upper_bound}

In this section, we prove Thm~\ref{thm:regret_any_mechanism}, which we recall below.
\begin{theorem*}
    For any privacy mechanism $\mathcal{M}$ satisfying Asm.~\ref{assumption:concentration_privacy} with $\varepsilon>0$, $\delta_{0}\geq 0$, and for any $\delta>0$ the regret of \ldpucbvi is bounded with probability at least $1-\delta$ by:
    \begin{equation}\label{eq:regret}
    \begin{aligned}
    \Delta(K) \leq \tilde{\mathcal{O}}\Bigg( \underbrace{HS\sqrt{AT}}_{\text{\ding{182}}} + SAH^{2}c_{K,3}\left(\varepsilon ,\delta_{0}, \frac{3\delta}{2\pi^{2}K^{2}}\right) +
    H^{2}S^{2}Ac_{K,4}\left(\varepsilon ,\delta_{0},\frac{3\delta}{2\pi^{2}K^{2}}\right)&\\
    + SAH c_{K,2}\left(\varepsilon ,\delta_{0},\frac{3\delta}{2\pi^{2}K^{2}}\right) +
    SAH c_{K,1}\left(\varepsilon ,\delta_{0},\frac{3\delta}{2\pi^{2}K^{2}}\right)   \Bigg)&
    \end{aligned}
    \end{equation}
    The combination of $\mathcal{M}$ and \algo is also $(\varepsilon, \delta_{0})$-LDP.
\end{theorem*}

\paragraph{Good Event:}
Before proceeding the proof of the regret we define a good event under which all concentration inequalities holds with probability at least $1-\delta$.
First, we define the event that all inequalities from Def.~\ref{assumption:concentration_privacy} holds. Let:
\begin{align*}
&L_{1,k} = \bigcap_{s,a}\left\{\left| \wt{R}_{k}(s,a) - R_{k}(s,a) \right| \leq  c_{k,1}(\varepsilon_0, \delta_{0}, 3\delta/2k^{2}\pi^{2})\right\}\\
&L_{2,k} =  \bigcap_{s,a}\left\{\left|\wt{N}_{k}^{r}(s,a) - N_{k}^{r}(s,a) \right| \leq c_{k,2}(\varepsilon_0, \delta_{0}, 3\delta/2k^{2}\pi^{2})\right\}\\
&L_{3,k} = \bigcap_{s,a} \left\{ \left|  \sum_{s'} N_{k}^{p}(s,a,s') - \sum_{s‘}\wt{N}_{k}^{p}(s,a,s')\right| \leq c_{k,3}(\varepsilon_0, \delta_{0}, 3\delta/2k^{2}\pi^{2})\right\} \\
&L_{4,k} = \bigcap_{s,a,s'} \left\{\left| N_{k}^{p}(s,a,s') - \wt{N}_{k}^{p}(s,a,s')\right| \leq c_{k,4}(\varepsilon_0, \delta_{0}, 3\delta/2k^{2}\pi^{2})\right\}
\end{align*}
then thanks to Def.~\ref{assumption:concentration_privacy} we have :
\begin{align}
\mathbb{P}\left(\bigcup_{k=1}^{+\infty} L_{1,k}^{c}\cup L_{2,k}^{c} \cup L_{3,k}^{c}\cup L_{4,k}^{c}\right) \leq \sum_{k=1}^{+\infty} \frac{3\delta}{\pi^{2}k^{2}} = \frac{\delta}{4}
\end{align}

In addition, for all $k\in \mathbb{N}^{\star}$, we can define $\wb{r}_{k}(s,a) = R_k(s,a)/N_k^r(s,a)$ and $\wb{p}_{k}= N_k^p(s,a,s')/ \sum_{s'} N_k^p(s,a,s')$ as the empirical reward and transition probability computed with the non-private counters.
Note that in this case $N_k(s,a):=N_k^r(s,a) = \sum_{s'} N_k^p(s,a,s')$.
We also define $\wb{\beta}_{k}^{r}(\delta,s,a) = \sqrt{\frac{2\ln(1/\delta)}{N_{k}(s,a)}}$ and $\wb{\beta}_{k}^{p}(\delta,s,a) = \sqrt{\frac{14S\log(1/\delta)}{N_{k}(s,a)}}$.
as the size of the confidence intervals using Hoeffding and Weissman inequalities. Thus, we get:
{\small\begin{align*}
\mathbb{P}&\left(\bigcup_{k=1}^{+\infty}\bigcup_{s,a} \left| \wb{r}_{k}(s,a) - r(s,a)\right|\geq \wb{\beta}_{k}^{r}(3\delta/4\pi^{2}SAHk^{3},s,a) \right) \\
&\leq \sum_{k=1}^{+\infty}\sum_{s,a} \mathbb{P}\left( \left| \wb{r}_{k}(s,a) - r(s,a)\right|\geq \sqrt{\frac{2\ln(4\pi^{3}SAHk^{3}/3\delta)}{N_{k}(s,a)}}\right)\\
&\leq \sum_{k=1}^{+\infty}\sum_{s,a}\sum_{n=0}^{kH} \mathbb{P}\left( \left| \wb{r}_{k}(s,a) - r(s,a)\right|\geq \sqrt{\frac{2\ln(4\pi^{2}SAHk^{3}/3\delta)}{n}}\right)\leq \sum_{k=1}^{+\infty}\sum_{s,a}\sum_{n=0}^{kH} \frac{3\delta}{4\pi^{2}SHAk^{3}} \leq \frac{\delta}{8}
\end{align*}}
A similar result holds for the transition dynamics, i.e.,:
\todocout{should this be in L1 norm?}
\todoeout{yes}
\begin{align}
&\mathbb{P}\left(\bigcup_{k=1}^{+\infty}\bigcup_{s,a} \left|\left| \wb{p}_{k}(\cdot|s,a) - p(\cdot| s,a)\right|\right|_{1}\geq \wb{\beta}_{k}^{p}(3\delta/4\pi^{2}SAHk^{3},s,a) \right) \leq \frac{\delta}{8}
\end{align}

Thus we can define the good event $\mathcal{G}_{k}$ by:
\begin{align*}
\mathcal{G}_{k} = \bigcap_{l=1}^{k-1} \bigcap_{i=1}^{4} L_{i,l} &\cap \bigcap_{s,a} \left\{ \left| \wb{r}_{l}(s,a) - r(s,a)\right| \leq \wb{\beta}^{r}_{l}(3\delta/(4\pi^{2}SAHl^{3}),s, a)  \right\}&\\
&\cap \left\{ \left|\left| \wb{p}_{k}(\cdot| s,a) - p(\cdot| s,a)\right|\right|_{1} \leq \wb{\beta}^{p}_{k}(3\delta/(4\pi^{2}SAHl^{3}), s, a) \right\}
\end{align*}
Then $\mathbb{P}\left(\bigcap_{k=1}^{+\infty} \mathcal{G}_{k}\right) \geq 1 - \delta/2$
and $\mathcal{G}_{k} \subset \sigma(\mathcal{H}_{k})$ (\ie the history before episode $k$).

\paragraph{Optimism:}
For each episode $k$, the value function $V_{k,1}$ computed by \ldpucbvi is optimistic, that is to say: $V_{k,h}(s) \geq V^{\star}_{h}(s)$ for any $h$ and state $s$. We sum up this with the following lemma:
\todocout{In the paper do we use $V$ as our optimistic value function not $\hat V$?}
\todoeout{Right, modified it}
\begin{lemma}\label{lem:optimistic_value_function}
For any episode $k \in [k]$, the value function $V_{k,1}$ computed by running Alg.~\ref{alg:LDP-UCB-VI} is such that with probability $1-\delta$:
\begin{align}
\forall s\in \mathcal{S},h\in [1, H] \qquad V_{k,h}(s) \geq V^{\star}_{h}(s)
\end{align}
\end{lemma}
\begin{proof}
Fix an episode $k$ then we proceed by backward induction conditioned on the event $\mathcal{G}_{k}$:
\begin{itemize}
\item For $h=H$, we have for any state $s$ and action $a$:
\begin{align}
V_{k,H}(s)\geq Q_{k,H}(s,a) \geq \wt{r}_{k}(s,a) + \beta_{k}^{r}(s,a) \geq r(s,a) \text{ thanks to Prop.~\ref{prop:concentration_ldp}}
\end{align}
\item For $h<H$ when the property is true for $h+1$, we get for any state-action $(s,a)$:
\begin{align}
V_{k,h}(s)\geq Q_{k,h}(s,a) &= \wt{r}_{k}(s,a) + \beta_{k}^{r}(s,a) + \wt{p}_{k}(\cdot| s,a)^{\intercal}V_{k,h+1} + H\beta_{k}^{p}(s,a)\\
&\geq r(s,a) + p(\cdot| s,a)^{\intercal}V_{k,h+1} \geq Q^{\star}_{h}(s,a)
\end{align}
where we used the fact that $\|(\wt{p}_{k}(\cdot| s,a) - p(\cdot|s,a))^{\intercal} V_{k,h+1}\| \leq \|\wh{p}_{k}(\cdot| s,a) - p(\cdot|s,a)\|_1 \|V_{k,h+1}\|_{\infty} \leq H\beta_{k}^{p}(s,a)$ and the inductive hypothesis.
\end{itemize}
\end{proof}

\paragraph{Regret Decomposition:}
We are now ready to analyze the regret of \ldpucbvi. Consider an episode $k$, then, conditioned on $\mathcal{G}_{k}$:
\begin{align*}
V^{\star}_{1}(s_{k,1}) - V^{\pi_{k}}_{1}(s_{k,1}) \leq V_{k,1}(s_{k,1}) - V^{\pi_{k}}_{1}(s_{k,1})
\leq \wt{r}_{k}(s_{k,1},a_{k,1}) + \beta_{k}^{r}(s_{k,1},a_{k,1}) - r(s_{k,1}, a_{k,1})&\\
+ \wt{p}_{k}(\cdot| s,a)^{\intercal}V_{k,2}
    - p(\cdot| s,a)^{\intercal}V^{\pi_{k}}_{2} + H\beta_{k}^{p}(s_{k,1}, a_{k,1})&\\
\end{align*}
where the last inequality follows from recursively applying the same technique.
Then, observe that $(\eta_{k,h})_{k,h}$ is a Martingale Difference Sequence with respect to the history before episode $k$ and thanks to Azuma-Hoeffding inequality we have that with probability at least $1-\delta/2$, $\sum_{k=1}^{K}\sum_{h=1}^{H-1} \eta_{k,h}\leq 2H\sqrt{KH\ln(2/\delta)}$. Therefore, we have with probability at least $1-\delta$:
\begin{align}
    R(\ldpucbvi, K) \leq 2\sum_{k=1}^{K}\sum_{h=1}^{H} \beta_{k}^{r}(s_{k,h}, a_{k,h}) + H\beta_{k}^{p}(s_{k,h},a_{k,h}) + \underbrace{2H\sqrt{T\ln(2/\delta)}}_{\text{MDS error term}}
\end{align}

Let $\nu_{k}(s,a) = \sum_{h=1}^{H} \mathds{1}_{\{s_{k,h}=s, a_{k,h}=a\}}$.
Then summing over the reward bonus and using the fact that $\alpha >1$, we get:
{\small\begin{equation}\label{eq:upper_reward_bonus}
\begin{aligned}
\sum_{k=1}^{K}\sum_{h=1}^{H} \beta_{k}^{r}(s_{k,h}, a_{k,h}) &= \sum_{s,a,k} \frac{\nu_{k}(s,a)L_{k,r}}{\sqrt{\wt{N}_{k}^{r}(s,a) + \alpha c_{k,2}\left(\varepsilon_{0}, \delta_{0}, \frac{3\delta}{2\pi^{2}k^{2}}\right)}}\\
&+ \sum_{s,a,k} \frac{\nu_{k}(s,a)(\alpha + 1)c_{k,2}\left(\varepsilon_{0}, \delta_{0}, \frac{3\delta}{2\pi^{2}k^{2}}\right)}{\alpha c_{k,2}\left(\varepsilon_{0}, \delta_{0}, \frac{3\delta}{2\pi^{2}k^{2}}\right) + \wt{N}_{k}^{r}(s,a)}\\
 &+ \sum_{s,a,k}\frac{\nu_{k}(s,a)c_{k,1}\left(\varepsilon_{0}, \delta_{0}, \frac{3\delta}{2\pi^{2}k^{2}}\right)}{\alpha c_{k,2}\left(\varepsilon_{0}, \delta_{0}, \frac{3\delta}{2\pi^{2}k^{2}}\right) + \wt{N}_{k}^{r}(s,a)}
\end{aligned}
\end{equation}}
where $L_{k,r} = \sqrt{2\ln\left(\frac{4\pi^{2}SAHk^{3}}{3\delta}\right)}$. Then, using that $\wt{N}_{k}^{r}(s,a) + c_{k,2}\left(\varepsilon_{0}, \delta_{0}, \frac{3\delta}{2\pi^{2}k^{2}}\right) \geq N_{k}(s,a)$ on the good event from $\mathcal{G}_{k}$:
{\small\begin{equation}\label{eq:temp}
    \begin{aligned}
 \eqref{eq:upper_reward_bonus} \leq \sum_{s,a,k}  \frac{\nu_{k}(s,a)L_{k,r}}{\sqrt{N_{k}(s,a) + (\alpha - 1)c_{k,2}\left(\varepsilon_{0}, \delta_{0}, \frac{3\delta}{2\pi^{2}k^{2}}\right)}} + \frac{\nu_{k}(s,a)(\alpha + 1)c_{k,2}\left(\varepsilon_{0}, \delta_{0}, \frac{3\delta}{2\pi^{2}k^{2}}\right)}{(\alpha-1) c_{k,2}\left(\varepsilon_{0}, \delta_{0}, \frac{3\delta}{2\pi^{2}k^{2}}\right) + N_{k}(s,a)}& \\
 +\sum_{s,a,k} \frac{\nu_{k}(s,a)c_{k,1}\left(\varepsilon_{0}, \delta_{0}, \frac{3\delta}{2\pi^{2}k^{2}}\right)}{(\alpha-1) c_{k,2}\left(\varepsilon_{0}, \delta_{0}, \frac{3\delta}{2\pi^{2}k^{2}}\right) + N_{k}(s,a)}&
    \end{aligned}
\end{equation}}
But because $c_{k,2}$ is non-decreasing in $k$, we have that,
\begin{equation}\label{eq:temp_2}
    \begin{aligned}
    \eqref{eq:temp} \leq \left((\alpha + 1)c_{K,2}\left(\varepsilon_{0}, \delta_{0}, \frac{3\delta}{2\pi^{2}K^{2}}\right) + c_{K,1}\left(\varepsilon_{0}, \delta_{0}, \frac{3\delta}{2\pi^{2}K^{2}}\right)\right)\sum_{k,s,a} \frac{\nu_{k}(s,a)}{N_{k}(s,a)}\\
                            + \sum_{s,a,k} \frac{\nu_{k}(s,a)L_{K,r}}{\sqrt{N_{k}(s,a)}}\\
    \end{aligned}
\end{equation}
Which can be rewritten as:
\begin{equation}
    \begin{aligned}
        \eqref{eq:temp_2} \leq 2\left((\alpha + 1)c_{K,2}\left(\varepsilon_{0}, \delta_{0}, \frac{3\delta}{2\pi^{2}K^{2}}\right) + c_{K,1}\left(\varepsilon_{0}, \delta_{0}, \frac{3\delta}{2\pi^{2}K^{2}}\right) \right)SA(\ln(2TSA) + H) \\
        + \sqrt{6\ln\left(14SAT/\delta\right)}\left(\sqrt{2SAT} + HSA\right)
    \end{aligned}
\end{equation}
where the last inequality comes from Lem. 19 in \citep{Jaksch10}. 
For the sum of the bonus on the transition dynamics we have that:
{\small
\begin{equation}\label{eq:upper_dynamics_bonus}
\begin{aligned}
\sum_{k=1}^{K}\sum_{h=1}^{H} H\beta_{k}^{p}(s_{k,h}, a_{k,h}) = \sum_{s,a,k} \frac{H\nu_{k}(s,a)L_{k,p}}{\sqrt{\wt{N}_{k}^{p}(s,a) + \alpha c_{k,3}\left(\varepsilon_{0}, \delta_{0}, \frac{3\delta}{2\pi^{2}k^{2}}\right)}} \\
+\sum_{s,a,k}\frac{HS\nu_{k}(s,a)c_{k,4}\left(\varepsilon_{0}, \delta_{0},\frac{3\delta}{2\pi^{2}k^{2}}\right)}{\alpha c_{k,3}\left(\varepsilon_{0}, \delta_{0}, \frac{3\delta}{2\pi^{2}k^{2}}\right) + \wt{N}_{k}^{p}(s,a)} \\
+\sum_{s,a,k}\frac{H\nu_{k}(s,a)(\alpha + 1)c_{k,3}\left(\varepsilon_{0}, \delta_{0}, \frac{3\delta}{2\pi^{2}k^{2}}\right)}{\alpha c_{k,3}\left(\varepsilon_{0}, \delta_{0}, \frac{3\delta}{2\pi^{2}k^{2}}\right) + \wt{N}_{k}^{p}(s,a)}
\end{aligned}
\end{equation}}
where $L_{k,p} = \sqrt{14S\ln\left(\frac{4\pi^{2}SAHk^{3}}{3\delta}\right)}$. Then similarly to the reasonning used to bound Eq.~\eqref{eq:upper_reward_bonus}, we have:
{\small
\begin{equation}
    \begin{aligned}
\eqref{eq:upper_dynamics_bonus} &\leq \sum_{s,a,k}  \frac{H\nu_{k}(s,a)L_{k,p}}{\sqrt{N_{k}(s,a) + (\alpha - 1)c_{k,3}\left(\varepsilon_{0}, \delta_{0}, \frac{3\delta}{2\pi^{2}k^{2}}\right)}}
+ \sum_{s,a,k}\frac{H\nu_{k}(s,a)(\alpha + 1)c_{k,3}\left(\varepsilon_{0}, \delta_{0}, \frac{3\delta}{2\pi^{2}k^{2}}\right)}{(\alpha-1) c_{k,3}\left(\varepsilon_{0}, \delta_{0}, \frac{3\delta}{2\pi^{2}k^{2}}\right) + N_{k}(s,a)}\\
&+\sum_{k,s,a}\frac{HSc_{k,4}\left(\varepsilon_{0}, \delta_{0}, \frac{3\delta}{2\pi^{2}k^{2}}\right)}{(\alpha-1) c_{k,3}\left(\varepsilon_{0}, \delta_{0}, \frac{3\delta}{2\pi^{2}k^{2}}\right) + N_{k}(s,a)}\nonumber\\
&\leq + \left((\alpha + 1)c_{K,3}\left(\varepsilon_{0}, \delta_{0}, \frac{3\delta}{2\pi^{2}K^{2}}\right) +
Sc_{K,4}\left(\varepsilon_{0}, \delta_{0},
\frac{3\delta}{2\pi^{2}K^{2}}\right)\right)
\sum_{k,s,a} \frac{H\nu_{k}(s,a)}{N_{k}(s,a)} \\
&\sum_{s,a,k} \frac{H\nu_{k}(s,a)L_{K,p}}{\sqrt{N_{k}(s,a)}} \\
&\leq 2SAH\left((\alpha + 1)c_{K,3}\left(\varepsilon_{0}, \delta_{0}, \frac{3\delta}{2\pi^{2}K^{2}}\right) + Sc_{K,4}\left(\varepsilon_{0}, \delta_{0}, \frac{3\delta}{2\pi^{2}K^{2}}\right)\right)\left(\ln(2TSA) + H\right) \\
&+ H\sqrt{46S\ln\left(14SAT/\delta\right)}\left(\sqrt{2SAT} + HSA\right)\nonumber
\end{aligned}
\end{equation}}
where the last inequality comes from~\citep[][Lem. 19]{Jaksch10} and~\citep[][Lem. 8]{fruit2020improved}.
Hence putting everything together, we get that with probability $1-\delta$:
{\small
\begin{align*}
R(\ldpucbvi, K) \leq H\sqrt{46S\ln(14SAT/\delta)}(\sqrt{2SAT} + HSA) + \sqrt{6\ln(14SAT/\delta)}(\sqrt{2SAT} + HSA)\\
+ 2SAH\left((\alpha + 1)c_{K,3}\left(\varepsilon_{0}, \delta_{0}, \frac{3\delta}{2\pi^{2}K^{2}}\right) + Sc_{K,4}\left(\varepsilon_{0}, \delta_{0}, \frac{3\delta}{2\pi^{2}K^{2}}\right)\right)\left(\ln(2TSA) + H\right) \\
+ 2\left((\alpha + 1)c_{K,2}\left(\varepsilon_{0}, \delta_{0}, \frac{3\delta}{2\pi^{2}K^{2}}\right) + c_{K,1}\left(\varepsilon_{0}, \delta_{0}, \frac{3\delta}{2\pi^{2}K^{2}}\right)\right)SA(\ln(2TSA) + H) + 2H\sqrt{T\ln(2/\delta)}&
\end{align*}}

In addition, because \algo has only access to the privatized data, that is to say it only uses the output of $\mathcal{M}(\{(s_{k,h},a_{k,h},r_{k,h})_{h\leq H}\})$ for each episode $k$, the LDP constraint is satsified as long as the privacy mechanism $\mathcal{M}$ satisfies Def.~\ref{def:RL-LDP}.

\paragraph{Note:} the proof of this regret upper-bound relies on concentration inequalities more generally used in the average reward regret minimization setting. Stated otherwise, we directly study the error between the estimated model and the true model, \ie $|\wt{r}_{k} - r|$ and  $||\wt{p}_{k}(.\mid s,a) - p(.\mid s,a)||_{1}$ for each $s,a$.
In the non-private setting, it is possible to get a more refined regret using more precise concentration inequalities, mainly Bernstein inequality and other tools introduced in~\citep{azar2017minimax}. However, in the private setting, using such results only leads to a gain in lower order terms and terms independent of $\varepsilon$ while the technical derivations are much more intricate.

\section{The Laplace Mechanism for Local Differential Privacy}\label{app:proof_algo}
\todomp{Check references}
In this appendix, we show how the well-known Laplace mechanism~\citep{dwork2006calibrating} can be used with \algo to ensure LDP and a sublinear regret.

\begin{algorithm}[h]
  \small
  \caption{Laplace mechanism for LDP}
  \label{alg:laplace_mechanism}
  \begin{algorithmic}
    \STATE {\bfseries Input:} Trajectory: $X = \{(s_{h}, a_{h}, r_{h}) \mid h\leq H\}$, Privacy Parameter: $\varepsilon_{0}$
    \STATE Draw $(Y_{i, X}(s,a))_{(s,a)\in \mathcal{S}\times \mathcal{A}, i\leq 2}$ i.i.d $\text{Lap}(1/\varepsilon_{0})$ and $(Z_{X}(s,a,s'))_{(s,a,s')\in \mathcal{S}\times \mathcal{A}\times \mathcal{S}}$ i.i.d $\text{Lap}(1/\varepsilon_{0})$ and independent from $Y_{i,X}$ for $i\in\{1,2\}$\;
    \FOR{$(s,a)\in \mathcal{S}\times \mathcal{A}$}
          \STATE $\wt{R}_{X}(s,a) = \sum_{h=1}^{H} r_{h}\mathds{1}_{\{s_{h}, a_{h}=s,a\}} + Y_{1, X}(s,a)$
          \STATE $\wt{N}_{X}^{r}(s,a) = \sum_{h=1}^{H} \mathds{1}_{\{s_{h}, a_{h}=s,a\}} + Y_{2, X}(s,a)$
          \FOR{$s'\in \mathcal{S}$}
            \STATE $\wt{N}_{X}^{p}(s,a,s') = \sum\limits_{h=1}^{H-1} \mathds{1}_{\{s_{h}, a_{h}, s_{h+1} = s,a,s'\}} + Z_{X}(s,a,s')$
          \ENDFOR
    \ENDFOR
    \STATE{\bfseries Return:} $(\wt{R}_{X}, \wt{N}_{X}^{r}, \wt{N}_{X}^{p})\in \mathbb{R}^{S\times A} \times \mathbb{R}^{S\times A}\times \mathbb{R}^{S\times A\times S}$
  \end{algorithmic}
\end{algorithm}

\subsection{The Laplace mechanism (Alg.~\ref{alg:laplace_mechanism}) satisfies local differential privacy (Asm.~\ref{assumption:concentration_privacy})}\label{app:proof_laplace_def_2}
We first prove Thm.~\ref{thm:ldp_laplace} which states that using Alg.~\ref{alg:laplace_mechanism} with parameter $\varepsilon_{0} = \varepsilon/6H$ guarantees $(\varepsilon,\delta)$-LDP.
\begin{theorem}\label{thm:ldp_laplace}
  For any $\varepsilon>0$, the Laplace mechanism described by Alg.~\ref{alg:laplace_mechanism} with parameter $\varepsilon_{0} = \varepsilon/6H$ is $(\varepsilon,0)$-LDP (and thus $(\varepsilon, \delta_{0})$-LDP for every $\delta_{0}\geq 0$).
\end{theorem}
Formally, we need to show that, for any two trajectories $X$ and $X'$ and tuple $(r,n,n')$, the following inequality holds
\begin{equation}
  \mathbb{P}\Big( \mathcal{M}(X) = (r, n, n') \Big) \leq e^{\varepsilon} \mathbb{P}\Big( \mathcal{M}(X') = (r, n, n') \Big) + \delta
\end{equation}
where $r$, $n$, $n'$ are vectors of dimension $SA$, $SA$ and $S^2A$, respectively. See the LDP definition in Def.~\ref{def:RL-LDP}.

\begin{proof}[Proof of Thm.~\ref{thm:ldp_laplace}]
Let's consider two trajectories $X = \{(s_{h}, a_{h}, r_{h}) \mid h\leq H\}$ and $X' = \{ (s_{h}', a_{h}', r_{h}') \mid h\leq H\}$. We denote the output of the private randomizer $\mathcal{M}$ by $\mathcal{M}(X) = (\wt{R}_{X}, \wt{N}_{X}^{r}, \wt{N}_{X}^{p})$ and $\mathcal{M}(X') = (\wt{R}_{X'}, \wt{N}_{X'}^{r}, \wt{N}_{X'}^{p})$.
Recall that $\wt{R}_{X}(s,a) := \sum_{h=1}^{H}r_{h}\mathds{1}_{\{s_{h} = s, a_{h} = a\}} + Y_{1,X}(s,a)$ where $(Y_{1, X}(s,a))_{(s,a)\in\mathcal{S}\times\mathcal{A}}$ are independent Laplace variables with parameter $\varepsilon/(6H)$.
Consider a vector $r\in \mathbb{R}^{S\times A}$, then:
{\small\begin{align}
\frac{ \mathbb{P}\left(\forall (s,a),\wt{R}_{X}(s,a) = r_{s,a} \mid X\right)}{\mathbb{P}\left(\forall (s,a), \wt{R}_{X'}(s,a) = r_{s,a} \mid X'\right)} &=\prod_{s,a} \frac{\mathbb{P}\left(Y_{1,X}(s,a) = \sum_{h=1}^{H}r_{h}\mathds{1}_{\{s_{h} = s, a_{h} = a\}} - r_{s,a} \mid X\right)}{\mathbb{P}\left(Y_{1,X'}(s,a) = \sum_{h=1}^{H}r_{h}'\mathds{1}_{\{s_{h}' = s, a_{h}' = a\}} - r_{s,a}\mid X'\right)}
\end{align}}
since the Laplace distribution is symmetric.
But $Y_{1,X}(s,a)$ and $Y_{1,X'}(s,a)$ are independent 
random variables for any state-action pair. Thus:
{\small\begin{equation}\label{eq:dp_reward_laplace}
\begin{aligned}
\prod_{s,a} &\frac{\mathbb{P}\left(Y_{1,X}(s,a) = \sum_{h=1}^{H}r_{h}\mathds{1}_{\Big\{\begin{subarray}{l}s_{h} = s,\\a_{h} = a\end{subarray}\Big\}} - r_{s,a}\mid X\right)}{\mathbb{P}\left(Y_{1,X'}(s,a) = \sum_{h=1}^{H}r_{h}'\mathds{1}_{\Big\{\begin{subarray}{l}s_{h}' = s,\\ a_{h}' = a\end{subarray}\Big\}} - r_{s,a} \mid X'\right)}
= \prod_{s,a} \frac{e^{\left(\varepsilon_{0}\left|\sum_{h=1}^{H} (r_{h}\mathds{1}_{\Big\{\begin{subarray}{l}s_{h} = s,\\a_{h} = a\end{subarray}\Big\}} - r_{s,a} \right| \right)}}{e^{\left(\varepsilon_{0}\left|\sum_{h=1}^{H} (r_{h}'\mathds{1}_{\Big\{\begin{subarray}{l}s_{h}' = s,\\a_{h}' = a\end{subarray}\Big\}} - r_{s,a} \right| \right)}}\\
&\leq \exp\Bigg(\varepsilon_{0}\sum_{s,a} \Bigg|\sum_{h=1}^{H} (r_{h}\mathds{1}_{\{s_{h} = s,a_{h} = a\}} - r_{h}'\mathds{1}_{\{s_{h}' = s, a_{h}' = a\}}) \Bigg| \Bigg) \\
&\leq \exp\left(\varepsilon_{0}\sum_{s,a,h} (|r_{h}| \mathds{1}_{\{s_{h} = s, a_{h} = a\}} + |r_{h}'|\mathds{1}_{\{s_{h}' = s, a_{h}' = a\}})\right)\\
& = \exp\left(\varepsilon_{0}\sum_{h} (|r_{h}| + |r_{h}'|)\right) \leq \exp\left(2H\varepsilon_{0}\right) = \exp\left(\frac{\varepsilon}{3}\right)
\end{aligned}
\end{equation}}
where we used the definition of the Laplace distribution, $x\mapsto \frac{1}{2b}\exp(|x|/b)$. Let $n\in\mathbb{R}^{S\times A}$ and $n'\in \mathbb{R}^{S\times A\times S}$.
Similarly, since $\wt{N}_{X}^{r}(s,a) = \sum_{h=1}^{H} \mathds{1}_{\{s_{h}=s, a_{h}=a\}} + Y_{2, X}(s,a)$ and $\wt{N}_{X}^{p}(s,a,s') = \sum_{h=1}^{H-1} \mathds{1}_{\{s_{h}=s, a_{h}=a, s_{h+1} = s'\}} + Z_{X}(s,a,s')$, we have:
\begin{align}\label{eq:dp_state_action_laplace}
&\frac{\mathbb{P}\left(\forall (s,a),\wt{N}^{r}_{X}(s,a) = n_{s,a} \mid X\right)}{\mathbb{P}\left(\forall (s,a), \wt{N}^{r}_{X'}(s,a) = n_{s,a} \mid X'\right)} \leq \exp\left(\frac{\varepsilon}{3}\right)
\end{align}
and:
\begin{align}\label{eq:dp_state_action_state_laplace}
\frac{\mathbb{P}\left(\forall (s,a,s'),\wt{N}^{p}_{X}(s,a,s') = {\color{red}n'_{s,a,s'}} \mid X\right)}{\mathbb{P}\left(\forall (s,a,s'), \wt{N}^{p}_{X'}(s,a,s') = {\color{red}n'_{s,a,s'}} \mid X'\right)} \leq \exp\left(\frac{\varepsilon}{3}\right)
\end{align}
Then because $(Y_{i,X}(s,a))_{i\leq 2, (s,a)\in \mathcal{S}\times \mathcal{A}}$, $(Z_{X}(s,a,s'))_{(s,a,s')\in \mathcal{S}\times \mathcal{A}\times \mathcal{S}}$ are independent it holds that:
$$\mathbb{P}\left( \wt{R}_{X} = r, \wt{N}_{X}^{r} = n, \wt{N}_{X}^{p} = n'\mid X\right) = \mathbb{P}\left( \wt{R}_{X} = r\mid X\right)\mathbb{P}\left( \wt{N}_{X}^{r} = n\mid X\right)\mathbb{P}\left(\wt{N}_{X}^{p} = n'\mid X\right)$$
Thus for any $(r,n,n')\in \mathbb{R}^{S\times A}\times \mathbb{R}^{S\times A} \times \mathbb{R}^{S\times A\times S}$ and any two trajectories $X$ and $X'$:
\begin{align*}
\mathbb{P}\Big(\mathcal{M}(X) = (r,n,n')\mid X \Big) &= \mathbb{P}\left(\wt{R}_{X} = r, \wt{N}_{X}^{r} = n, \wt{N}^{p}_{X} = n'\mid X \right) \\
&= \mathbb{P}\left(\wt{R}_{X} = r\mid X\right) \mathbb{P}\left(\wt{N}_{X}^{r} = n\mid X\right) \mathbb{P}\left(\wt{N}^{p}_{X} = n'\mid X \right)
\label{eq:laplace_prob_pm_decomp}
\end{align*}
where we use the convention that $\wt R_X=r$ implies that $\wt R_X(s,a)=r_{x,a}$, and similarly for $\wt N_X^r=n, \wt N_X^p=n'$.
Therefore using inequalities \eqref{eq:dp_reward_laplace}, \eqref{eq:dp_state_action_laplace} and \eqref{eq:dp_state_action_state_laplace} in~\eqref{eq:laplace_prob_pm_decomp}, we have:
\begin{equation}
\begin{aligned}
\mathbb{P}\Big(\mathcal{M}(X) = (r,n,n')\mid X \Big) &=
\mathbb{P}\left(\wt{R}_{X} = r\mid X\right) \mathbb{P}\left(\wt{N}_{X}^{r} = n\mid X\right) \mathbb{P}\left(\wt{N}^{p}_{X} = n'\mid X \right)  \\
&\leq \exp(\varepsilon)\mathbb{P}\left(\wt{R}_{X'} = r\mid X'\right)
\mathbb{P}\left(\wt{N}_{X'}^{r} = n\mid X'\right)\mathbb{P}\left(\wt{N}^{p}_{X'} = n'\mid X' \right) \nonumber\\
&= \exp(\varepsilon)\mathbb{P}\left(\wt{R}_{X'} = r,\wt{N}_{X'}^{r} = n,  \wt{N}^{p}_{X'} = n'\mid X' \right) \nonumber\\
&= \exp(\varepsilon)\mathbb{P}\left(\mathcal{M}(X') = (r,n,n') \mid X' \right) \nonumber\\
\end{aligned}
\end{equation}
This concludes the proof.
\end{proof}

Now that we shown the Laplace mechanism ensures LDP with the reight parameter, let's show that the latter satisfies Asm.~\ref{assumption:concentration_privacy} by showing the following proposition:
\begin{proposition}\label{prop:laplace_assumption1}
  For any $\varepsilon>0$, the Laplace mechnism, Alg.~\ref{alg:laplace_mechanism}, with parameter $\varepsilon_{0} = \varepsilon/(6H)$ satisfies Def.~\ref{assumption:concentration_privacy} for any $\delta>0$ and $k\in \mathbb{N}$ with $c_{k,1}(\varepsilon, \delta) = c_{k,2}(\varepsilon, \delta)$, $c_{k,3}(\varepsilon, \delta) = \sqrt{S}c_{k,4}(\varepsilon, \delta)$ and:
\begin{align*}
&c_{k,1}(\varepsilon, \delta)  = \max\left\{\sqrt{k}, \ln\left(\frac{6SA}
{\delta}\right)\right\}\frac{\sqrt{8\ln\left(\frac{6SA}{\delta}\right)}}{\varepsilon/6H},\\
&c_{k,3}(\varepsilon,\delta) = \max\left\{\sqrt{kS}, \ln\left(\frac{6S^{2}A}{\delta}\right)\right\} \frac{\sqrt{8\ln\left(\frac{6S^{2}A}{\delta}\right)}}{\varepsilon/6H}
\end{align*}
\end{proposition}

Before proving Prop.~\ref{prop:laplace_assumption1}\todompout{ref?} we state the following concentration inequality for the sum of Laplace variables.
\begin{proposition}{\citep[Cor. 12.3]{dwork2014algorithmic}}\label{prop:conc.sum.laplace}
  Let $Y_{1}, \dots, Y_{k}$ be independent Lap($b$) random variables with $b>0$ and $\delta\in (0,1)$ then for any $\nu > b\max\left\{ \sqrt{k}, \sqrt{\ln(2/\delta)}\right\}$,
  \begin{align*}
  \mathbb{P}\left( \left|\sum_{l=1}^{k} Y_{l}\right| > \nu\sqrt{8\ln(2/\delta)}\right) \leq \delta
  \end{align*}
\end{proposition}

We can now prove Prop.~\ref{prop:laplace_assumption1} that shows that Alg.~\ref{alg:laplace_mechanism} satisfies Def.~\ref{assumption:concentration_privacy}.
\begin{proof}[Proof of Prop.~\ref{prop:laplace_assumption1}]
Let $X_{1}, \ldots, X_{k-1}$ be the $k-1$ trajectories generated before episode $k \geq 1$.
Consider the private statistic $\wt{R}_k(s,a)$ generated by the private randomizer before episode $k$.
Then for any state-action pair $(s,a)\in \mathcal{S}\times\mathcal{A}$:
\begin{align*}
\left|\wt{R}_{k}(s,a) - R_{k}(s,a)\right|
&= \Bigg|\sum_{l<k} (\wt{R}_{X_l}(s,a) - R_{X_l}(s,a))\Bigg|\\
&= \Bigg| \sum_{l < k} \left(  Y_{1, X_l}(s,a) + \sum_{h=1}^{H} r_{h}\mathds{1}_{\Big\{\begin{subarray}{l} s_{l,h}=s, \\ a_{l,h}=a\end{subarray}\Big\}} \right)- \sum_{l < k} \sum_{h=1}^{H} r_{h}\mathds{1}_{\Big\{\begin{subarray}{l}s_{l,h}=s, \\ a_{l,h}=a\end{subarray}\Big\}}\Bigg|\\
&= \Bigg| \sum_{l=1}^{k-1} Y_{1,X_{l}}(s,a) \Bigg|
\end{align*}
which is the sum of independent Laplace variables.
Let $\delta >0$. By Prop.~\ref{prop:conc.sum.laplace} we have that with probability at least $1- \delta/(3SA)$
\begin{equation}
  \left| \sum_{l=1}^{k-1} Y_{1,X_{l}}(s,a) \right| \leq \frac{1}{\varepsilon_0} \max \left\{ \sqrt{k-1}, \ln \left(\frac{6 SA}{\delta}\right)\right\} \sqrt{8\ln \left(\frac{6 SA}{\delta}\right)}
\end{equation}
The same property holds for $\wt{N}_{k}^{r}$ and $\wt{N}_{k}^{p}$ and we again apply Prop.~\ref{prop:conc.sum.laplace}.
Properties in Def.~\ref{assumption:concentration_privacy} follow from union bounds.
\end{proof}

\section{Other Privacy Preserving Mechanisms}\label{app:other_mechanisms}

We have shown in App.~\ref{app:proof_laplace_def_2} that the Laplace mechanism, Alg.~\ref{alg:laplace_mechanism}\todompout{Check algo, it has been removed from main paper}, satisfies Def.~\ref{assumption:concentration_privacy}. However it is not the only mechanism to do so. In this appendix we present the Gaussian, Randomized Response and bounded noise mechanisms and show that these also satisfy Def.~\ref{assumption:concentration_privacy}.

\subsection{Gaussian Mechanism:}\label{app:gaussian_mechanism}

The Gaussian mechanism is a fundamental mechanism in the differential privacy literature~\citep[see \eg][]{dwork2014algorithmic}.
However, contrary to the Laplace mechanism the Gaussian mechanism can only guarantees $(\varepsilon, \delta)$-LDP for $\delta>0$.
The mechanism is based on the same idea as the Laplace mechanism, that is to say it adds Gaussian noise to the result of a given computation on the input data. This noise is centered and the standard deviation $\sigma(\varepsilon,\delta)$ is $\frac{cH}{\epsilon_0}$.

\begin{algorithm}[tb]
  \caption{Gaussian mechanism for LDP}
  \label{alg:gaussian_mechanism}
\begin{algorithmic}
  \STATE {\bfseries Input:} Trajectory: $X = \{(s_{h}, a_{h}, r_{h}) \mid h\leq H\}$, Privacy Parameter: $\varepsilon_{0}, c$
  \STATE Draw $(Y_{i, X}(s,a))_{(s,a)\in \mathcal{S}\times \mathcal{A}, i\leq 2}$ i.i.d $\mathcal{N}\left(0,\sigma^{2}\right)$ and $(Z_{X}(s,a,s'))_{(s,a,s')\in \mathcal{S}\times \mathcal{A}\times \mathcal{S}}$ i.i.d $\mathcal{N}\left(0,\sigma^{2}\right)$ and independent from $Y_{i,X}$ for $i\in\{1,2\}$ with $\sigma = cH/\varepsilon_{0}$
  \FOR{$(s,a)\in \mathcal{S}\times \mathcal{A}$}
          \STATE $\wt{R}_{X}(s,a) = \sum_{h=1}^{H} r_{h}\mathds{1}_{\{s_{h}=s, a_{h}=a\}} + Y_{1, X}(s,a)$
          \STATE $\wt{N}_{X}^{r}(s,a) = \sum_{h=1}^{H} \mathds{1}_{\{s_{h}=s, a_{h}=a\}} + Y_{2, X}(s,a)$
          \FOR{$s'\in \mathcal{S}$}
            \STATE $\wt{N}_{X}^{p}(s,a,s') = \sum_{h=1}^{H-1} \mathds{1}_{\{s_{h}=s, a_{h}=a, s_{h+1} = s'\}} + Z_{X}(s,a,s')$
          \ENDFOR
  \ENDFOR
  \STATE {\bfseries Return:} $(\wt{R}_{X}, \wt{N}_{X}^{r}, \wt{N}_{X}^{p})\in \mathbb{R}^{S\times A} \times \mathbb{R}^{S\times A}\times \mathbb{R}^{S\times A\times S}$
\end{algorithmic}
\end{algorithm}

In the following, we show that the Gaussian mechanism almost satisfies Def.~\ref{assumption:concentration_privacy}.
The Gaussian mechanism can not guarantee $(\varepsilon_{0}, 0)$-LDP for any $\varepsilon_{0}>0$, however we show that it satisfies the other necessary conditions, including $(\varepsilon_{0}, \delta)$-LDP for any $\delta>0$.
First, we show that the mechanism guarantees Local Differential Privacy for high enough noise.
\begin{proposition}\label{prop:gaussian_ldp}
For any $1\geq \varepsilon_{0}>0$ and $\delta_{0}>0$ and parameter $c> 4\ln\left(\frac{24}{\delta_{0}}\right)$, the Gaussian mechanism, Alg.~\ref{alg:gaussian_mechanism}, is $(\varepsilon_{0}, \delta_{0})$-LDP.
\end{proposition}
\begin{proof}[Proof of Prop.~\ref{prop:gaussian_ldp}:]
The proof is based on the proof presented in \citep{dwork2014algorithmic}. Similarly to the proof of Prop.~\ref{prop:laplace_assumption1} let's consider two trajectories $X = \{(s_{h}, a_{h}, r_{h}) \mid h\leq H\}$ and $X' = \{ (s_{h}', a_{h}', r_{h}') \mid h\leq H\}$ and also denote the output of the private randomizer $\mathcal{M}$ by $\mathcal{M}(X) = (\wt{R}_{X}, \wt{N}_{X}^{r}, \wt{N}_{X}^{p})$ and $\mathcal{M}(X') = (\wt{R}_{X'}, \wt{N}_{X'}^{r}, \wt{N}_{X'}^{p})$.

For a given vector $r\in \mathbb{R}^{S\times A}$,
{\small\begin{align}\label{eq:ratio_gaussian}
\frac{ \mathbb{P}\left(\forall (s,a),\wt{R}_{X}(s,a) = r_{s,a} \mid X\right)}{\mathbb{P}\left(\forall (s,a), \wt{R}_{X'}(s,a) = r_{s,a} \mid X'\right)} &=\prod_{s,a} \frac{\mathbb{P}\left(Y_{1,X}(s,a) = \sum_{h=1}^{H}r_{h}\mathds{1}_{\{s_{h} = s, a_{h} = a\}} - r_{s,a} \mid X\right)}{\mathbb{P}\left(Y_{1,X'}(s,a) = \sum_{h=1}^{H}r_{h}'\mathds{1}_{\{s_{h}' = s, a_{h}' = a\}} - r_{s,a}\mid X'\right)}
\end{align}}
since the Gaussian distribution is symmetric. Then,
\begin{equation}\label{eq:ldp_gaussian}
\begin{aligned}
\prod_{s,a} &\frac{\mathbb{P}\left(Y_{1,X}(s,a) = \sum_{h=1}^{H}r_{h}\mathds{1}_{\{s_{h} = s, a_{h} = a\}} - r_{s,a}\mid X\right)}{\mathbb{P}\left(Y_{1,X'}(s,a) = \sum_{h=1}^{H}r_{h}'\mathds{1}_{\{s_{h}' = s, a_{h}' = a\}} - r_{s,a} \mid X'\right)} \\
&= \prod_{s,a} \exp\left(\frac{\left(\sum_{h=1}^{H}r_{h}\mathds{1}_{\{s_{h} = s, a_{h} = a\}} - r_{s,a}\right)^{2} - \left(\sum_{h=1}^{H}r_{h}'\mathds{1}_{\{s_{h}' = s, a_{h}' = a\}} - r_{s,a}\right)^{2}}{2\sigma^{2}} \right)
\end{aligned}
\end{equation}
But, considering the squared term, we get
{\small\begin{align*}
\left(\sum_{h=1}^{H}r_{h}\mathds{1}_{\Big\{\begin{subarray}{l}s_{h} = s, \\a_{h} = a\end{subarray}\Big\}} - r_{s,a}\right)^{2}  &= \left(\sum_{h=1}^{H}r_{h}\mathds{1}_{\Big\{\begin{subarray}{l}s_{h} = s, \\a_{h} = a\end{subarray}\Big\}} - \sum_{h=1}^{H}r_{h}'\mathds{1}_{\Big\{\begin{subarray}{l}s_{h}' = s, \\a_{h}' = a\end{subarray}\Big\}} + \sum_{h=1}^{H}r_{h}'\mathds{1}_{\Big\{\begin{subarray}{l}s_{h}' = s, \\a_{h}' = a\end{subarray}\Big\}} - r_{s,a}\right)^{2} \\
&= \left(\sum_{h=1}^{H}r_{h}\mathds{1}_{\Big\{\begin{subarray}{l}s_{h} = s, \\a_{h} = a\end{subarray}\Big\}} - \sum_{h=1}^{H}r_{h}'\mathds{1}_{\Big\{\begin{subarray}{l}s_{h}' = s, \\a_{h}' = a\end{subarray}\Big\}}\right)^{2} + \left( \sum_{h=1}^{H}r_{h}'\mathds{1}_{\Big\{\begin{subarray}{l}s_{h}' = s, \\a_{h}' = a\end{subarray}\Big\}} - r_{s,a}\right)^{2} \\
&+ 2\left(\sum_{h=1}^{H}r_{h}\mathds{1}_{\Big\{\begin{subarray}{l}s_{h} = s, \\a_{h} = a\end{subarray}\Big\}} - \sum_{h=1}^{H}r_{h}'\mathds{1}_{\Big\{\begin{subarray}{l}s_{h}' = s, \\a_{h}' = a\end{subarray}\Big\}}\right)\left(\sum_{h=1}^{H}r_{h}'\mathds{1}_{\Big\{\begin{subarray}{l}s_{h}' = s, \\a_{h}' = a\end{subarray}\Big\}} - r_{s,a}\right)
\end{align*}}
Hence we get that
\begin{equation}
\begin{aligned}
\eqref{eq:ldp_gaussian}&= \prod_{s,a} \exp\Bigg(\frac{1}{2\sigma^{2}}\Bigg(\Bigg(\sum_{h=1}^{H}r_{h}\mathds{1}_{\Big\{\begin{subarray}{l}s_{h} = s, \\a_{h} = a\end{subarray}\Big\}} - \sum_{h=1}^{H}r_{h}'\mathds{1}_{\Big\{\begin{subarray}{l}s_{h}' = s, \\a_{h}' = a\end{subarray}\Big\}}\Bigg)^{2} \\
&\hspace{2cm}- 2\Bigg(\sum_{h=1}^{H} r_{h}\mathds{1}_{\Big\{\begin{subarray}{l}s_{h} = s, \\a_{h} = a\end{subarray}\Big\}} - r_{h}'\mathds{1}_{\Big\{\begin{subarray}{l}s_{h}' = s, \\a_{h}' = a\end{subarray}\Big\}}\Bigg)\Bigg(\sum_{h=1}^{H}r_{h}'\mathds{1}_{\Big\{\begin{subarray}{l}s_{h}' = s, \\a_{h}' = a\end{subarray}\Big\}} - r_{s,a}\Bigg) \Bigg)\Bigg). \\
\end{aligned}
\end{equation}
But, $\sum_{s,a}\Big(\sum_{h=1}^{H}r_{h}\mathds{1}_{\{s_{h} = s, a_{h} = a\}} - \sum_{h=1}^{H}r_{h}'\mathds{1}_{\{s_{h}' = s, a_{h}' = a\}}\Big)^{2} \leq 2H^{2}$ because for each step $h$, $r_{h}\in [0,1]$. By the same reasonning, we have $\sum_{s,a}\left|\Big(\sum_{h=1}^{H} r_{h}\mathds{1}_{\{s_{h} = s, a_{h} = a\}} - r_{h}'\mathds{1}_{\{s_{h}' = s, a_{h}' = a\}}\Big)\sum_{h=1}^{H}r_{h}'\mathds{1}_{\{s_{h}' = s, a_{h}' = a\}}\right| \leq H^{2}$. Therefore, we have:
\begin{equation}\label{eq:ldp_gaussian_2}
\begin{aligned}
\eqref{eq:ldp_gaussian} &\leq \exp\Bigg(\frac{1}{2\sigma^{2}}\Bigg(2\sum_{s,a}\Bigg(\sum_{h=1}^{H} r_{h}\mathds{1}_{\{s_{h} = s, a_{h} = a\}} - r_{h}'\mathds{1}_{\{s_{h}' = s, a_{h}' = a\}}\Bigg)r_{s,a} + 3H^{2} \Bigg)\Bigg)\\
&\leq \exp\Bigg(\frac{1}{2\sigma^{2}}\Bigg(2\sqrt{2}H\sqrt{\sum_{s,a} r_{s,a}^{2}} + 3H^{2} \Bigg)\Bigg)
\end{aligned}
\end{equation}
where the last inequality follows from Cauchy-Schwartz.
Note that if $||r||_{2} \leq \frac{\sigma^{2}\varepsilon_{0}}{3\sqrt{2}H} - \frac{3H}{2\sqrt{2}}$, Eq. \eqref{eq:ldp_gaussian_2} is bounded by $\exp(\varepsilon_{0}/3)$.
Therefore, to finish, we partition $\mathbb{R}^{S\times A}$
in two subspaces $R_{1} = \left\{x \in \mathbb{R}^{S\times A} \mid ||x||_{2} \leq \frac{c^{2}H}{3\sqrt{2}\varepsilon_0} - \frac{3H}{2\sqrt{2}}\right\}$ and $R_{2} = \left\{x \in \mathbb{R}^{S\times A} \mid ||x||_{2} > \frac{c^{2}H}{3\sqrt{2}\varepsilon_0} - \frac{3H}{2\sqrt{2}}\right\}$ where we used the fact that $\sigma = cH/\varepsilon_{0}$ with $c$ a constant to be chosen later.
Then for $c^{2}\geq 4\ln\left(\frac{3}{\delta_{1}}\right)$, for $\delta_{1}$ to be chosen later, $\mathbb{P}\left(Y_{1,X} \in R_{2}\right) \leq \delta_{1}$ and $\mathbb{P}\left(Y_{1,X'} \in R_{2}\right) \leq \delta_{1}$. Thus for Eq.~\eqref{eq:ratio_gaussian}:
{\small\begin{align}
\mathbb{P}&\left(\forall (s,a),\wt{R}_{X}(s,a) = r_{s,a} \mid X\right) = \mathbb{P}\left(\forall (s,a),\wt{R}_{X}(s,a) = r_{s,a} \mid X\right)\mathds{1}_{\{r - (\sum_{h=1}^{H} r_{h}\mathds{1}_{\big\{\begin{subarray}{l}s_{h} = s, \\a_{h} = a\end{subarray}\big\}})_{s,a} \in R_{1}\}} \\
&+ \mathbb{P}\left(\forall (s,a),\wt{R}_{X}(s,a) = r_{s,a} \mid X\right)\mathds{1}_{\{r - (\sum_{h=1}^{H} r_{h}\mathds{1}_{\big\{\begin{subarray}{l}s_{h} = s, \\a_{h} = a\end{subarray}\big\}})_{s,a} \in R_{2}\}} \nonumber \\
&\leq e^{\frac{\varepsilon_{0}}{3}}\mathbb{P}\left(\forall (s,a),\wt{R}_{X'}(s,a) = r_{s,a} \mid X'\right)\mathds{1}_{\{r - (\sum_{h=1}^{H} r_{h}\mathds{1}_{\{\begin{subarray}{l}s_{h} = s, \\a_{h} = a\end{subarray}\}})_{s,a} \in R_{1}\}} \\
&+ \mathbb{P}\left(Y_{1,X} \in R_{2}\right) \nonumber \\
&\leq \exp(\varepsilon_{0}/3)\mathbb{P}\left(\forall (s,a),\wt{R}_{X'}(s,a) = r_{s,a} \mid X'\right) + \delta_{1}
\end{align}}

We get the same results for $\wt{N}^{r}$ and $\wt{N}^{p}$. Then, because $(Y_{i,X}(s,a))_{i\leq 2, (s,a)\in \mathcal{S}\times \mathcal{A}}$, $(Z_{X}(s,a,s'))_{(s,a,s')\in \mathcal{S}\times \mathcal{A}\times \mathcal{S}}$ are independent, see Alg.~\ref{alg:gaussian_mechanism}
it holds that:
$$\mathbb{P}\left( \wt{R}_{X} = r, \wt{N}_{X}^{r} = n, \wt{N}_{X}^{p} = n'\mid X\right) = \mathbb{P}\left( \wt{R}_{X} = r\mid X\right)\mathbb{P}\left( \wt{N}_{X}^{r} = n\mid X\right)\mathbb{P}\left(\wt{N}_{X}^{p} = n'\mid X\right)$$
and so,
\begin{align*}
\mathbb{P}\Big(\mathcal{M}(X) = (r,n,n')\mid X \Big) &= \mathbb{P}\left(\wt{R}_{X} = r, \wt{N}_{X}^{r} = n, \wt{N}^{p}_{X} = n'\mid X \right) \\
&= \mathbb{P}\left(\wt{R}_{X} = r\mid X\right) \mathbb{P}\left(\wt{N}_{X}^{r} = n\mid X\right) \mathbb{P}\left(\wt{N}^{p}_{X} = n'\mid X \right)
\label{eq:Randomized Response_prob_pm_decomp}
\end{align*}
 Then for any two trajectories $X$ and $X'$, we have:
\begin{align*}
\mathbb{P}\left(\wt{R}_{X} = r\mid X\right) \mathbb{P}\left(\wt{N}_{X}^{r} = n\mid X\right) \mathbb{P}\left(\wt{N}^{p}_{X} = n'\mid X \right) \leq \left(e^{\frac{\varepsilon_{0}}{3}}\mathbb{P}\left(\wt{R}_{X'} = r\mid X'\right) + \delta_{1}\right)\\
\times\left(e^{\frac{\varepsilon_{0}}{3}}\mathbb{P}\left(\wt{N}^{r}_{X'} = n \mid X'\right) + \delta_{1}\right)&\\
\times\left(e^{\frac{\varepsilon_{0}}{3}}\mathbb{P}\left(\wt{N}^{p}_{X'} = n' \mid X'\right) + \delta_{1}\right)&\\
\leq e^{\varepsilon_{0}}\mathbb{P}\left(\wt{R}_{X'} = r\mid X'\right) \mathbb{P}\left(\wt{N}_{X'}^{r} = n\mid X'\right)\mathbb{P}\left(\wt{N}^{p}_{X'} = n'\mid X' \right)+ 2\delta_{1}\exp\left(2\varepsilon_{0}/3\right)&\\
+ 2\delta_{1}^{2}\exp\left(\varepsilon_{0}/3\right) + \delta_{1}^{3}&
\end{align*}
Thus by choosing $\delta_{1} = \delta_{0}/8$, it holds that $2\delta_{1}\exp\left(2\varepsilon_{0}/3\right) + 2\delta_{1}^{2}\exp\left(\varepsilon_{0}/3\right) + \delta_{1}^{3} \leq \delta_{0}$ for $\varepsilon_{0} \leq 1$, and so we can conclude that the Gaussian mechanism is $(\varepsilon_0,\delta_0)$-LDP.
\end{proof}

In addition, the precision of the Gaussian mechanism is of the same order as the Laplace mechanism, that is to say:
\begin{proposition}\label{prop:utility_gaussian_mechanism}
The Gaussian mechanism, Alg.~\ref{alg:gaussian_mechanism}, with parameter $\varepsilon_{0}>0$ and $c^{2}\geq 4\ln\left(\frac{24}{\delta_{0}}\right)$ for any $\delta_{0}>0$ satisfies Def.~\ref{assumption:concentration_privacy} for any $\delta>0$ and $k\in \mathbb{N}^{\star}$ with:
\begin{align*}
&c_{k,1}(\varepsilon_0, \delta_{0}, \delta)  = c_{k,2}(\varepsilon_0, \delta_{0}, \delta)  = c_{k,4}(\varepsilon_0,\delta_{0},\delta) = \max\left\{\frac{cH}{\varepsilon_{0}}\sqrt{(k-1)\ln\left(\frac{6SA}{\delta}\right)}, 1\right\}\\
&c_{k,3}(\varepsilon_0,\delta_{0},\delta) = \max\left\{\frac{cH}{\varepsilon_{0}}\sqrt{(k-1)S\ln\left(\frac{6SA}{\delta}\right)}, 1\right\}
\end{align*}
\end{proposition}

This result shows that using the Gaussian mechanism rather than the Laplace mechanism would not lead to improved regret rate as the utilities $c_{k,1}, c_{k,2}, c_{k,3}, c_{k,4}$ have the same depency of $S,A,H, \varepsilon_{0}$ and $k$ . 
Moreover, the Gaussian mechanism only guarantees LDP for $\delta>0$ whereas using the Laplace mechanism ensures that we can guarantee LDP for $\delta=0$ as well.

\begin{proof}[Proof of Prop.~\ref{prop:utility_gaussian_mechanism}:]
Following the same steps as in the proof of Prop~\ref{prop:laplace_assumption1}, we have that at the beginning of episode $k$ with probability at least $1 - \frac{\delta}{3SA}$:
\begin{align}
\left|\wt{R}_{k}(s,a) - R_{k}(s,a)\right|
&= \left|\sum_{l<k} (\wt{R}_{X_l}(s,a) - R_{X_l}(s,a))\right|\\
&= \Bigg| \sum_{l < k} \Bigg(  Y_{1, X_l}(s,a) + \sum_{h=1}^{H} r_{h}\mathds{1}_{\Big\{\begin{subarray}{l}s_{l,h}=s,\\ a_{l,h}=a\end{subarray}\Big\}}\Bigg) - \sum_{l < k} \sum_{h=1}^{H} r_{h}\mathds{1}_{\Big\{\begin{subarray}{l}s_{l,h}=s,\\ a_{l,h}=a\end{subarray}\Big\}}\Bigg|\\
&= \left| \sum_{l=1}^{k-1} Y_{1,X_{l}}(s,a) \right| \leq \sigma\sqrt{2(k-1)\ln\left(\frac{6SA}{\delta}\right)}
\end{align}
for $\sigma = cH/\varepsilon_{0}$ thanks to Chernoff bounds. The same result follows for $\wt{N}^{r}$ and $\wt{N}^{p}$. Therefore, the Gaussian mechanism satisfies Def.~\ref{assumption:concentration_privacy} with $c_{k,1}(\varepsilon_{0}, \delta_{0}, \delta) = c_{k,2}(\varepsilon_{0}, \delta_{0}, \delta) = c_{k,4}(\varepsilon_{0}, \delta_{0}, \delta)$ with:
\begin{align}
c_{k,1}(\varepsilon_{0}, \delta_{0}, \delta) = \max\left\{\frac{cH}{\varepsilon_{0}}\sqrt{(k-1)\ln\left(\frac{6SA}{\delta}\right)}, 1\right\}
\end{align}
with $c> 0$ and:
\begin{align}
c_{k,3}(\varepsilon_{0}, \delta_{0}, \delta) = \max\left\{\frac{cH}{\varepsilon_{0}}\sqrt{(k-1)S\ln\left(\frac{6SA}{\delta}\right)}, 1\right\}
\end{align}
where $c_{k,3}(\varepsilon_{0}, \delta_{0}, \delta)$ is defined such that $\left|  \sum_{s'} N_{k}^{p}(s,a,s') - \sum_{s‘}\wt{N}_{k}^{p}(s,a,s')\right| \leq c_{k,3}(\varepsilon_{0}, \delta_{0},\delta)$.
\end{proof}




\subsection{Randomized Response Mechanism:}\label{app:Randomized Response_mechanism}

The second alternative mechanism we consider is the Randomized Response mechanism. In general, it is used for discrete data like indicator functions $(\mathds{1}_{\{s_{h} =s, a_{h}=a\}})_{h,s,a}$. We therefore use it to privatize the number of visits of a state-action pair and state-action-next-state tuple for each trajectory. With the assumption that reward are supported in $[0,1]$, we can also use this mechanism for privatizing the cumulative reward of a given trajectory.
Contrary to previous ones, the output of the Randomized Response mechanism is three vectors, 
two of size $H \times S\times A$, 
and the last one of size $(H-1)\times S\times A\times S$. We slightly modify the requirements of Def.~\ref{assumption:concentration_privacy} by changing the size of the output of the privacy preserving mechanism. We summarize the mechanism in Alg.~\ref{alg:Randomized Response_mechanism}. 

\begin{algorithm}[tb]
  \caption{Randomized Response mechanism for LDP}
  \label{alg:Randomized Response_mechanism}
\begin{algorithmic}
  \STATE {\bfseries Input:} Trajectory: $X = \{(s_{h}, a_{h}, r_{h}) \mid h\leq H\}$, Privacy Parameter: $\varepsilon_{0}$
  \STATE Draw $(Y_{i, X}(s,a))_{(s,a)\in \mathcal{S}\times \mathcal{A}, i\leq 2}$ i.i.d $\mathcal{N}\left(0,\sigma^{2}\right)$ and $(Z_{X}(s,a,s'))_{(s,a,s')\in \mathcal{S}\times \mathcal{A}\times \mathcal{S}}$ i.i.d $\mathcal{N}\left(0,\sigma^{2}\right)$ and independent from $Y_{i,X}$ for $i\in\{1,2\}$ with $\sigma = cH/\varepsilon_{0}$
  \FOR{$(s,a)\in \mathcal{S}\times \mathcal{A}$}
    \FOR{$h=1, \hdots, H$}
      \STATE Sample $Y_{1,X}(h, s, a)\sim \text{Ber}\left(\frac{e^{\varepsilon_{0}} - 1}{e^{\varepsilon_{0}} + 1}r_{h}\mathds{1}_{\{s_{h} = s, a_{h} = a\}}  + \frac{1}{e^{\varepsilon_{0}} + 1}\right)$
      \STATE $\wt{R}_{X}(h, s,a) = \frac{e^{\varepsilon_{0}} + 1}{e^{\varepsilon_{0}} - 1}\left( Y_{1,X}(h,s,a) - \frac{1}{e^{\varepsilon_{0}} + 1}\right)$
      \STATE Sample $\wt{n}^{r}_{X}(h,s,a) \sim \text{Ber}\left(\frac{e^{\varepsilon_{0}} - 1}{e^{\varepsilon_{0}} + 1}\mathds{1}_{\{s_{h} = s, a_{h} = a\}}  + \frac{1}{e^{\varepsilon_{0}} + 1}\right)$
      \IF{$h<H$}
      \FOR{$s'\in \mathcal{S}$}
        \STATE Sample $\wt{n}^{p}_{X}(h,s,a,s') \sim \text{Ber}\left(\frac{e^{\varepsilon_{0}} - 1}{e^{\varepsilon_{0}} + 1}\mathds{1}_{\{s_{h} = s, a_{h} = a, s_{h+1} = s'\}}  + \frac{1}{e^{\varepsilon_{0}} + 1}\right)$
        \STATE $\wt{N}^{p}_{X}(h,s,a,s') = \frac{e^{\varepsilon_{0}} + 1}{e^{\varepsilon_{0}} - 1}\left( \wt{n}_{X}^{p}(h,s,a,s') - \frac{1}{e^{\varepsilon_{0}} +1}\right)$
      \ENDFOR
      \ENDIF
    \ENDFOR
  \ENDFOR
  \STATE {\bfseries Return:} $(\wt{R}_{X}, \wt{N}_{X}^{r}, \wt{N}_{X}^{p})\in \left\{\frac{-1}{e^{\varepsilon_{0}} - 1},\frac{e^{\varepsilon_{0}}}{e^{\varepsilon_{0}} - 1}\right\}^{HSA}\times \left\{\frac{-1}{e^{\varepsilon_{0}} - 1},\frac{e^{\varepsilon_{0}}}{e^{\varepsilon_{0}} - 1}\right\}^{HSA}\times \left\{\frac{-1}{e^{\varepsilon_{0}} - 1},\frac{e^{\varepsilon_{0}}}{e^{\varepsilon_{0}} - 1}\right\}^{(H-1)SAS}$
\end{algorithmic}
\end{algorithm}

Just as for the Gaussian mechanism, we show that Alg.~\ref{alg:Randomized Response_mechanism} satisfies Def.~\ref{assumption:concentration_privacy}. We begin by showing that this mechanism satisfies $(\varepsilon_{0}, 0)$-LDP for any $\varepsilon_{0}>0$.

\begin{proposition}\label{prop:Randomized Response_ldp}
For any $\varepsilon>0$, the Randomized Response mechanism, Alg.~\ref{alg:Randomized Response_mechanism}, with parameter $\varepsilon_{0} = \varepsilon/6H$ is $(\varepsilon, 0)$-LDP.
\end{proposition}

\begin{proof}[Proof of Prop.~\ref{prop:Randomized Response_ldp}:]
Just as in the proof of Prop.~\ref{prop:gaussian_ldp} and Prop.~\ref{prop:laplace_assumption1}, let's consider two trajectories $X = \{(s_{h}, a_{h}, r_{h}) \mid h\leq H\}$ and $X' = \{ (s_{h}', a_{h}', r_{h}') \mid h\leq H\}$
and also denote the output of the private randomizer $\mathcal{M}$ by $\mathcal{M}(X) = (\wt{R}_{X}, \wt{N}_{X}^{r}, \wt{N}_{X}^{p})$ and $\mathcal{M}(X') = (\wt{R}_{X'}, \wt{N}_{X'}^{r}, \wt{N}_{X'}^{p})$.

%

For a given $r\in \left\{\frac{-1}{e^{\varepsilon_{0}} - 1},\frac{e^{\varepsilon_{0}}}{e^{\varepsilon_{0}} - 1}\right\}^{HSA}$ (note that by definition of $r$ in Alg.~\ref{alg:Randomized Response_mechanism}, these are the only values it can take), we have that:
{\small\begin{equation}\label{eq:ratio_reward_Randomized Response}
\begin{aligned}
\frac{ \mathbb{P}\left(\forall (h,s,a),\wt{R}_{X}(h,s,a) = r_{h,s,a} \mid X\right)}{\mathbb{P}\left(\forall (h,s,a), \wt{R}_{X'}(h,s,a) = r_{h,s,a} \mid X'\right)} = \prod_{h,s,a} \left(\frac{\frac{e^{\varepsilon_{0}} - 1}{e^{\varepsilon_{0}} + 1}r_{h}\mathds{1}_{\{s_{h} = s, a_{h} = a\}}  + \frac{1}{e^{\varepsilon_{0}} + 1}}{\frac{e^{\varepsilon_{0}} - 1}{e^{\varepsilon_{0}} + 1}r_{h}'\mathds{1}_{\{s_{h}' = s, a_{h}' = a\}}  + \frac{1}{e^{\varepsilon_{0}} + 1}}\right)^{y^{r}_{h,s,a}}\times \\
\times \left(\frac{1 - \left(\frac{e^{\varepsilon_{0}} - 1}{e^{\varepsilon_{0}} + 1}r_{h}\mathds{1}_{\{s_{h} = s, a_{h} = a\}}  + \frac{1}{e^{\varepsilon_{0}} + 1}\right)}{1 - \left(\frac{e^{\varepsilon_{0}} - 1}{e^{\varepsilon_{0}} + 1}r_{h}'\mathds{1}_{\{s_{h}' = s, a_{h}' = a\}}  + \frac{1}{e^{\varepsilon_{0}} + 1}\right)}\right)^{1 - y^{r}_{h,s,a}}
\end{aligned}
\end{equation}}
where for every $(h,s,a)\in H\times \mathcal{S}\times\mathcal{A}$, we define $y^{r}_{h,s,a} = \frac{e^{\varepsilon_{0}} - 1}{e^{\varepsilon_{0}} + 1}r + \frac{1}{e^{\varepsilon_{0}}+1}$ belongs to $\{0, 1\}$ because $r\in\left\{\frac{-1}{e^{\varepsilon_{0}} - 1},\frac{e^{\varepsilon_{0}}}{e^{\varepsilon_{0}} - 1}\right\}^{HSA}$.
Eq.~\eqref{eq:ratio_reward_Randomized Response} can be rewritten as:
{\small\begin{equation}\label{eq:ratio_reward_Randomized Response_2}
\begin{aligned}
\eqref{eq:ratio_reward_Randomized Response} = \prod_{h,s,a} \left(\frac{(e^{\varepsilon_{0}} - 1)r_{h}\mathds{1}_{\{s_{h} = s, a_{h} = a\}}  + 1}{(e^{\varepsilon_{0}} - 1)r_{h}'\mathds{1}_{\{s_{h}' = s, a_{h}' = a\}}  + 1}\right)^{y^{r}_{h,s,a}}
\left(\frac{e^{\varepsilon_{0}} - (e^{\varepsilon_{0}} - 1)r_{h}\mathds{1}_{\{s_{h} = s, a_{h} = a\}}  }{e^{\varepsilon_{0}} - (e^{\varepsilon_{0}} - 1)r_{h}'\mathds{1}_{\{s_{h}' = s, a_{h}' = a\}}}\right)^{1 - y^{r}_{h,s,a}}
\end{aligned}
\end{equation}}

Then for a given $(h,s,a)$, because $r_{h}\in [0,1]$ we have:
\begin{align}
\frac{(e^{\varepsilon_{0}} - 1)r_{h}\mathds{1}_{\{s_{h} = s, a_{h} = a\}}  + 1}{(e^{\varepsilon_{0}} - 1)r_{h}'\mathds{1}_{\{s_{h}' = s, a_{h}' = a\}}  + 1} \leq \left\{\begin{matrix}
  e^{\varepsilon_{0}} & \text{ if } \mathds{1}_{\{s_{h} = s, a_{h} = a\}} = \mathds{1}_{\{s_{h}' = s, a_{h}' = a\}} = 1 \\
  1 & \text{ if } \mathds{1}_{\{s_{h} = s, a_{h} = a\}} = \mathds{1}_{\{s_{h}' = s, a_{h}' = a\}} = 0 \\
 e^{\varepsilon_{0}} & \text{ if } \mathds{1}_{\{s_{h} = s, a_{h} = a\}} = 1 \text{ and } \mathds{1}_{\{s_{h}' = s, a_{h}' = a\}} = 0\\
 1 & \text{ if } \mathds{1}_{\{s_{h} = s, a_{h} = a\}} = 0 \text{ and } \mathds{1}_{\{s_{h}' = s, a_{h}' = a\}} = 1\\
\end{matrix}\right.\\
\frac{e^{\varepsilon_{0}} - (e^{\varepsilon_{0}} - 1)r_{h}\mathds{1}_{\{s_{h} = s, a_{h} = a\}}  }{e^{\varepsilon_{0}} - (e^{\varepsilon_{0}} - 1)r_{h}'\mathds{1}_{\{s_{h}' = s, a_{h}' = a\}}} \leq  \left\{\begin{matrix}
 e^{\varepsilon_{0}} & \text{ if } \mathds{1}_{\{s_{h} = s, a_{h} = a\}} = \mathds{1}_{\{s_{h}' = s, a_{h}' = a\}} = 1 \\
 1 & \text{ if } \mathds{1}_{\{s_{h} = s, a_{h} = a\}} = \mathds{1}_{\{s_{h}' = s, a_{h}' = a\}} = 0 \\
 1 & \text{ if } \mathds{1}_{\{s_{h} = s, a_{h} = a\}} = 1 \text{ and } \mathds{1}_{\{s_{h}' = s, a_{h}' = a\}} = 0\\
 e^{\varepsilon_{0}} & \text{ if } \mathds{1}_{\{s_{h} = s, a_{h} = a\}} = 0 \text{ and } \mathds{1}_{\{s_{h}' = s, a_{h}' = a\}} = 1\\
\end{matrix}\right.
\end{align}
\todocout{what happens to the $r_h$ terms when the indicators are different? Is $r_h$ of some special form?}
\todoeout{It's because $r_{h}$ are in $[0,1]$}
Therefore, we can simplify each term in \eqref{eq:ratio_reward_Randomized Response_2} by:
\begin{align*}
&\frac{(e^{\varepsilon_{0}} - 1)r_{h}\mathds{1}_{\{s_{h} = s, a_{h} = a\}}  + 1}{(e^{\varepsilon_{0}} - 1)r_{h}'\mathds{1}_{\{s_{h}' = s, a_{h}' = a\}}  + 1} \leq \exp\left(\varepsilon_{0}\left(\mathds{1}_{\{s_{h} = s, a_{h} = a\}}+ \mathds{1}_{\{s_{h}' = s, a_{h}' = a\}} \right)\right)\\
&\frac{e^{\varepsilon_{0}} - (e^{\varepsilon_{0}} - 1)r_{h}\mathds{1}_{\{s_{h} = s, a_{h} = a\}}  }{e^{\varepsilon_{0}} - (e^{\varepsilon_{0}} - 1)r_{h}'\mathds{1}_{\{s_{h}' = s, a_{h}' = a\}}} \leq \exp\left(\varepsilon_{0}\left(\mathds{1}_{\{s_{h} = s, a_{h} = a\}}+ \mathds{1}_{\{s_{h}' = s, a_{h}' = a\}} \right)\right)
\end{align*}
Hence, using the two inequalities above:
{\small\begin{align*}
\eqref{eq:ratio_reward_Randomized Response_2} &\leq \prod_{h,s,a} \exp\left(y^{r}_{h,s,a}\varepsilon_{0}\left(\mathds{1}_{\Big\{\begin{subarray}{l}s_{h} = s, \\ a_{h} = a\end{subarray}\Big\}}+ \mathds{1}_{\Big\{\begin{subarray}{l}s_{h}' = s,\\ a_{h}' = a\end{subarray}\Big\}} \right) + (1 - y^{r}_{h,s,a})\varepsilon_{0}\left(\mathds{1}_{\Big\{\begin{subarray}{l}s_{h}' = s,\\ a_{h}' = a\end{subarray}\Big\}}+ \mathds{1}_{\Big\{\begin{subarray}{l}s_{h} = s, \\ a_{h} = a\end{subarray}\Big\}} \right)\right)\\
&=\prod_{h,s,a} \exp\left(\varepsilon_{0}\left(\mathds{1}_{\Big\{\begin{subarray}{l}s_{h} = s,\\ a_{h} = a\end{subarray}\Big\}}+ \mathds{1}_{\Big\{\begin{subarray}{l}s_{h}' = s,\\ a_{h}' = a\end{subarray}\Big\}} \right) \right)\\
&= \exp\left(2\varepsilon_{0}H \right)
\end{align*}}

In addition, let's consider $m\in\left\{ \frac{-1}{e^{\varepsilon_{0}} - 1}, \frac{e^{\varepsilon_{0}}}{e^{\varepsilon_{0}} - 1}\right\}^{H\times S\times A}$
and $y = \frac{e^{\varepsilon_{0}} - 1}{e^{\varepsilon_{0}} + 1}m + \frac{1}{e^{\varepsilon_{0}}+1} \in \{0, 1\}$, we then have that:
{\small\begin{equation}\label{eq:ratio_counters_Randomized Response}
\begin{aligned}
\frac{ \mathbb{P}\left(\forall (h,s,a),\wt{N}^{r}_{X}(h,s,a) = m_{h,s,a} \mid X\right)}{\mathbb{P}\left(\forall (h,s,a), \wt{N}_{X'}^{r}(h,s,a) = m_{h,s,a} \mid X'\right)} = \prod_{h,s,a} \left(\frac{\frac{e^{\varepsilon_{0}} - 1}{e^{\varepsilon_{0}} + 1}\mathds{1}_{\{s_{h} = s, a_{h} = a\}}  + \frac{1}{e^{\varepsilon_{0}} + 1}}{\frac{e^{\varepsilon_{0}} - 1}{e^{\varepsilon_{0}} + 1}\mathds{1}_{\{s_{h}' = s, a_{h}' = a\}}  + \frac{1}{e^{\varepsilon_{0}} + 1}}\right)^{y_{h,s,a}}\times& \\
\times \left(\frac{1 - \left(\frac{e^{\varepsilon_{0}} - 1}{e^{\varepsilon_{0}} + 1}\mathds{1}_{\{s_{h} = s, a_{h} = a\}}  + \frac{1}{e^{\varepsilon_{0}} + 1}\right)}{1 - \left(\frac{e^{\varepsilon_{0}} - 1}{e^{\varepsilon_{0}} + 1}\mathds{1}_{\{s_{h}' = s, a_{h}' = a\}}  + \frac{1}{e^{\varepsilon_{0}} + 1}\right)}\right)^{1 - y_{h,s,a}}&
\end{aligned}
\end{equation}}
Which can be rewritten as:
{\small\begin{equation}\label{eq:ratio_counters_Randomized Response_2}
\begin{aligned}
\frac{ \mathbb{P}\left(\forall (h,s,a),\wt{N}^{r}_{X}(h,s,a) = m_{h,s,a} \mid X\right)}{\mathbb{P}\left(\forall (h,s,a), \wt{N}_{X'}^{r}(h,s,a) = m_{h,s,a} \mid X'\right)} = \prod_{h,s,a} \left(\frac{(e^{\varepsilon_{0}} - 1)\mathds{1}_{\{s_{h} = s, a_{h} = a\}}  + 1}{(e^{\varepsilon_{0}} - 1)\mathds{1}_{\{s_{h}' = s, a_{h}' = a\}}  + 1}\right)^{y_{h,s,a}}\times& \\
\times \left(\frac{e^{\varepsilon_{0}} - (e^{\varepsilon_{0}} - 1)\mathds{1}_{\{s_{h} = s, a_{h} = a\}}  }{e^{\varepsilon_{0}} - (e^{\varepsilon_{0}} - 1)\mathds{1}_{\{s_{h}' = s, a_{h}' = a\}}}\right)^{1 - y_{h,s,a}}&
\end{aligned}
\end{equation}}
Thus for a given $(h,s,a)$:
\begin{align}
\frac{(e^{\varepsilon_{0}} - 1)\mathds{1}_{\{s_{h} = s, a_{h} = a\}}  + 1}{(e^{\varepsilon_{0}} - 1)\mathds{1}_{\{s_{h}' = s, a_{h}' = a\}}  + 1} = \left\{\begin{matrix}
 1 & \text{ if } \mathds{1}_{\{s_{h} = s, a_{h} = a\}} = \mathds{1}_{\{s_{h}' = s, a_{h}' = a\}} \\
 e^{\varepsilon_{0}} & \text{ if } \mathds{1}_{\{s_{h} = s, a_{h} = a\}} = 1 \text{ and } \mathds{1}_{\{s_{h}' = s, a_{h}' = a\}} = 0\\
 e^{-\varepsilon_{0}} & \text{ if } \mathds{1}_{\{s_{h} = s, a_{h} = a\}} = 0 \text{ and } \mathds{1}_{\{s_{h}' = s, a_{h}' = a\}} = 1\\
\end{matrix}\right.\\
\frac{e^{\varepsilon_{0}} - (e^{\varepsilon_{0}} - 1)\mathds{1}_{\{s_{h} = s, a_{h} = a\}}  }{e^{\varepsilon_{0}} - (e^{\varepsilon_{0}} - 1)\mathds{1}_{\{s_{h}' = s, a_{h}' = a\}}} = \left\{\begin{matrix}
 1 & \text{ if } \mathds{1}_{\{s_{h} = s, a_{h} = a\}} = \mathds{1}_{\{s_{h}' = s, a_{h}' = a\}} \\
 e^{-\varepsilon_{0}} & \text{ if } \mathds{1}_{\{s_{h} = s, a_{h} = a\}} = 1 \text{ and } \mathds{1}_{\{s_{h}' = s, a_{h}' = a\}} = 0\\
 e^{\varepsilon_{0}} & \text{ if } \mathds{1}_{\{s_{h} = s, a_{h} = a\}} = 0 \text{ and } \mathds{1}_{\{s_{h}' = s, a_{h}' = a\}} = 1\\
\end{matrix}\right.
\end{align}
Therefore, here again we can simplify each term in \eqref{eq:ratio_counters_Randomized Response_2} by:
\begin{align*}
&\frac{(e^{\varepsilon_{0}} - 1)\mathds{1}_{\{s_{h} = s, a_{h} = a\}}  + 1}{(e^{\varepsilon_{0}} - 1)\mathds{1}_{\{s_{h}' = s, a_{h}' = a\}}  + 1} \leq \exp\left(\varepsilon_{0}\left(\mathds{1}_{\{s_{h} = s, a_{h} = a\}}- \mathds{1}_{\{s_{h}' = s, a_{h}' = a\}} \right)\right)\\
&\frac{e^{\varepsilon_{0}} - (e^{\varepsilon_{0}} - 1)\mathds{1}_{\{s_{h} = s, a_{h} = a\}}  }{e^{\varepsilon_{0}} - (e^{\varepsilon_{0}} - 1)\mathds{1}_{\{s_{h}' = s, a_{h}' = a\}}} \leq \exp\left(\varepsilon_{0}\left(\mathds{1}_{\{s_{h} = s, a_{h} = a\}}- \mathds{1}_{\{s_{h}' = s, a_{h}' = a\}} \right)\right)
\end{align*}
Therefore:
{\small\begin{align*}
\eqref{eq:ratio_counters_Randomized Response_2} &= \prod_{h,s,a} \exp\left(y_{h,s,a}\varepsilon_{0}\left(\mathds{1}_{\Big\{\begin{subarray}{l}s_{h} = s,\\ a_{h} = a\end{subarray}\Big\}}- \mathds{1}_{\Big\{\begin{subarray}{l}s_{h}' = s,\\ a_{h}' = a\end{subarray}\Big\}}\right) + (1 - y_{h,s,a})\varepsilon_{0}\left(\mathds{1}_{\Big\{\begin{subarray}{l}s_{h}' = s,\\ a_{h}' = a\end{subarray}\Big\}} - \mathds{1}_{\Big\{\begin{subarray}{l}s_{h} = s,\\ a_{h} = a\end{subarray}\Big\}} \right)\right)\\
&=\prod_{h,s,a} \exp\left((2y_{h,s,a} - 1)\varepsilon_{0}\left(\mathds{1}_{\{s_{h} = s, a_{h} = a\}}- \mathds{1}_{\{s_{h}' = s, a_{h}' = a\}} \right) \right)\\
&\leq \exp\left(2\varepsilon_{0}H \right)
\end{align*}}
\todocout{same questions here}
\todoeout{Done}
Using the same reasonning we have that for any $m'\in \left\{ - \frac{1}{e^{\varepsilon_{0}} - 1}, \frac{e^{\varepsilon_{0}}}{e^{\varepsilon_{0}} - 1}\right\}^{(H-1)\times S\times A\times S}$:
\begin{align}
\frac{\mathbb{P}\left(\forall (h,s,a,s'),\wt{N}^{p}_{X}(h,s,a,s') = m_{h,s,a,s'}' \mid X\right)}{\mathbb{P}\left(\forall (h,s,a,s'), \wt{N}_{X'}^{p}(h,s,a,s') = m_{h,s,a,s'}' \mid X'\right)} \leq \exp(2\varepsilon_{0}H)
\end{align}
We conclude the proof the same way as the proof of Prop.~\ref{thm:ldp_laplace}.
\end{proof}

In addition, the precision $c_{k,1}$, $c_{k,2}$, $c_{k,3}$ and $c_{k,4}$ of the Randomized Response mechanism are still of order $\sqrt{k}$ just as the Gaussian and Laplace mechanisms.
Contrary to any of those two, the dependence is exponential on $\varepsilon_{0}$ which is closer to the lower bound of Sec.~\ref{sec:lower_bound}. Indeed, we have an additional factor $S$ for $c_{k,3}$ compared to the other mechanisms but those terms scale with $1/(e^{\varepsilon_{0}} - 1)$ instead of the worse dependency $1/\varepsilon$.

\begin{proposition}\label{prop:utility_Randomized Response_mechanism}
The Randomized Response mechanism, Alg.~\ref{alg:Randomized Response_mechanism}, with parameter $\varepsilon_{0}>0$ satisfies Def.~\ref{assumption:concentration_privacy} for any $\delta>0$ and $k\in \mathbb{N}^{\star}$ with:
\begin{align*}
&c_{k,1}(\varepsilon_0, \delta) = c_{k,2}(\varepsilon_0, \delta)  =  \max\left\{1, \frac{2e^{\varepsilon_{0}} - 1}{e^{\varepsilon_{0}} - 1}\sqrt{\frac{(k-1)H}{2}\ln\left(\frac{4SA}{\delta}\right)}\right\}\\
&c_{k,3}(\varepsilon_0,\delta) =\max\left\{1, \frac{S(2e^{\varepsilon_{0}} - 1)}{e^{\varepsilon_{0}} - 1}\sqrt{\frac{(k-1)H}{2}\ln\left(\frac{4SA}{\delta}\right)}\right\}\\
&c_{k,4}(\varepsilon_0, \delta)  = \max\left\{1, \frac{2e^{\varepsilon_{0}} - 1}{e^{\varepsilon_{0}} - 1}\sqrt{\frac{(k-1)H}{2}\ln\left(\frac{4S^{2}A}{\delta}\right)}\right\}
\end{align*}
\end{proposition}

\begin{proof}[Proof of Prop.~\ref{prop:utility_Randomized Response_mechanism}:]

Let's consider a given state-action-next state tuple, $(s,a,s')$, then when summing over $h$:
\begin{align}
\left|\sum_{h=1}^{H} \wt{N}_{k}^{r}(h,s,a) - \sum_{l<k}\sum_{h=1}^{H} \mathds{1}_{\{s_{l,h} = s, a_{l,h} = a\}} \right| = \left|\sum_{h=1}^{H}\sum_{l<k} \wt{N}_{X_{l}}^{r}(h,s,a) - \mathds{1}_{\{s_{l,h} = s, a_{l,h} = a\}} \right|
\end{align}
We now construct a filtration $(\mathcal{F}_{k,h})_{k,h}$ such that $(\wt{N}_{X_{l}}^{r}(h,s,a) - \mathds{1}_{\{s_{l,h} = s, a_{l,h} = a\}})_{l,h}$ is a Martingale Difference Sequence. For an episode $k$ and step $h$, define $\mathcal{F}_{k,h} = \sigma(\{(s_{l,j}, a_{l,j}, r_{l,j})_{j\leq H}, \mathcal{M}((s_{l,j}, a_{l,j}, r_{l,j})_{j\leq H})\} \mid l < k\} \cup \{(s_{k,j}, a_{k,j}, r_{k,j})_{j\leq h}\})$
 to be the filtration that contains the history before episode $k$. Then 
 $\mathds{1}_{\{s_{k,h}= s, a_{k,h} = a\}}$ is $\mathcal{F}_{k,h}$-measurable and thus we have:
 \begin{align*}
\mathbb{E}\left(\wt{N}_{X_{k}}^{r}(h,s,a) - \mathds{1}_{\{s_{k,h}= s, a_{k,h} = a\}}\mid \mathcal{F}_{k,h} \right) = \frac{e^{\varepsilon_{0}}+1}{e^{\varepsilon_{0}}-1}\left(\mathbb{E}\left(\wt{n}_{X_{k}}(h,s,a)\mid \mathcal{F}_{k,h} \right) - \frac{1}{e^{\varepsilon_{0}} + 1}\right)& \\
- \mathds{1}_{\{s_{k,h} = s, a_{k,h} = a\}}= 0&
 \end{align*}
 where $\tilde{n}_{X_{k}}(h,s,a)$ is a Randomized Response random variable generated by Alg.~\ref{alg:Randomized Response_mechanism} for each step $h$, state $s$, action $a$ and trajectory $X_{k}$.
 \todocout{remind us what $\tilde n$ is}
 \todoeout{Done}
 And $\left| \wt{N}_{X_{k}}^{r}(h,s,a) - \mathds{1}_{\{s_{k,h}= s, a_{k,h} = a\}}\right| \leq \frac{2e^{\varepsilon_{0}} - 1}{e^{\varepsilon_{0}} - 1}$. Then thanks to Azuma-Hoeffding inequality we have that with probability at least $1 - \delta/(4SA)$:
 \begin{align}
  \left|\sum_{h=1}^{H} \wt{N}_{k}^{r}(h,s,a) - \sum_{l<k}\sum_{h=1}^{H} \mathds{1}_{\{s_{l,h} = s, a_{l,h} = a\}} \right| \leq \frac{2e^{\varepsilon_{0}} - 1}{e^{\varepsilon_{0}} - 1}\sqrt{\frac{(k-1)H}{2}\ln\left(\frac{4SA}{\delta}\right)}
 \end{align}
With the same reasonning, we have with probability at least $1 - \delta/4S^{2}A$:
{\small\begin{align}
\left|\sum_{h=1}^{H} \wt{N}_{k}^{p}(h,s,a,s') - \sum_{l<k}\sum_{h=1}^{H-1} \mathds{1}_{\{s_{l,h} = s, a_{l,h} = a, s_{l,h+1} = s'\}} \right| \leq \frac{2e^{\varepsilon_{0}} - 1}{e^{\varepsilon_{0}} - 1}\sqrt{\frac{(k-1)H}{2}\ln\left(\frac{4S^{2}A}{\delta}\right)}
\end{align}}
Also, we have:
\begin{align}
 \left|\sum_{h=1}^{H} \wt{R}_{k}^{r}(h,s,a) - \sum_{l<k}\sum_{h=1}^{H} r_{h}\mathds{1}_{\{s_{l,h} = s, a_{l,h} = a\}} \right| \leq \frac{2e^{\varepsilon_{0}} - 1}{e^{\varepsilon_{0}} - 1}\sqrt{\frac{(k-1)H}{2}\ln\left(\frac{4SA}{\delta}\right)}
\end{align}
with $\wt{R}_{k}^{r}(h,s,a) = \sum_{l<k} \wt{R}_{X_{l}}$.
Finally, with probability at least $1 - \delta/4SA$:
{\small\begin{align}
\left|\sum_{h=1}^{H}\sum_{s'} \wt{N}_{k}^{p}(h,s,a,s') - \sum_{s'}\sum_{l<k}\sum_{h=1}^{H-1} \mathds{1}_{\Big\{\begin{subarray}{l}s_{l,h} = s,\\ a_{l,h} = a,\\ s_{l,h+1} = s'\end{subarray}\Big\}}\right| \leq \frac{S(2e^{\varepsilon_{0}} - 1)}{e^{\varepsilon_{0}} - 1}\sqrt{\frac{(k-1)H}{2}\ln\left(\frac{4SA}{\delta}\right)}
\end{align}}
Compared to the bounds we derived for previous mechanisms there is an additional factor $\sqrt{S}$. This comes from using a triangular inequality instead of using concentration inequalities like in previous mechanisms.
Then thanks to a union bound over the state-action pair and the state-action-next state tuple we have that the Randomized Response mechanism satisfies Def.~\ref{assumption:concentration_privacy} with:
\begin{align}
c_{k,1}(\varepsilon_{0}, \delta) = c_{k,2}(\varepsilon_{0}, \delta) = \max\left\{1, \frac{2e^{\varepsilon_{0}} - 1}{e^{\varepsilon_{0}} - 1}\sqrt{\frac{(k-1)H}{2}\ln\left(\frac{4SA}{\delta}\right)}\right\}\\
c_{k,3}(\varepsilon_{0}, \delta) = \max\left\{1, \frac{S(2e^{\varepsilon_{0}} - 1)}{e^{\varepsilon_{0}} - 1}\sqrt{\frac{(k-1)H}{2}\ln\left(\frac{4SA}{\delta}\right)}\right\},\\
c_{k,4}(\varepsilon_{0}, \delta) = \max\left\{1, \frac{2e^{\varepsilon_{0}} - 1}{e^{\varepsilon_{0}} - 1}\sqrt{\frac{(k-1)H}{2}\ln\left(\frac{4S^{2}A}{\delta}\right)}\right\}
\end{align}
\end{proof}




\subsection{Bounded Noise Mechanism for DP:}\label{app:bounded_noise}

Recently, \cite{dagan2020bounded} showed how to construct a differential privacy with an almost surely bounded noise mechanism.
This mechanism, $\mathcal{M}$, computes an $(\varepsilon, \delta)$-DP approximation of the average of a dataset $\mathcal{D} = \{x_{1}, \hdots, x_{n}\}\subset \mathbb{R}^{n\times k}$, for any $\varepsilon>0$ and $\delta\in [\exp(- k/\log(k)^{8}), 1/2]$ (see Theorem $1.1$ in \citep{dagan2020bounded}).
In the local differentially private setting in RL, we apply this bounded noise mechanism to each user $k$ in order to compute the cumulative reward for each
state-action $(s,a)$, the number of visits to $(s,a)$ and the number of visits to state-action-next state tuple $(s,a,s')$.

This noise mechanism is similar to the Laplace or Gaussian mechanism and add a noise drawn from a well-chosen distribution, $\mu_{\text{DE}, R}$
 supported on $(-R, R)$ for any $R$, whose density at $\eta\in (-R, R)$ is:
\begin{align}
  \frac{\exp(- f_{\text{DE}, \text{R}}(\eta))}{Z_{\text{DE}, \text{R}}} \text{ with } f_{\text{DE}, \text{R}}(\eta) = \exp\left(\frac{R^{2}}{R^{2} - \eta^{2}} \right) \text{ and } Z_{\text{DE}, \text{R}} = \int_{-R}^{R} e^{- f_{\text{DE}, \text{R}}(\eta)} d\eta
\end{align}

\cite{dagan2020bounded} shows that when taking $\delta \geq \exp(-k/\log(k)^{8})$ and $\varepsilon\in (0,1)$ there exists
 a universal constant $C>0$ such that when taking $R = \frac{C}{\varepsilon n}\sqrt{k\log\left(\frac{1}{\delta}\right)}$ adding noise
 from $\mu_{\text{DE}, \text{R}}$ ensures $(\varepsilon, \delta)$-DP to the average of $n$ data of dimension $k$.

\begin{algorithm}[tb]
  \caption{Bounded Noise Mechanism for LDP}
  \label{alg:bounded_noise_dp}
  \begin{algorithmic}
  \STATE {\bfseries Input:} Trajectory: $X = \{(s_{h}, a_{h}, r_{h}) \mid h\leq H\}$, Privacy Parameter: $\varepsilon, \delta$, Constant: $C$
  \STATE Set $R_{1} = \frac{C}{\varepsilon}\sqrt{SA\ln(1/\delta)}$ and $R_{2} = \frac{CS}{\varepsilon}\sqrt{A\ln(1/\delta)}$
  \FOR{$(s,a)\in \mathcal{S}\times \mathcal{A}$}
    \STATE Sample $Y_{1,X}(s, a)\sim \mu_{\text{DE}, \text{R}_{1}}$
    \STATE $\wt{R}_{X}(s,a) = Y_{1,X}(s,a) + \sum_{h=1}^{H} r_{h}\mathds{1}_{\{s_{h} = s, a_{h} = a\}}$
    \STATE Sample $\wt{n}^{r}_{X}(s,a) \sim \mu_{\text{DE}, \text{R}_{1}}$
    \STATE $\wt{N}^{r}_{X}(s,a) = \wt{n}^{r}_{X}(s,a) + \sum_{h=1}^{H} \mathds{1}_{\{s_{h} = s, a_{h} = a\}}$
    \FOR{$s'\in \mathcal{S}$}
        \STATE Sample $\wt{n}^{p}_{X}(s,a, s') \sim \mu_{\text{DE}, \text{R}_{2}}$
        \STATE $\wt{N}^{p}_{X}(s,a,s') = \wt{n}^{r}_{X}(s,a, s') + \sum_{h=1}^{H-1} \mathds{1}_{\{s_{h} = s, a_{h} = a, s_{h+1} = s'\}}$
    \ENDFOR
  \ENDFOR
  \STATE {\bfseries Return:} $(\wt{R}_{X}, \wt{N}_{X}^{r}, \wt{N}_{X}^{p})\in \mathbb{R}^{S\times A}\times \mathbb{R}^{S\times A}\times \mathbb{R}^{S\times A\times S}$
\end{algorithmic}
\end{algorithm}

Similarly to the previous mechanisms we studied we can show the following proposition, which states the parameter we need to use
to ensure $(\varepsilon, \delta)$-DP.

\begin{proposition}\label{prop:privacy_bounded_noise}
  For any $\varepsilon\in (0,1)$, $\delta_{0} \geq \exp( - SA/ \log(SA)^{8})$ and $\delta_{1} \geq \exp( - S^{2}A/ \log(S^{2}A)^{8})$ then the bounded noise mechanism, Alg.~\ref{alg:bounded_noise_dp}, is $(3H\varepsilon, \delta')$-LDP
  with $\delta_{0}' =  \delta_{0} \frac{e^{H\varepsilon} - 1}{e^{\varepsilon} - 1}$, $\delta_{1}' = \delta_{1} \frac{e^{H\varepsilon} - 1}{e^{\varepsilon} - 1}$ and $\delta' = \delta_{1}'e^{2H\varepsilon}
  + 2\delta_{0}'e^{2H\varepsilon}  + 2\delta_{0}'\delta_{1}'e^{H\varepsilon} + (\delta_{0}')^{2}e^{H\varepsilon} + (\delta_{0}')^{2}\delta_{1}'$.
\end{proposition}
\begin{proof}{of Prop.~\ref{prop:privacy_bounded_noise}}

  For any $\varepsilon\in (0,1)$ and $\delta_{0}\geq \exp( - SA/ \log(SA)^{8})$, for any $r\in \mathbb{R}^{S\times A}$ and two trajectories $X = \{(s_{h}, a_{h}, r_{h})_{h\leq H}\}$ and $X' = \{(s_{h}', a_{h}', r_{h}')_{h\leq H}\}$
  let's define $R_{X}(s,a) = \sum_{h=1}^{H} r_{h}\mathds{1}_{\{s_{h} = s, a_{h} = a\}}$ the cumulative reward in state-action $(s,a)$ associated to trajectory
  $X$. Finally, let's define for a set of indexes $I \subset [H] \rrbracket$ the new trajectory $X_{I}$ where for $h\in I$, $(X_{I})_{h} = (s_{h}, a_{h}, r_{h})$ and
  for $h\not\in I$, $(X_{I})_{h} = (s_{h}', a_{h}', r_{h}')$. Therefore, using Theorem $3.2$ from \cite{dagan2020bounded}, we have that for $I = [H-1] \rrbracket$ and $\wt R_X$ defined as in Alg.~\ref{alg:bounded_noise_dp},
  \begin{align}
    \mathbb{P}\left(\wt{R}_{X} = r \right) &\leq \exp(\varepsilon)\mathbb{P}\left(\wt{R}_{X_{I}} = r \right) + \delta_{0}\\
    &\leq \exp(\varepsilon)\left( \exp(\varepsilon) \mathbb{P}\left( \wt{R}_{X_{[H-2]}} = r\right) + \delta_{0}\right) + \delta_{0}
  \end{align}
  Therefore repeating the same argument $H$ times, we have that:
  \begin{align}
    \mathbb{P}\left(\wt{R}_{X} = r \right) &\leq \exp(H\varepsilon)\mathbb{P}\left(\wt{R}_{X'} = r \right) + \delta_{0}\sum_{h=0}^{H-1} \exp(h\varepsilon)
   \\&= \exp(H\varepsilon)\mathbb{P}\left(\wt{R}_{X'} = r \right) + \delta_{0}\frac{\exp(H\varepsilon) - 1}{\exp(\varepsilon) - 1}
  \end{align}

In addition, we have with the same reasoning that for any $n\in\mathbb{R}^{S\times A}$ and $n^{p}\in \mathbb{R}^{S\times A\times S}$ that:
\begin{align}
  \mathbb{P}\left(\wt{N}_{X}^{r} = n \right)  &\leq \exp(H\varepsilon)\mathbb{P}\left(\wt{N}_{X'} = n\right) + \delta_{0}\frac{\exp(H\varepsilon) - 1}{\exp(\varepsilon) - 1}
\end{align}
and for any $\delta_{1}\geq \exp( - S^{2}A/ \log(S^{2}A)^{8})$:
\begin{align}
  \mathbb{P}\left(\wt{N}_{X}^{p} = n^{p} \right)  &\leq \exp(H\varepsilon)\mathbb{P}\left(\wt{N}_{X'}^{p} = n^{p}\right) + \delta_{1}\frac{\exp(H\varepsilon) - 1}{\exp(\varepsilon) - 1}
\end{align}

Therefore we have that:
{\small\begin{align*}
  \mathbb{P}&\left(\wt{R}_{X} = r, \wt{N}_{X}^{r} = n, \wt{N}_{X}^{p} = n^{p} \mid X \right) = \mathbb{P}\left(\wt{R}_{X} = r\mid X\right)\mathbb{P}\left(\wt{N}_{X}^{r} = n \mid X\right) \mathbb{P}\left(\wt{N}_{X}^{p} = n^{p}\mid X\right)\\
  &\leq \left( e^{H\varepsilon}\mathbb{P}\left(\wt{R}_{X'} = r \right) + \delta_{0}\frac{e^{H\varepsilon} - 1}{e^{\varepsilon} - 1} \right)
  \left( e^{H\varepsilon}\mathbb{P}\left(\wt{N}_{X'} = n\right) + \delta_{0}\frac{e^{H\varepsilon} - 1}{e^{\varepsilon} - 1}\right)\times\\
  &\times \left( e^{H\varepsilon}\mathbb{P}\left(\wt{N}_{X'}^{p} = n^{p}\right) + \delta_{1}\frac{e^{H\varepsilon} - 1}{e^{\varepsilon} - 1}\right)\\
  &\leq   e^{3H\varepsilon}\mathbb{P}\left( \wt{R}_{X'} = r, \wt{N}_{X'}^{r} = n, \wt{N}_{X'}^{p} = n^{p}\right) + \delta_{1}'e^{2H\varepsilon}\mathbb{P}\left(\wt{R}_{X'} = r \right)\mathbb{P}\left(\wt{N}_{X'}^{r} = n \right)\\
  &+ \delta_{0}'e^{2H\varepsilon} \mathbb{P}\left( \wt{N}_{X'}^{p} = n^{p}\right)\left( \mathbb{P}\left( \wt{N}_{X'}^{r} = n\right) + \mathbb{P}\left( \wt{R}_{X'} = r\right)\right) \\
  &+ \delta_{0}'\delta_{1}'e^{H\varepsilon}\left( \mathbb{P}\left( \wt{N}_{X'}^{r} = n\right) + \mathbb{P}\left( \wt{R}_{X'} = r\right)\right) + (\delta_{0}')^{2}e^{H\varepsilon}\mathbb{P}\left( \wt{N}_{X'}^{p} = n^{p}\right) + (\delta_{0}')^{2}\delta_{1}'
\end{align*}}
with $\delta_{0}' = \delta_{0} \frac{e^{H\varepsilon} - 1}{e^{\varepsilon} - 1}$ and $\delta_{1}' = \delta_{1} \frac{e^{H\varepsilon} - 1}{e^{\varepsilon} - 1}$. Therefore, we have that the mechanism
is $(3H\varepsilon, \delta')$-LDP that is to say:
{\small\begin{align*}
  \mathbb{P}\left(\wt{R}_{X} = r, \wt{N}_{X}^{r} = n, \wt{N}_{X}^{p} = n^{p} \mid X \right) &\leq e^{3H\varepsilon}\mathbb{P}\left( \wt{R}_{X'} = r, \wt{N}_{X'}^{r} = n, \wt{N}_{X'}^{p} = n^{p}\right) + \delta_{1}'e^{2H\varepsilon}\\
  &+ 2\delta_{0}'e^{2H\varepsilon}  + 2\delta_{0}'\delta_{1}'e^{H\varepsilon} + (\delta_{0}')^{2}e^{H\varepsilon} + (\delta_{0}')^{2}\delta_{1}'
\end{align*}}
with $\delta_{0}' =  \delta_{0} \frac{e^{H\varepsilon} - 1}{e^{\varepsilon} - 1}$, $\delta_{1}' = \delta_{1} \frac{e^{H\varepsilon} - 1}{e^{\varepsilon} - 1}$ and $\delta' = \delta_{1}'e^{2H\varepsilon}
+ 2\delta_{0}'e^{2H\varepsilon}  + 2\delta_{0}'\delta_{1}'e^{H\varepsilon} + (\delta_{0}')^{2}e^{H\varepsilon} + (\delta_{0}')^{2}\delta_{1}'$.
\end{proof}

In addition, because the noise is bounded we can apply standard sub-gaussian concentration inequalities to show that Alg.~\ref{alg:bounded_noise_dp} satisfies
Def.~\ref{def:RL-LDP}.

\begin{proposition}\label{prop:concentration_bounded_noise}
  The bounded noise mechanism, Alg.~\ref{alg:bounded_noise_dp}, with parameter $\varepsilon_{0}>0$ satisfies Def.~\ref{assumption:concentration_privacy} for any $\delta>0$ and $k\in \mathbb{N}^{\star}$ with:
  \begin{align*}
  &c_{k,1}(\varepsilon_0, \delta) = c_{k,2}(\varepsilon_0, \delta)  =  R\sqrt{2(k-1)\ln\left(\frac{6SA}{\delta}\right)}\\
  &c_{k,3}(\varepsilon_0,\delta) = R_{2}\sqrt{2S(k-1)\ln\left(\frac{6S^{2}A}{\delta}\right)}\\
  &c_{k,4}(\varepsilon_0, \delta)  = R_{2}\sqrt{2(k-1)\ln\left(\frac{6S^{2}A}{\delta}\right)}
  \end{align*}
  with $R = \frac{1}{\varepsilon}\sqrt{SA\ln(1/\delta_{0})}$ and $R_{2} = \frac{S}{\varepsilon}\sqrt{A\ln(1/\delta_{0})}$
\end{proposition}
\begin{proof}{of Prop.~\ref{prop:concentration_bounded_noise}}
  For any $\delta>0$ and at the beginning of episode $k$, we have thanks to Hoeffding inequality that with probability at least $1 - \frac{\delta}{3SA}$ for any state-action $(s,a)\in\mathcal{S}\times \mathcal{A}$:
  \begin{align}
    \left|\wt{R}_{k}(s,a) - R_{k}(s,a)\right| = \left|\sum_{l=1}^{k-1} Y_{1, X_{l}}(s,a)\right| \leq R\sqrt{2(k-1)\ln\left(\frac{6SA}{\delta}\right)}
  \end{align}
  with $(Y_{1, X_{l}}(s,a))_{l\leq k-1}$ are i.i.d distributed according to $\mu_{\text{DE}, \text{R}_{1}}$. With the same reasonning, we have that
  with probability at least $1 - \frac{\delta}{3SA}$:
  \begin{align}
    \left|\wt{N}_{k}^{r}(s,a) - N_{k}^{r}(s,a)\right| = \left|\sum_{l=1}^{k-1} \wt{n}_{X_{l}}^{r}(s,a)\right| \leq R\sqrt{2(k-1)\ln\left(\frac{6SA}{\delta}\right)}
  \end{align}
Finally, still using Hoeffding inequality, and definning $R_{2} = \frac{CS}{\varepsilon}\sqrt{A\ln(1/\delta)}$, we have that with probability at least $1 - \frac{\delta}{3S^{2}A}$:
\begin{align}
  \left|\wt{N}_{k}^{p}(s,a,s') - \sum_{l<k}\sum_{h=1}^{H-1} \mathds{1}_{\{s_{l,h} = s, a_{l,h} = a, s_{l,h+1} = s'\}} \right| \leq R_{2}\sqrt{2(k-1)\ln\left(\frac{6S^{2}A}{\delta}\right)}
\end{align}
And finally with probability at least $1 - \frac{\delta}{3SA}$:
{\small\begin{align}
  \left|\sum_{s'\in \mathcal{S}}\wt{N}_{k}^{p}(s,a,s') - \sum_{s'\in \mathcal{S}}\sum_{l<k}\sum_{h=1}^{H-1} \mathds{1}_{\{s_{l,h} = s, a_{l,h} = a, s_{l,h+1} = s'\}} \right| \leq R_{2}\sqrt{2S(k-1)\ln\left(\frac{6S^{2}A}{\delta}\right)}
\end{align}}
\end{proof}

\subsection{Experimental Results:}

  We show empirical results for three mechanisms discussed in the RandomMDP
  environment in Figures~\ref{fig:epsilon_0_2},~\ref{fig:epsilon_2} and~\ref{fig:epsilon_20}.

  \begin{figure}[!h]
  \begin{minipage}{0.45\linewidth}
    \centering
    \includegraphics[width=\linewidth]{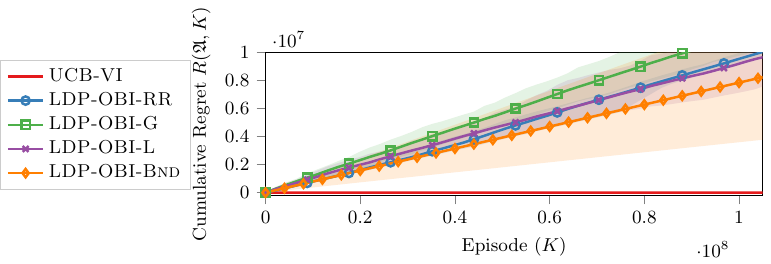}
    \caption{$\varepsilon = 0.2$ and $\delta = 0.1$ (\textit{only} for the Gaussian and bounded noise mechanism)}
    \label{fig:epsilon_0_2}
  \end{minipage}\hfill
  \begin{minipage}{0.45\linewidth}
    \centering
    \includegraphics[width=\linewidth]{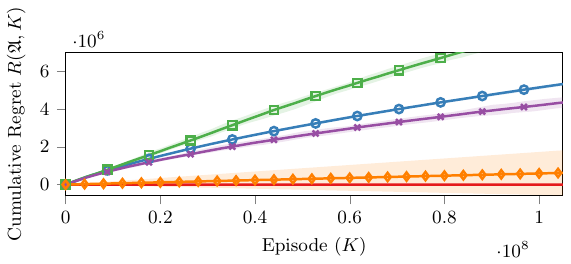}
    \caption{$\varepsilon = 2$ and $\delta = 0.1$ (\textit{only} for the Gaussian and bounded noise mechanism)}
    \label{fig:epsilon_2}
  \end{minipage}\\
  \begin{minipage}{\linewidth}
    \centering
    \includegraphics[width=0.5\linewidth]{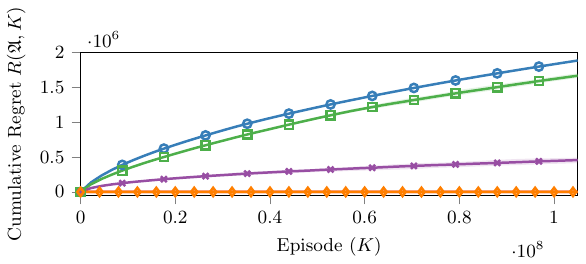}
    \caption{$\varepsilon = 20$ and $\delta = 0.1$ (\textit{only} for the Gaussian and bounded noise mechanism)}
    \label{fig:epsilon_20}
  \end{minipage}\\
  \end{figure}

  As we have seen in Fig.~\ref{fig:exp.randommdp}, the LDP constraint has a significant impact on the
  regret especially as $\varepsilon$ decreases. In particular for $\varepsilon = 0.2$,
  \algol, \algog, \algob, \algobnd have not reached the usual square root growth phase of the regret
  usually seen in \ucbvi or other regret minimizing algorithm.

From figures ~\ref{fig:epsilon_0_2},~\ref{fig:epsilon_2} and~\ref{fig:epsilon_20}, we can observe that the
bounded noise mechanism has a lower impact on the regret compared to the Laplace, Gaussian and Randomized Response mechanisms.
However, this benefit does not appear in the regret bound of Table~\ref{tab:algovariants.guarantees}.
This suggests that the regret analysis of Sec.~\ref{sec:regret} may be improved to show this empirically observed advantage.

\section{Posterior Sampling for Local Differential Privacy}\label{app:psrl_ldp}
The Posterior Sampling for Reinforcement Learning algorithm~\citep[\psrl,][]{Osband2013more} is a Thompson Sampling based algorithm for Reinforcement Learning. It works by maintaining a Bayesian posterior distribution over MDPs.
We focus on a particular instantiation of \psrl where for each state-action pair $(s,a)$ we have an independent Gaussian prior for the reward distribution and a Dirichlet prior for the transition dynamics.
With those priors, the posterior distributions are Normal-Gamma and Dirichlet distributions.

Let $\alpha_0(s,a)$ denote the parameters of the prior distribution over the transition dynamics, so the prior is given by $\text{Dir}(\alpha_{0}(s,a))$.  In addition, let $\mu_{0}(s,a)\in\mathbb{R}$,
$\lambda_{0}(s,a)\in \mathbb{R}_{+}^{\star}$, $\nu_{0}(s,a)\in\mathbb{R}_{+}^{\star}$ and $\beta_{0}(s,a)\in\mathbb{R}^{\star}_{+}$ be the parameters of the Normal-Gamma prior distribution that we place on the rewards.
Then, at the beginning of episode $k$ and for a given pair $(s,a)\in \mathcal{S}\times \mathcal{A}$, let $\alpha_{k}(s,a)\in (\mathbb{R}^{\star}_{+})^{S}$ be such that the posterior distribution over the transition dynamics is $\text{Dir}(\alpha_{k}(s,a))$. We then define $\mu_{k}(s,a)\in\mathbb{R}$,
$\lambda_{k}(s,a)\in \mathbb{R}_{+}^{\star}$, $\nu_{k}(s,a)\in\mathbb{R}_{+}^{\star}$ and $\beta_{k}(s,a)\in\mathbb{R}^{\star}_{+}$ to the parameters of the Normal-Gamma posterior distributions. Using standard results from Bayesian Learning we have that, for all state $s'\in\mathcal{S}$:
{\small\begin{align}
&\alpha_{k}(s,a) = \alpha_{0}(s,a) + N_{k}(s,a,s') \label{eq:updates_alpha}\\
&\lambda_{k}(s,a) = \lambda_{0}(s,a) + N_{k}(s,a)\label{eq:update_lambda}\\
&\nu_{k}(s,a) = \nu_{0}(s,a) + \frac{N_{k}(s,a)}{2} \label{eq:update_nu}\\
&\mu_{k}(s,a) = \frac{\lambda_{0}(s,a)\mu_{0}(s,a) + N_{k}(s,a)\hat{R}_{k}(s,a)}{\lambda_{0}(s,a) + N_{k}(s,a)} \label{eq:update_mu}\\
&\beta_{k}(s,a) = \beta_{0}(s,a) + \frac{1}{2}\wh{\text{Var}}(R(s,a)) + \frac{N_{k}(s,a)\lambda_{0}(s,a)}{2(\lambda_{0}(s,a)+ N_{k}(s,a))}\left(\hat{R}_{k}(s,a) - \mu_{0}(s,a)\right)^{2} \label{eq:update_beta}
\end{align}}
where $\alpha_{0}, \mu_{0}, \lambda_{0}, \nu_{0}, \beta_{0}$ are prior parameters provided at the beginning of the algorithm. We denote by $N_{k}(s,a)$, the number of visits to the state-action pair $(s,a)$, $N_{k}(s,a,s')$ the number visits to $(s,a,s')$, $\hat{R}_{k}(s,a)$ the average reward observed for $(s,a)$ and $\wh{\text{Var}}(R(s,a))$ the empirical variance for $(s,a)$.

At each episode $k$, \psrl samples an MDP from the posterior distributions, then computes the optimal policy and executes it in the true MDP. \citep{Osband2013more} showed that the \textit{Bayesian} regret of this algorithm is bounded by $\tilde{O}\left(HS\sqrt{AT}\right)$.

\begin{figure}[t]
   \begin{minipage}{0.65\linewidth}
         \centering
         \includegraphics[width=\linewidth]{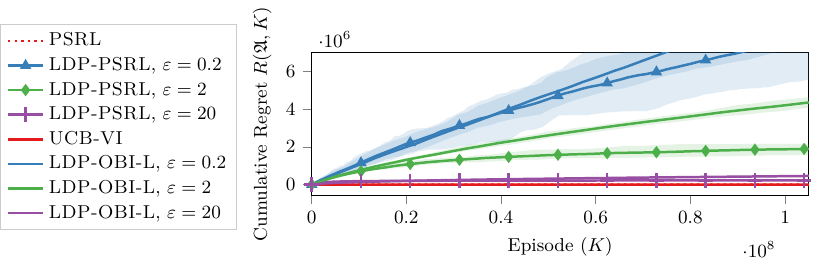}

   \end{minipage}\hfill
   \begin{minipage}{0.3\linewidth}
      \centering
      \includegraphics[width=\linewidth]{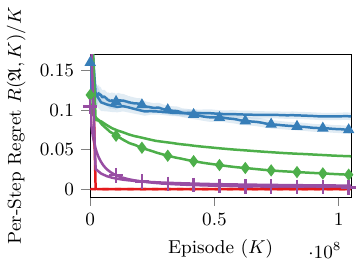}
   \end{minipage}
   \caption{Evaluation of \ppsrl in the RandomMDP environment. \emph{Left)} Cumulative regret. \emph{Right)} per-step regret ($k \mapsto R_{k}/k$). Results are averaged over $20$ runs and the {\color{black} the confidence intervals are the minimum and maximum runs}. While the regret looks almost linear for $\varepsilon=0.2$, the decreasing trend of the per-step regret shows that the algorithms are learning.}
      \label{fig:exp.randommdp_psrl}
\end{figure}

\paragraph{Locally Differentially Private Posterior Sampling for Reinforcement Learning:}
We now discuss how to adapt PSRL to ensure it is locally differentially private.
Def.~\ref{def:RL-LDP} states that LDP is ensured at the collection time of trajectories therefore it is enough for us to design a LDP posterior sampling algorithm which takes as input the trajectories outputted by a mechanism similar to Alg.~\ref{alg:laplace_mechanism}.
Here, we use the LDP mechanism to pertub the statistics used to define the parameters of the posterior distribution in \psrl.
More precisely, we replace the aggregate counts in Eqs.~\ref{eq:updates_alpha}-\ref{eq:update_beta} by noisy counts provided by an LDP mechanism. In order to do this, we need to modify the initial values of those parameters to guarantee they are non-negative.

In this appendix, we assume that the privacy-preserving mechanism $\mathcal{M}$ is such that for a given trajectory $X$, $\mathcal{M}(X) = (\wt{R}_{X}, \wt{R}_{2,X}, \wt{N}^{r}_{X}, \wt{N}^{p}_{X})$ where $\wt{R}_{X}, \wt{R}_{2,X}, \wt{N}^{r}_{X}$ and $\wt{N}^{p}_{X}$ are noisy version of the following aggregate statistics:
\begin{align*}
&R_X(s,a) = \sum_{h=1}^{H} r_{h}\mathds{1}_{\{s_{h} = s, a_{h} = a\}},
&&R_{2,X}(s,a) = \sum_{h=1}^{H} r_{h}^{2}\mathds{1}_{\{s_{h} = s, a_{h} = a\}} \\
& N^r_X(s,a) = \sum_{h=1}^{H} \mathds{1}_{\{s_{h} = s, a_{h} = a\}},
&& N^p_X(s,a,s') = \sum_{h=1}^{H-1} \mathds{1}_{\{s_{h} = s, a_{h} = a, s_{h+1} = s'\}}
\end{align*}
In particular, $\wt{R}_{X},  \wt{N}^{r}_{X}$ and $\wt{N}^{p}_{X}$ are defined as for the optimistic algorithm in Section~\ref{sec:privmech} and $\wt{R}_{2,X}$ is a privatized version of $R_{2,X}(s,a) = \sum_{h=1}^{H} r_{h}^{2}\mathds{1}_{\{s_{h} = s, a_{h} = a\}}$ for a trajectory $X$.s

The posterior updates we use in \ppsrl are then for all $s'\in \mathcal{S}$:
{\small\begin{equation}\label{eq:updates_perturb_posterior_1}
   \begin{aligned}
&\wt{\alpha}_{k}(s,a) = \alpha_{0}(s,a) + \wt{N}^{p}_{k}(s,a,s')\\
&\wt{\mu}_{k}(s,a) = \frac{\lambda_{0}(s,a)\mu_{0}(s,a) + \wt{R}_{k}(s,a)}{\lambda_{0}(s,a) + \wt{N}^{r}_{k}(s,a)}\\
&\wt{\lambda}_{k}(s,a) = \lambda_{0}(s,a) + \wt{N}^{r}_{k}(s,a) \\
%
&\wt{\nu}_{k}(s,a) = \tilde{\nu}_{0}(s,a) + \frac{\wt{N}^{r}_{k}(s,a)}{2}\\
&\wt{\beta}_{k}(s,a) = \beta_{0}(s,a) + \frac{\lambda_{0}(s,a)\wt{N}_{k}^{r}(s,a)\mu_{0}^{2}(s,a) - \wt{R}_{k}^{2}(s,a)}{2( \lambda_{0}(s,a) + \wt{N}_{k}^{r}(s,a))} \\
&\hspace{6cm}+ \frac{1}{2}\sum_{l\leq k-1} \wt{R}_{2,l} - \frac{\mu_{0}(s,a)\wt{R}_{k}(s,a)}{\lambda_{0}(s,a) + \wt{N}_{k}^{r}(s,a)}
\end{aligned}
\end{equation}}

In the following, we choose the Laplace mechanism as our privacy-preserving mechanism for \ppsrl, although we believe that it should be possible to use one of the other mechanisms we discussed. For each trajectory $X$, we add independent Laplace variables
to $(R_{X}(s,a), R_{X,2}(s,a), N^{r}_{X}(s,a), N_{X}^{p}(s,a))$ with parameter $8H/\varepsilon$. Following the same argument outlined in the proof of Thm.~\ref{thm:ldp_laplace},\todompout{Check reference} we can show that this privacy-preserving mechanism is $(\varepsilon, 0)$-LDP.

To ensure positivity, by concentration of Laplace variables we set the initial values of the parameters of the posterior distributions to:
\begin{align}\label{eq:initial_values_psrl}
&\alpha_{0}(s,a,s') = \max\{\sqrt{KS}, \ln(6S^{2}A/\delta)\}\frac{\sqrt{8\ln\left(6S^{2}A/\delta\right)}}{\varepsilon_{0}}\\
&\mu_{0}(s,a) = 0\\
&\lambda_{0}(s,a) = \max\{\sqrt{K}, \ln(6SA/\delta)\}\frac{\sqrt{8\ln\left(6SA/\delta\right)}}{\varepsilon_{0}}\\
&\nu_{0}(s,a) = \max\{\sqrt{K}, \ln(6SA/\delta)\}\frac{\sqrt{8\ln\left(6SA/\delta\right)}}{\varepsilon_{0}}\\
&\beta_{0}(s,a) = 5\max\{\sqrt{K}, \ln(6SA/\delta)\}\frac{\sqrt{8\ln\left(6SA/\delta\right)}}{\varepsilon_{0}}
\end{align}
where $K$ is the total number of episodes.
The pseudocode of \ppsrl is reported in Alg.~\ref{alg:privatepsrl}.

\paragraph{Empirical results}
We show empirical results for the \ppsrl algorithm in the RandomMDP environment in Figure~\ref{fig:exp.randommdp_psrl}.
While we have shown that this algorithm is $\varepsilon$-LDP and empirically outperforms optimistic approaches, we leave the regret analysis to future work.

\begin{algorithm}[t]
    \caption{\ppsrl}
    \label{alg:privatepsrl}
 \begin{algorithmic}
    \STATE {\bfseries Input:} Initial values: $\alpha_{0}, \mu_{0}, \lambda_{0}, \nu_{0}$ and $\beta_{0}$
    \FOR{episodes $k=1, \dots, K$}
            \STATE Draw empirical MDP, $\theta_{k}$ from the posterior and compute $\pi_{k}$ as the optimal policy for MDP $\theta_{k}$
            \STATE User $u_{k}$ executes policy $\pi_{k}$, collect trajectory $X_{k} = \{(s_{k,h}, a_{k,h}, r_{k,h})\mid h\leq H\}$
            \STATE Update noisy counts with $(\wt{R}_{X_{k}}(s,a), \wt{R}_{X_{k},2}(s,a), \wt{N}^{r}_{X_{k}}(s,a), \wt{N}_{X_{k}}^{p}(s,a))$ and posterior distribution
    \ENDFOR
 \end{algorithmic}
\end{algorithm}

\section{Additional Experiment}\label{app:additional_exp}
\todocout{is  this mentioned at all in the main text?}
\todoeout{No, changed it}
In this section, we explore a second experiment, in which we use the same the RandomMDP environment with the same parameters as in Sec.~\ref{sec:experiments} in order to
 investigate the effect of differential privacy on the learning process. For this, we run the \ucbvi algorithm for $K = 10^{3}$ episodes and collect the aggregate noisy statistics, $(\wt{R}_{K}(s,a))_{(s,a)\in \mathcal{S}\times\mathcal{A}}, (\wt{N}_{K}^{r}(s,a))_{(s,a)\in \mathcal{S}\times\mathcal{A}}$ and $(\wt{N}_{K}^{p}(s,a,s'))_{(s,a,s')\in \mathcal{S}\times\mathcal{A}\times\mathcal{S}}$
that have been generated by using the Laplace mechanism for each episode. We collect those statistics, $10^{3}$ times. We compare the histogram of those noisy statistics to that of the noiseless statistics used by \ucbvi in Fig.~\ref{fig:histogram_statistic}. This demonstrates that, as expected, there is much more variation in the statistics provided by the private mechanism. 
In Fig.~\ref{fig:histo_two_traj}, we applied the Laplace mechanism to two different random trajectories, $X$ and $X'$. We can see that, after applying the Laplace mechanism, the two distinct trajectories become almost indistinguishable. These two figures combined demonstrate the difficulty of learning from locally differentially private data.


\begin{figure}[t]
\begin{minipage}{0.45\linewidth}
\centering
\includegraphics[width=\linewidth]{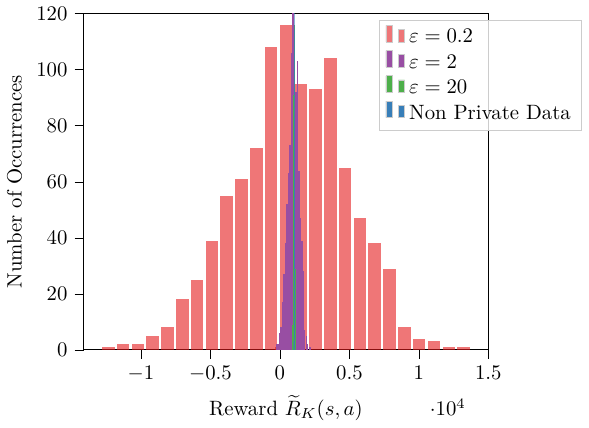}
\caption{Aggregate reward for privatized data with $\varepsilon\in \{0.2, 2, 20\}$ and non-privatized data for state $0$ and action $1$}
\label{fig:histogram_statistic}
\end{minipage}\hfill
\begin{minipage}{0.45\linewidth}
\centering
\includegraphics[width=0.87\linewidth]{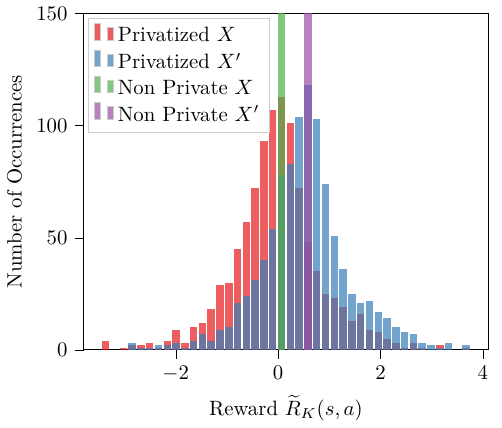}
\caption{Privatized cumulative reward over an episode for a given state-action pair and two different trajectories $X$ and $X'$ with $\varepsilon = 20$ for state $0$ and action $1$}
\label{fig:histo_two_traj}
\end{minipage}
\end{figure}

\section{Privacy Amplification by Shuffling in RL}\label{app:shuffle_model_rl}

In recent years, the shuffle model for privacy \citep{Cheu_2019,feldman2020hiding,chen2020distributed, balle2019privacy, erlingsson2020encode, erlingsson2020amplification}
 has attracted a lot of attention thanks its amplification property tof the differential privacy guarantees of locally differential data.

In this model of privacy, we consider $n$ users equipped with a local differential privacy mechanism, each user submits a locally private report to a random shuffler
which computes a random permutation of the users' reports. Those randomly shuffled reports are then sent to a analyzer which computes functions of
interests based on them. This setting was first introduced in \cite{bittau_2017} and was named the \emph{ESA} model (Encode-Shuffle-Analyze) and motivated
by need for anonymous data collection. \citep{erlingsson2020amplification} later provided an analysis of the amplification of privacy thanks to the combined
use of shuffling and local differential privacy showing that the shuffling model of privacy
is able to strike a middle ground between the totally decentralized but somewhat sample inefficient
\emph{local} model and the centralized but more sample efficient central model of privacy.

The  shuffling model has then been refined to study the impact on the size of the reports sent by users, i.e., how the accuracy of
a shuffling protocol can be improved when user are allowed to have higher communication threshold \citep{Cheu_2019, balle2019differentially}.
It has also been studied for different analyzer function, for instance histograms \citep{balcer2020separating} or summation \citep{Cheu_2019, balle2019privacy},
obtaining optimal protocol with better accuracy and lesser communication costs (\ie the number of messages or the size of those messages sent by a user).
Finally, the shuffle model has inspired  a privacy amplification algorithm for learning in distributed setting without server-initiated communication  \citep{balle2019privacy}.

Overall, the most attractive feature of this privacy model is that it offers a smooth transition in terms of privacy/utility tradeoff between stringent LDP requirements and
differential privacy requirements (see \citep{feldman2020hiding} for an example of this transition in the problem of estimating a distribution).

Formally, in our RL setting each episode $k$ represents a user $u_{k}$ which completes a trajectory $X_{u_{k}}$ in the MDP. The user computes a locally private version of its trajectory thanks to a privacy-preserving mechanism
$\mathcal{M}$. The result $\mathcal{M}(X_{u_{k}})$ is passed to a shuffler $\mathcal{R}$. This shuffler stores all the previous privatized trajectories before the current episode $k$, $(\mathcal{M}(X_{u_{l}}))_{l < k}$,
computes a random permutation $\sigma: [k-1] \rightarrow [k-1]$ and sends the permuted
set of privatized trajectories, $(\mathcal{M}(X_{u_{\sigma(l)}}))_{l\leq k-1}$ to an RL algorithm like \algo. This interaction protocol is detailed in Alg.~\ref{alg:shuffling_mechanism}.

\begin{algorithm}[tb]
  \caption{Shuffling Protocol}
  \label{alg:shuffling_mechanism}
  \begin{algorithmic}
    \STATE {\bfseries Input:} number of episodes $K$, horizon $H$,
    failure probability $\delta \in (0,1)$, bias $\alpha>1$, private randomizer $\mathcal{M}_{\text{sh}}$
    with LDP parameters $(\epsilon_0,\delta_0)$
    \FOR{$k=1$ {\bfseries to} $K$}
      \STATE Shuffler $\mathcal{R}$ sends $(\mathcal{M}_{\text{sh}}(X_{u_{\sigma_{k}(l)}}))_{l\leq k-1}$ with $\sigma_{k}$ a random permutatioon at each episode
      \STATE \algo computes policy $\pi_{k}$ based on $(\mathcal{M}_{\text{sh}}(X_{u_{\sigma_{k}(l)}}))_{l\leq k-1}$
      \STATE User $u_{k}$ executes policy $\pi_{k}$ in the environment, collects trajectory $X_{k} = \{(s_{k,h}, a_{k,h}, r_{k,h})_{h\leq H}\}$ and sends the privatized trajectory $\mathcal{M}_{\text{sh}}(X_{k})$ to $\mathcal{R}$
    \ENDFOR
  \end{algorithmic}
\end{algorithm}

In the specific case of RL, thanks to \citep{vietri2020privaterl} we know that any regret minimizing algorithm using $(\varepsilon, \delta)$-DP counters,
like $(N_{k}^{p})_{k\leq K}$ is $(\varepsilon, \delta)$-joint differentially private.

\subsection{Privacy-preserving mechanism $\mathcal{M}_{\text{sh}}$}

  A trajectory $X_{u} := \{(s_{h}, a_{h}, r_{h}) \mid h\leq H\}$ is a
  sequence of $H$ states, actions and rewards. In order to build
  a model of the MDP, \algo uses counters of the numbers of occurrences of
  each tuple of state-action $(s,a)$ and
  state, actions and next-state $(s,a,s')$. We adapt to the RL setting,
  the algorithm for bit-sum protocol presented in \citep{Cheu_2019}.
  The first step of the process $\mathcal{M}_{\text{sh}}$ is to apply a
  one-hot encoding the trajectory for each state-action. Let $x\in \{0,1\}^{H\times S \times A}$ and $y\in\{0,1\}^{(H-1)\times S \times A \times S}$ such that for each $(s,a,s')\in\mathcal{S}\times\mathcal{A}\times\mathcal{S}$
  \begin{align}\label{eq:temp_one_hot}
    \forall h\in \llbracket 1, H\rrbracket, \qquad x_{h,s,a} = \mathds{1}_{\{s_{h} = s, a_{h} = a\}}, \text{ and } y_{h,s,a,s'} = \mathds{1}_{\{s_{h} = s, a_{h}=a, s_{h+1} = s'\}}
  \end{align}
  To encode the reward, we first compute the reward for each state-action pair, $\left(r_{h}\mathds{1}_{\{s_{h} = s, a_{h} = a\}}\right)_{(h,s,a)\in\llbracket 1, H\rrbracket\times\mathcal{S}\times\mathcal{A}}$ then given a parameter $m\in\mathbb{N}^{\star}$ for each state-action pair $(s,a)$, we compute $b_{h,s,a}\in\{0,1\}^{m}$ such that
  for $j\in\llbracket 1, m\rrbracket$:
  \begin{align}\label{eq:temp_one_hot_reward}
    (b_{h,s,a})_{j} = \left\{\begin{matrix}
      1 & \text{ if } j < \mu_{h,s,a}\\
      \text{Ber}\left( p_{h,s,a}\right) & \text{ if } j = \mu_{h,s,a}\\
      0 & \text{ if } j>\mu_{h,s,a}\\
     \end{matrix} \right.
  \end{align}
  with $\mu_{h,s,a} = \left\lceil m r_{h}\mathds{1}_{\{s_{h} = s, a_{h} = a\}}\right\rceil$ and $p_{h, s,a} = m r_{h}\mathds{1}_{\{s_{h} = s, a_{h} = a\}} - \mu_{h, s,a} + 1$.

  \begin{algorithm}[tb]
    \caption{Local randomizer $R_{p}^{0/1}$}
    \label{alg:local_randomizer}
    \begin{algorithmic}
      \STATE {\bfseries Input:} randomization probability: $p\in[0,1]$, $x\in\{0,1\}$
      \STATE Let $b\sim \text{Ber}(p)$
      \IF{$b = 0$}
        \STATE Return $x$
      \ELSE
        \STATE Return $\text{Ber}(1/2)$
      \ENDIF
    \end{algorithmic}
  \end{algorithm}

  It is a well known result, \citep{Cheu_2019} that Alg.~\ref{alg:local_randomizer} with parameter $p$ guarantees $\ln(2/p - 1)$ differential privacy.
  Finally, the privacy-preserving mechanism $\mathcal{M}_{\text{sh}}$ is described by Alg.~\ref{alg:shuffle_randomizer}.

  \begin{algorithm}[tb]
    \caption{Privacy-preserving mechanism $\mathcal{M}_{\text{sh}}$}
    \label{alg:shuffle_randomizer}
    \begin{algorithmic}
      \STATE {\bfseries Input:} trajectory $\tau = \{(s_{h}, a_{h}, r_{h})_{h\leq H}\}$, privacy parameter $\varepsilon > 0$, parameter $m\in\mathbb{N}^{\star}$
      \STATE Compute $x$ and $y$ as in Eq.~\eqref{eq:temp_one_hot} and $(b_{h, s,a})_{(s,a)\in\mathcal{S}\times\mathcal{A}}$ as in Eq.~\eqref{eq:temp_one_hot_reward}
      \STATE Set $p = \frac{2}{\exp(\varepsilon) + 1}$
      \STATE Return $(R_{p}^{0/1}(x_{h,s,a}))_{(h,s,a)}$, $(R_{p}^{0/1}(y_{h,s,a,s'}))_{(h,s,a,s')}$ and $( (R_{p}^{0/1}((b_{h, s,a})_{j})_{j\leq m})_{(h, s,a)} $
    \end{algorithmic}
  \end{algorithm}

  Using standard analysis, we can show that this local mechanism $R_{p}^{0/1}$ is
  roughly $H\varepsilon$-LDP for any $\varepsilon>0$.Upon receiving
  the shuffled privatized, the algorithm \algo computes the different
  counts $(\tilde{N}_{k}^{p}(s,a,s'))_{(s,a,s')}$, $(\tilde{N}_{k}^{r}(s,a))_{(s,a)}$ and $(\tilde{R}_{k}(s,a))_{(s,a)}$. For any $(s,a,s')\in \mathcal{S}\times\mathcal{A}\times \mathcal{S}$, we define the counters as:
\begin{align}
  &\tilde{N}_{k}^{r}(s,a) = \frac{1}{1 - p}\left( \sum_{l=1}^{k-1} \sum_{h=1}^{H} \left[R_{p}^{0/1}(x_{h,s,a}) - \frac{p}{2}\right]\right)\\
  &\tilde{N}_{k}^{p}(s,a,s') = \frac{1}{1 - p}\left( \sum_{l=1}^{k-1} \sum_{h=1}^{H} \left[R_{p}^{0/1}(y_{h,s,a,s'}) - \frac{p}{2}\right]\right)\\
  &\tilde{R}_{k}^{r}(s,a) = \frac{1}{m(1 - p)}\left( \sum_{j=1}^{m}\sum_{l=1}^{k-1} \sum_{h=1}^{H} \left[R_{p}^{0/1}((b_{h,s,a})_{j}) - \frac{p}{2}\right]\right)
\end{align}

Therefore, thanks to Claim~$4.6$ of \citep{Cheu_2019}, we have at the beginning of episode $k$,$(\tilde{N}_{k}^{r}(s,a))_{(s,a)}$ and $(\tilde{N}_{k}^{p}(s,a,s'))_{(s,a,s')}$ are $(\varepsilon_{k,c}, \delta_{0})$-DP with any $\delta_{0}>0$ and:
\begin{align}
  \varepsilon_{k,c} = \frac{32\log(4/\delta_{0})/\sqrt{(k-1)H}}{\sqrt{p - \sqrt{\frac{2p\log(2/\delta_{0})}{(k-1)H}}}}\left(1 - \left(p - \sqrt{\frac{2p\log(2/\delta_{0})}{(k-1)H}}\right) \right)
\end{align}
with $p\in \left[\frac{14}{(k-1)H}\log(4/\delta_{0}), 1\right]$.
But we have that with probability at least $1 - \delta$, for any $\delta>0$, that:
{\small\begin{align*}
  &\left|\sum_{l=1}^{k-1}\sum_{h=1}^{H} \mathds{1}_{\{s_{l,h}=s, a_{l,h} = a\}}  - \tilde{N}_{k}^{r}(s,a)\right| \leq \frac{1}{1 - p}\left(\sqrt{(k-1)Hp(1 - p/2)\ln(1/\delta)} + \frac{2\ln(1/\delta)}{3}\right)\\
  &\left|\sum_{l=1}^{k-1}\sum_{h=1}^{H-1} \mathds{1}_{\Big\{\begin{subarray}{l}s_{k,h}=s, \\a_{k,h} = a\\, s_{k,h+1}= s'\end{subarray}\Big\}}  - \tilde{N}_{k}^{p}(s,a,s')\right| \leq \frac{1}{1 - p}\left(\sqrt{(k-1)Hp(1 - p/2)\ln(1/\delta)} + \frac{2\ln(1/\delta)}{3}\right)
\end{align*}}

The same type of result of result holds for the cumulative reward in each state-action pair $(s,a)$, albeit some
small technical difficulties due the estimated sum being in $\mathbb{R}$ and not an integer contrary to the counters for the number of visits.

\subsection{Impact on the Regret}
  We have mentioned that thanks to the shuffling mechanism the counters $(\tilde{R}_{k}(s,a))_{(s,a)}$, $(\tilde{N}_{k}^{r}(s,a))_{(s,a)}$, $(\tilde{N}_{k}^{p}(s,a,s'))_{(s,a,s')}$
  enjoy a $(\varepsilon_{c},\delta)$-DP guarantee, in addition to the
  $\epsilon_{0}$-LDP guarantee. But the utility bound in the last subsection
  highlights that for a strict constraint on the level of local differential privacy
  the utility of each counters is of order $\frac{\sqrt{kH}}{\exp(\varepsilon_{0}) - 1}$ therefore
  using Thm.~\ref{thm:regret_any_mechanism}, the regret of \algo coupled with
  $\mathcal{M}_{\text{sh}}$ is bounded with high probability by $\frac{H^{2}S^{2}A\sqrt{KH}}{\exp(\varepsilon_{0}/H) - 1}$. This result is similar to the result of \citep{feldman2020hiding} of Sec.~$5.1$ about density estimation where the shuffle model recovers the known rate of convergence
  of $\mathcal{O}(1/\varepsilon\sqrt{n})$ under an $\varepsilon$-LDP constraint
  with $n$ samples.

  However, in the reinforcement learning setting the shuffle model
  might allow to interpolate between LDP setting presented in this paper
  and the joint differential privacy setting of \cite{shariff2018differentially,
  vietri2020privaterl}. One difficulty here being that because each user interacts only once with the RL algorithm
  the probability used by the local randomizer $R_{p}^{0/1}$ ha to be dependent on the number of previous episode to ensure
  a good $(\varepsilon, \delta)$-JDP guarantee. In other words, for the very first episodes the privacy amplification of the shuffle
  model is negligible therefore the privacy parameter for those early users has to be stronger than for the latter ones which are somewhat hidden by the crowd.
  Albeit this minor issue, a good choice of the probabilities $(p_{i})_{k\leq K}$ may be able to guarantee $(\varepsilon, \delta)$-JDP (for any $\varepsilon>0$ and $\delta>0$) and a regret of order
  $\mathcal{O}(\sqrt{K} + \frac{\log(K)}{\varepsilon})$.



\end{document}